\documentclass[nohyperref]{article}

\usepackage{microtype}
\usepackage{graphicx}
\usepackage{subfigure}
\usepackage{booktabs} 

\usepackage[accepted]{icml2022}

\usepackage{amsmath}
\usepackage{amssymb}
\usepackage{mathtools}
\usepackage{amsthm}

\usepackage{dsfont}
\usepackage{mathabx}
\usepackage{enumerate}

\usepackage{titletoc}
\definecolor{someblue}{RGB}{0, 26, 101}
\usepackage[page,header]{appendix}
\usepackage{hyperref}
\hypersetup{
	colorlinks   = true,
	urlcolor     = someblue,
	linkcolor    = someblue,
	citecolor    = someblue
}
\usepackage[capitalize,noabbrev]{cleveref}

\theoremstyle{plain}
\newtheorem{theorem}{Theorem}[section]

\newtheorem{lemma}[theorem]{Lemma}

\theoremstyle{definition}
\newtheorem{definition}[theorem]{Definition}

\theoremstyle{remark}
\newtheorem{remark}[theorem]{Remark}

\usepackage[textsize=tiny]{todonotes}
\allowdisplaybreaks

\icmltitlerunning{Nearly Minimax Optimal Reinforcement Learning
with Linear Function Approximation}

\newcommand{\bug}[1]{{\color{black}  \text{}#1}}

\begin{document}
	
	\twocolumn[
	\icmltitle{Nearly Minimax Optimal Reinforcement Learning with\\ Linear Function Approximation}
	
	
	
	
	\begin{icmlauthorlist}
		\icmlauthor{Pihe Hu}{thuiiis}
		\icmlauthor{Yu Chen}{thumath}
		\icmlauthor{Longbo Huang}{thuiiis}
	\end{icmlauthorlist}
	
	\icmlaffiliation{thuiiis}{Institute for Interdisciplinary Information Sciences, Tsinghua University, Beijing, China}
	\icmlaffiliation{thumath}{Department of Mathematical Sciences, Tsinghua University, Beijing, China}
	
	
	\icmlkeywords{Reinforcement Learning, Linear Function Approximation, Minimax Optimality}
	
	\vskip 0.3in
	]
	
	
	
	\printAffiliationsAndNotice{}  
        \textcolor{red}{\textbf{Erratum:}
        We call the attention of the reader that there is a technical error in building the over-optimistic value function $\widedot{V}_{k,h}(\cdot)$ in the ICML camera ready \cite{hu2022nearly} for our paper. This technical error has been identified by \cite{he2022nearly,agarwal2022vo}. They have independently proposed new algorithms and analyses to achieve the minimax optimal regret for linear MDPs. We refer readers to their papers for more details. In this manuscript, we have fixed the error of our algorithm by using the technique of the ``rare-switching" value function from \cite{he2022nearly}.
        We acknowledge that \cite{he2022nearly,agarwal2022vo} are the first to achieve the minimax optimal regret for linear MDPs (to the best of our knowledge).
        }
	\begin{abstract}
		We study reinforcement learning with linear function approximation where the transition probability and reward functions are linear with respect to a feature mapping $\boldsymbol{\phi}(s,a)$. Specifically, we consider the episodic inhomogeneous linear Markov Decision Process (MDP), and propose a novel computation-efficient algorithm, LSVI-UCB$^+$, which achieves an $\widetilde{O}(Hd\sqrt{T})$ regret bound where $H$ is the episode length, $d$ is the feature dimension, and $T$ is the number of steps. LSVI-UCB$^+$ builds on weighted ridge regression and upper confidence value iteration with a Bernstein-type exploration bonus. Our statistical results are obtained with novel analytical tools, including a new Bernstein self-normalized bound with conservatism on elliptical potentials, and refined analysis of the correction term. This is a minimax optimal algorithm for linear MDPs up to logarithmic factors, which closes the $\sqrt{Hd}$ gap between the upper bound of $\widetilde{O}(\sqrt{H^3d^3T})$ in \cite{jin2020provably} and lower bound of $\Omega(Hd\sqrt{T})$ for linear MDPs.
	\end{abstract}
	\section{Introduction}
	Reinforcement Learning (RL) has demonstrated phenomenal empirical success in many areas, including games, robotic control, etc., where improving sample complexity is always an important topic.
	When the state space and action space are finite, the Markov decision process (MDP) has been proven to achieve nearly minimax optimal sample-complexity with the generative model in \cite{azar2013minimax}. For harder RL settings, nearly minimax optimal sample-complexities are obtained in \cite{azar2017minimax} for finite horizon episodic MDPs, and \cite{he2020minimax,tossou2019near} for infinite horizon MDPs\footnote{An algorithm is nearly minimax optimal if its sample complexity matches the minimax lower bound up to logarithmic factors.}.
	However, MDPs are known to suffer from the curse-of-dimensionality due to large and possibly infinite state and action space.
	
	Function approximation is an essential approach for handling large MDPs, which assumes that the problem structure has a compact representation concerning state or state-action pairs and enables the development of nearly minimax optimal theoretical guarantees for RL problems.
	Linear function approximation is one of the most fundamental function approximations.
	It has a significant impact since many problems can be linearly-parameterized structurally or combined with embedding, where linear MDPs and linear mixture MDPs are two of the most popular models.
	Representative works for these two settings are presented in Table~\ref{tb:comp}.
 
    In this paper, we design the LSVI-UCB$^+$ algorithm, which reaches minimax optimal regret up to logarithmic factors.
	LSVI-UCB$^+$ overcomes barriers to nearly minimax optimality in existing works \cite{jin2020provably,wang2020optimism,zanette2020learning,wang2020reinforcement} for linear MDPs and their variations, including overly aggressive exploration and extra cost for building a uniform convergence argument by covering net.
	It constructs a Bernstein-type bonus to perform efficient exploration, which enables a $\sqrt{H}$ factor reduction in regret.
	Besides, the extra $\sqrt{d}$ dependency from building the uniform convergence argument can be removed by our novel technique of bounding the correction term $(\widehat{\mathbb{P}}-\mathbb{P})(\widehat{V}-V^*)$.
	Notably, minimax optimal algorithms \cite{azar2017minimax, zanette2019tighter} for tabular MDPs also utilize the Bernstein-type bonus for exploration with refined consideration of the correction term in analysis.
	However, results for tabular MDPs cannot be applied directly to our settings due to the need for building the Bernstein inequality for vector-valued martingales in linear settings.
	It is worth mentioning that the above Bernstein inequality has been studied in  \cite{zhou2021nearly}, where the UCRL-VTR$^+$ is nearly minimax optimal when $d\ge H$ for linear mixture MDPs.
	By contrast, our proposed LSVI-UCB$^+$ algorithm achieves nearly minimax optimal regret without requiring this assumption for linear MDPs and can be further generalized to linear mixture MDPs such that the nearly minimax optimal regret can also be obtained without $d\ge H$.
	This is because our proposed Bernstein inequality (Theorem \ref{th:self}) is sharper by considering the conservatism on elliptical potentials.
	Our contributions are summarized below:
	
	\begin{table}[t]
		\caption{Theoretical results on RL with linear function approximation, where $\dagger$ denotes that rewards are adversarial, and the lower bound holds for both settings.\vspace{-.2cm}}
		\label{tb:comp}
		\begin{center}
			\begin{small}
				\begin{tabular}{ll}
					\toprule
					Algorithm & Regret \\
					\midrule
					Linear MDP\\
					\midrule
					OPT-RLSVI \cite{zanette2020frequentist} & $\widetilde{O}(H^2d^2\sqrt{T})$ \\
     LSVI-UCB \cite{jin2020provably} & $\widetilde{O}(\sqrt{H^3d^3T})$ \\
					\bug{LSVI-UCB++\cite{he2022nearly}} & \bug{$\widetilde{O}(Hd\sqrt{T})$} \\
                    \bug{VOQL \cite{agarwal2022vo}} & \bug{$\widetilde{O}(Hd\sqrt{T})$ }\\
					LSVI-UCB$^+$ (\textbf{this paper}) & $\widetilde{O}(Hd\sqrt{T})$ \\
					\midrule
					Linear Mixture MDP\\
					\midrule
					OPPO$^\dagger$  \cite{cai2020provably}& $\widetilde{O}(\sqrt{H^3d^2T})$ \\
					UCRL-VTR \cite{ayoub2020model} & $\widetilde{O}(\sqrt{H^3d^2T})$ \\
					UCRL-VTR$^+$ \cite{zhou2021nearly} & $\widetilde{O}(\sqrt{H^2d^2T+H^3dT})$ \\
					\midrule
					Lower Bound \cite{zhou2021nearly} & $\Omega(Hd\sqrt{T})$\\
					\bottomrule
				\end{tabular}
			\end{small}
		\end{center}
		\vspace{-.4cm}
	\end{table}
	\vspace{-.2cm}
	\begin{itemize}
		\item We develop a novel Bernstein bound of $\widetilde{O}(\sigma\sqrt{d}+R)$ for self-normalized martingales, which is sharper than the analog inequalities in \cite{zhou2021nearly}.
		By utilizing the conservatism on elliptical potentials, the bound can be further improved to $\widetilde{O}(\sigma\sqrt{d})$, which serves as a new analytical tool for RL.

		\item We propose the LSVI-UCB$^+$ algorithm based on a Bernstein-type exploration bonus and weighted ridge regression, with weights determined by value function variances and exploration uncertainty.
		LSVI-UCB$^+$ achieves an $\widetilde{O}\left(Hd\sqrt{T}\right)$ regret, and is minimax optimal up to logarithmic factors in large-sample regime.
		
		\item
		We improve the analytical framework of statistical complexity for linear MDPs by bounding the correction term $(\widehat{\mathbb{P}}-\mathbb{P})(\widehat{V}-V)$.
		Combined with the Bernstein self-normalized bound, this new analytical framework can remove the extra dependencies on $H$ and $d$, which is very different from the traditional Hoeffding bound used in  \cite{jin2020provably,wang2020optimism,wang2020reinforcement}.
	\end{itemize}
	\vspace{-.6cm}
	\paragraph{Notations}	
	Scalars are denoted in lower case letters, and vectors/matrices are denoted in boldface letters.
	Denote $\|\boldsymbol{x}\|_{\mathbf{\Lambda}}^{2}=\boldsymbol{x}^{\top} \mathbf{\Lambda} \boldsymbol{x}$ for vector $\boldsymbol{x}$ and positive definite matrix $\mathbf{\Lambda}$.
	Denote $\{1, . . . , n\}$ as $[n]$ and the truncated value of $x$ in $[a, b]$ interval as $[x]_{[a, b]}$ for $a \leq b$.
	Define $a_{n}=O(b_{n})$ if there exists an absolute constant $c>0$ such that $a_{n} \leq c b_{n}$ holds for all $n \geq 1$ and define $a_{n}=\Omega(b_{n})$ for inverse direction.
	$\widetilde{O}(\cdot)$ further suppresses the polylogarithmic factors in $O(\cdot)$.

	\section{Related Work}
	\paragraph{Linear Bandits}
	Linear stochastic bandits can be regarded as a special case of linearly-parameterized MDPs with episode length $1$.
	\cite{dani2008stochastic} proposes an algorithm with $O(d\sqrt{T\log^3T})$ regret by building confidence ball with Freedman inequality \cite{freedman1975tail}.
	\cite{abbasi2011improved} improves the regret to $O(d\sqrt{T\log^2T})$ with a self-normalized tail inequality, derived by the method of mixture \cite{victor2009self}.
	\cite{li2021tight} further proposes an algorithm with  $O(d\sqrt{T\log T}\operatorname{poly}(\log\log T))$ regret by bounding the supremum of self-normalized processes, which matches the lower bound up to a  $\operatorname{poly}(\log\log T)$ factor.
	The self-normalized tail inequalities for linear bandits in these works are all Hoeffding-type, i.e., only consider sub-Gaussian noises.
	However, for linear RL, Bernstein-type inequalities considering the sub-exponential noise, are necessary for sharper statistical results.
	
	\paragraph{RL with Linear Function Approximation}
	Recent works have focused on designing statistically and/or computationally efficient algorithms for RL with linear function approximation.
	The first sample efficient algorithm is introduced by \cite{jiang2017contextual}, where low Bellman rank is considered.
	Subsequent works on this setting include \cite{dann2018oracle, sun2019model}.
	\cite{yang2019sample} develops the first statistically and computationally efficient algorithm for linear MDPs with a simulator, where the transition probability and reward functions are linear concerning a feature mapping $\boldsymbol{\phi}(s,a)$.
	Subsequently, \cite{jin2020provably} considers RL settings for linear MDPs and propose LSVI-UCB algorithm reaching  $\widetilde{O}(\sqrt{H^3d^3T})$ regret.
	Concurrently, \cite{zanette2020frequentist} provides a Thompson sampling based algorithm with regret bound of $\widetilde{O}(d^2H^2\sqrt{T})$.
	More works generalize linear MDPs includes \cite{zanette2020learning} for low inherent bellman error, \cite{wang2020optimism} for linear Q function, and \cite{wang2020reinforcement} for bounded Eluder dimension.
	
	Another popular linearly-parameterized MDP is the linear mixture MDP, where transition probability is linear to the feature function over (state, action, next state) triples.
	\cite{modi2020sample} firstly considers the statistical complexity of this setting and \cite{yang2020reinforcement} provides an $\widetilde{O}(H^2d\log T\sqrt{T})$ regret with special case of low-dimensional representation of the transition matrix.
	Subsequently, \cite{jia2020model, ayoub2020model} proposes UCRL-VTR algorithm with $\widetilde{O}(\sqrt{H^3d^2T})$ regret, and \cite{cai2020provably} considers adversarial rewards setting, giving same regret. 
	Notably, the nearly minimax optimal regret for linear mixture MDP is first obtained by UCRL-VTR$^+$ in \cite{zhou2021nearly} under $d\ge H$ case.
	
	\section{Preliminaries}
	We consider episodic finite horizon MDP $\mathcal{M}=\{\mathcal{S}, \mathcal{A}, H, \{\mathbb{P}_h\}_h, \{r_h\}_h\}$, where $\mathcal{S}$ is the state space, $\mathcal{A}$ is the action space, $H\in\mathbb{Z}^+$ is the length of each episode, $\mathbb{P}_h:\mathcal{S}\times\mathcal{A}\rightarrow\Delta(\mathcal{S})$ and $r_h:\mathcal{S}\times\mathcal{A}\rightarrow[0,1]$ are time-dependent transition probability and deterministic reward function.
	We assume that $\mathcal{S}$ is a measurable space with possibly infinite number of elements and $\mathcal{A}$ is a finite set.
	
	For a time-inhomogeneous MDP, the policy is time-dependent, which is denoted as $\pi=\{\pi_1,...,\pi_{H}\}$.
	Here $\pi_h(s)$ is the action that agent takes at state $s$ at the $h$-th step.
	The value function $V_h^\pi:\mathcal{S}\rightarrow\mathbb{R}$ is the expected value of cumulative rewards received under policy $\pi$ when starting from a state at $h$-th step, given as
	$$V_{h}^{\pi}(s):=\mathbb{E}\left[\sum_{h^{\prime}=h}^{H} r_{h^{\prime}}(s_{h^{\prime}}, \pi_{h'}(s_{h^{\prime}}))\mid s_{h}=s,\pi\right],$$
	for any $s\in\mathcal{S},h\in[H]$.
	The state-action function $Q_h^\pi:\mathcal{S}\times\mathcal{A}\rightarrow\mathbb{R}$ gives the expected value of cumulative rewards starting from a state-action pair at $h$-th step, defined as $$Q_{h}^{\pi}(s,a)=\mathbb{E}\left[\sum_{h^{\prime}=h}^{H} r_{h^{\prime}}\left(s_{h^{\prime}}, a_{h^{\prime}}\right) \mid s_{h}=s, a_{h}=a,\pi\right],$$
	for any $(s,a)\in\mathcal{S}\times\mathcal{A},h\in[H]$.
	For any function $V: \mathcal{S} \rightarrow \mathbb{R}$, we denote $\mathbb{P}_{h}V(s, a)=\mathbb{E}_{s^{\prime}\sim \mathbb{P}_{h}(\cdot \mid s, a)} V(s^{\prime})$ and $\left[\mathbb{V}_{h} V\right](s, a)=\mathbb{P}_{h} V^{2}(s, a)-[\mathbb{P}_{h} V(s, a)]^{2}$, where $V^{2}$ stands for the function whose value at $s$ is $V^{2}(s)$.
	The Bellman equation associated with a policy $\pi$ is
	\begin{align*}
	    Q_{h}^{\pi}(s,a)=&r_{h}(s,a)+\mathbb{P}_{h} V_{h+1}^{\pi}(s, a)\\
	    V_{h}^{\pi}(s)=&Q_{h}^{\pi}(s, \pi_{h}(s))
	\end{align*}
	for any $(s,a)\in\mathcal{S}\times\mathcal{A},h\in[H]$.
	Since the action space and the episode length are both finite, there always exists an optimal policy $\pi^*$ such that $V_{h}^{\star}(s)=\sup _{\pi} V_{h}^{\pi}(s)$ for any $s\in\mathcal{S},h\in[H]$, with Bellman optimality equation as
	\begin{align*}
	    Q_{h}^{\star}(s, a)=&r_{h}(s,a)+\mathbb{P}_{h} V_{h+1}^{\star}(s, a)\\
	    V_{h}^{\star}(s)=&\max _{a \in \mathcal{A}} Q_{h}^{\star}(s, a)
	\end{align*}
	for any $(s,a)\in\mathcal{S}\times\mathcal{A},h\in[H]$.
	
	The structural assumption we make in this paper is a linear structure in both transition and reward, which has been considered in \cite{yang2019sample, jin2020provably, zanette2020frequentist}. The formal definition is as follows.
	\begin{definition}[Linear MDP]\label{def:linear}
		A MDP $\mathcal{M}=\{\mathcal{S}, \mathcal{A}, H, \{\mathbb{P}_h\}_h, \{r_h\}_h\}$ is a linear MDP with a known feature mapping $\boldsymbol{\phi}:\mathcal{S}\times\mathcal{A}\rightarrow\mathbb{R}^d$, if for any $h\in[H]$, there exist $|\mathcal{S}|$ unknown $d$-dimensional measures $\boldsymbol{\mu}_h=(\mu_h{(1)},...,\mu_h{(|\mathcal{S}|)})\in\mathbb{R}^{d\times|\mathcal{S}|}$
		and an unknown vector $\boldsymbol{\theta}_h\in\mathbb{R}^d$, such that for any $(s,a)\in\mathcal{S}\times\mathcal{A}$, we have
		$$
		\mathbb{P}_{h}(\cdot\mid s,a)=\langle\boldsymbol{\boldsymbol{\phi}}(s,a),\boldsymbol{\mu}_{h}(\cdot)\rangle,\,\, r_{h}(s,a)=\langle\boldsymbol{\phi}(s,a),\boldsymbol{\theta}_{h}\rangle.
		$$
	\end{definition}
	
	We make the following assumptions, similar to existing literature \cite{jin2020provably,agarwal2019reinforcement}. Specifically, for any $h\in[H]$,  (\romannumeral1) $\sup_{s,a}\|\boldsymbol{\phi}(x, a)\|_2\leq 1$, (\romannumeral2) $\|\boldsymbol{\mu}_h\boldsymbol{v}\|_2\le\sqrt{d}$ for any vector $\boldsymbol{v}\in\mathbb{R}^{|\mathcal{S}|}$ with $\|\boldsymbol{v}\|_\infty\le1$, (\romannumeral3) $\|\boldsymbol{\theta}_{h}\|_2\leq W$, and (\romannumeral4) $r_h(s,a)\in[0,1]$ for all $(s,a)\in\mathcal{S}\times\mathcal{A}$.
	
	In this paper, we focus on the setting where the reward function $\{r_h\}_{h\in[H]}$, i.e.,  $\{\boldsymbol{\theta}_h\}_{h\in[H]}$ is known, but our algorithm can readily be extended to handle unknown rewards. 

	\paragraph{Learning Protocol}
	In every episode $k$, the learner first proposes a policy $\pi^{k}$ based on all history information up to the end of episode $k-1$. The learner then executes $\pi^{k}$ to generate a single trajectory $\tau^k=\{s_h^k,a_h^k\}_{h=1}^{H}$ with $a_h^k=\pi_h^{k}(s_h^k)$ and $s_{h+1}^k\sim\mathbb{P}_h(\cdot|s_h^k,a_h^k)$.
	The goal of the learner is to learn the optimal policy by interacting with the environment during $K$ episodes.
	For the $k$-th episode, the initial state $s_1^k$ is picked by the adversary and the optimal policy will minimize the cumulative regret over $K$ episodes: $$\operatorname{Regret}(K)=\sum_{k=1}^{K}[V_{1}^{\star}(s_{1}^{k})-V_{1}^{\pi_{k}}(s_{1}^{k})].$$
	
	\section{Strategic Exploration in Linear MDP}
	Section~\ref{sec:learning} illustrates the standard ways of strategic exploration in linear MDPs in existing works, i.e., optimistic value iteration with parameters estimated by linear ridge regression.
	Next, we point out in Section~\ref{sec:optimal} barriers to minimax optimality in existing algorithms, which also helps explain our algorithm design Section~\ref{sec:alg}.

	\subsection{Optimistic Learning in Linear MDPs}\label{sec:learning}
	Optimistic learning evolves in an episodic fashion.
	In episode $k$, the agent first estimates unknown parameters of the linear MDP by historical data up to episodes $k-1$.
	One standard approach is estimating the parameter $\boldsymbol{\omega}_{h}^*=\boldsymbol{\theta}_h+\boldsymbol{\mu}_h\boldsymbol{V}_{h+1}^*$ by linear ridge regression, as LSVI-UCB in \cite{jin2020provably} and its variants in  \cite{wang2020optimism,wang2020reinforcement}, since the optimal Q function $Q_h^*(s,a)=\langle\boldsymbol{\omega}_{h}^*,\boldsymbol{\phi}(s, a)\rangle$ according to Proposition 2.3 in \cite{jin2020provably}.
	Subsequently, an optimistic Q function $Q_{k,h}$ in Eq. (\ref{eq:old_q}) is constructed with the learned parameter $\boldsymbol{\omega}_{k,h}$ and the exploration bonus $\beta\|\boldsymbol{\phi}(\cdot,\cdot)\|_{\mathbf{\Lambda}_{k,h}^{-1}}$.
	The agent then follows a greedy policy $\pi^k$ of $Q_{k,h}$ to interact with the environment and repeat the above procedure in the next episode.
	
	We illustrate two major steps of linear ridge regression and the construction of the optimistic Q function below.
	
	\paragraph{Linear Ridge Regression}
	To estimate the optimal $\boldsymbol{\omega}_{h}^*$, the following regularized least-squares problem is proposed:
	\begin{align}\label{eq:regression1}
	    \underset{\boldsymbol{\omega} \in \mathbb{R}^{d}}{\operatorname{min}}\sum_{i=1}^{k-1}[V_{k,h+1}(s_{h+1}^i)-\boldsymbol{\omega}^{\top}\boldsymbol{\phi}(s_{h}^{i},a_{h}^{i})]^{2}+\lambda\|\boldsymbol{\omega}\|_2^{2}
	\end{align}
	The closed-form solution to Eq.~({\ref{eq:regression1}}) is
	$$\boldsymbol{\omega}_{k,h}={\mathbf{\Lambda}}_{k,h}^{-1}\sum_{i=1}^{k-1}\boldsymbol{\phi}(s_{h}^{i},a_{h}^{i})^{\top}V_{k,h+1}(s_{h+1}^i)$$
	where $\mathbf{\Lambda}_{k,h}=\sum_{i=1}^{k-1} \boldsymbol{\phi}(s_{h}^{i}, a_{h}^{i}) \boldsymbol{\phi}(s_{h}^{i}, a_{h}^{i})^{\top}+\lambda \mathbf{I}$, and $V_{k,h}(\cdot)$ is given in Eq.~(\ref{eq:old_v}).
	\paragraph{Optimistic Estimator}
	$\boldsymbol{\omega}_{k,h}$ is then used to build an optimistic state-action function in Eq.~(\ref{eq:old_q}) with exploration bonus $\beta\|\boldsymbol{\phi}(\cdot,\cdot)\|_{\mathbf{\Lambda}_{k,h}^{-1}}$ to encourage exploration, and optimistic value function is given in Eq.~(\ref{eq:old_v}) as well.
	Notice that these two functions are built in a backwards fashion from stage $H$ to $1$, such that named as optimistic value iteration.
	\begin{align}
	        Q_{k,h}(\cdot,\cdot)=&\langle\boldsymbol{\omega}_{k,h},\boldsymbol{\phi}(\cdot,\cdot)\rangle+\beta\|\boldsymbol{\phi}(\cdot,\cdot)\|_{\mathbf{\Lambda}_{k,h}^{-1}}\label{eq:old_q}\\
	        V_{k,h}(\cdot)=&\max_{a\in\mathcal{A}}Q_{k,h}(\cdot,a)\label{eq:old_v}
	\end{align}
	In particular, denote the optimistic confidence set $\mathcal{C}_{k,h}:=\{\boldsymbol{\omega}:\|\boldsymbol{\omega}-\boldsymbol{\omega}_{k, h}\|_{\boldsymbol{\Lambda}_{k, h}} \leq \beta\}$ such that $Q_{k,h}(\cdot,\cdot)=\max_{\boldsymbol{\omega}\in\mathcal{C}_{k,h}}\langle\boldsymbol{\omega},\boldsymbol{\phi}(\cdot,\cdot)\rangle$.
	Notably, confidence set $\mathcal{C}_{k, h}$ is an ellipsoid centered at $\boldsymbol{\omega}_{k,h}$, with shape parameter $\boldsymbol{\Lambda}_{k, h}$ and radius $\beta$ (usually named as the exploration radius).
	It can be proved that with high probability, $\boldsymbol{\omega}_{h}^*\in\mathcal{C}_{k, h}$ by using self-normalized tail inequalities for vector-valued martingales, e.g., Theorem 1 in \cite{abbasi2011improved}, used broadly in the analysis of linear bandits or RL with linear function approximation.
	Consequently, functions in Eq.~(\ref{eq:old_q}), (\ref{eq:old_v}) obtains optimism in high probability.
	
	\subsection{Barriers to Minimax Optimality}\label{sec:optimal}
	The above optimistic learning based value iteration is a commonly adopted paradigm of RL with linear function approximation in existing works, e.g.,  \cite{jin2020provably,wang2020optimism,wang2020reinforcement}. 
	However, the best-known regret upper bound for linear MDPs is $\widetilde{O}(\sqrt{H^3d^3T})$ by LSVI-UCB algorithm in \cite{jin2020provably}, while the best known lower bound is $\Omega(Hd\sqrt{T})$ according to \cite{zhou2021nearly}.
	As shown in Section~\ref{sec:results}, the lower bound is tight. We analyze where the $\sqrt{Hd}$ gap comes from and then propose corresponding solutions, which immediately sheds light on designing the efficient LSVI-UCB$^+$ algorithm in the next section.

	\subsubsection{Overly Aggressive Exploration}\label{sec:aggressive}
	The tradeoff between exploitation and exploration is a central task for RL algorithms, implemented by designing exploration bonuses in optimistic learning.
	The current $\sqrt{H}$ gap stems from the overly aggressive exploration, which means that the current exploration radius $\beta=\widetilde{O}(Hd)$ in existing works,  e.g.,  \cite{jin2020provably,wang2020optimism,ayoub2020model} is too large and leads to insufficient exploitation. 
	The underlying reason remains that a bonus with $\widetilde{O}(Hd)$ radius is intrinsically Hoeffding-type since it has the order of the magnitude of the considered martingale difference sequence (MDS).
	We prove that a Bernstein-type bonus, based on the variance of the MDS, combined with the Law of Total Variance (LTV) \cite{lattimore2012pac}, can reduce one $\sqrt{H}$ factor of regrets in linear MDPs.
	The motivation for this improvement comes from prior works \cite{azar2017minimax, jin2018q, zanette2019tighter} for tabular MDPs, which succeeded in achieving $\sqrt{H}$ regret reduction by introducing a Bernstein-type bonus.
	For linear mixtures MDPs, UCRL-VTR$^+$ in \cite{zhou2021nearly} firstly introduces a Bernstein-type bonus and also achieves a $\sqrt{H}$ regret reduction.
	However, a direct adaption of UCRL-VTR$^+$ in linear MDPs will not improve the regret due to the extra cost of building a uniform convergence argument.
	
	\subsubsection{Extra Uniform Convergence Cost}\label{sec:sqrtd}
	Introducing a $\varepsilon$-covering net is a common approach to build a uniform convergence argument over a function class.
	Many algorithms for RL with linear function approximation achieve polynomial sample complexity with this approach.
	However, this brings extra dependency on $d$ in the regret, as presented in prior analysis, e.g., LSVI-UCB in \cite{jin2020provably} and its variants in \cite{wang2020optimism,wang2020reinforcement}.
	Specifically, when bounding the deviation term  $(\widehat{\mathbb{P}}-\mathbb{P})\widehat{V}$, the self-normalized tail inequality cannot be applied directly since $\widehat{V}$ is not well-measurable.
	Prior works fix a value function $V(\cdot)\in\widehat{\mathcal{V}}$, where $\widehat{\mathcal{V}}$ is the function class contains all possible $\widehat{V}$, and build a uniform convergence argument by taking uniform bound over all functions in the $\varepsilon$-covering net $\widehat{\mathcal{N}}_{\varepsilon}$ of $\widehat{\mathcal{V}}$.
	In this way, a self-normalized bound concerning $\widehat{V}$ can be established (refer to proof of Lemma~\ref{lm:barbeta} in Appendix for details).
	However, the covering number of $\widehat{\mathcal{N}}_{\varepsilon}$ highly depends on the feature space dimension, resulting in extra dependency on $d$ in the regret.
	We propose a novel technique of bounding the deviation term by dominant term $(\widehat{\mathbb{P}}-\mathbb{P})V^*$ and the correction term $(\widehat{\mathbb{P}}-\mathbb{P})(\widehat{V}-V^*)$ separately, to remove the extra dependency on $d$.
	Note that bounding the correction term is also required for RL algorithms in tabular MDPs to achieve minimax optimality.
	However, adopting this idea to linear MDPs is nontrivial since we need to build the self-normalized bound for vector-valued martingale other than the well-studied scalar bound in tabular MDPs.
    \vspace{-.2cm}
	\section{Optimal Exploration for linear MDPs}\label{sec:alg}
 \vspace{-.4cm}
	\begin{algorithm}[H]
		\caption{LSVI-UCB$^+$ for Linear MDPs}\label{alg:plus}
		\begin{algorithmic}[1]
			\REQUIRE Regularization parameter $\lambda$, $\widehat{\beta},\widecheck{\beta}$.
			\FOR {step $h=H,...,1$}
			\STATE $\widehat{\mathbf{\Lambda}}_{1,h},\widetilde{\mathbf{\Lambda}}_{1,h}\leftarrow\lambda \mathbf{I}$;
			$\widehat{\boldsymbol{\mu}}_{1,h}\leftarrow\mathbf{0}$
			\ENDFOR
			\STATE $k_0\leftarrow0$
			\FOR {episode $k=1,...,K$}
			\STATE $\widehat{V}_{k,H+1}(\cdot),\widecheck{V}_{k,H+1}(\cdot)\leftarrow 0$
			\FOR {step $h=H,...,1$\textcolor{blue}{~// Optimistic value iteration}}
			\IF{there exists a stage $h'\in[H]$ such that $\det({\widehat{\Lambda}}_{k,h'})\ge2\det({\widehat{\Lambda}}_{k_0,h'}$)}
			{
    \STATE $\widehat{Q}_{k,h}(\cdot,\cdot)\leftarrow\min\{r_h(\cdot,\cdot)+\langle\widehat{\boldsymbol{\mu}}_{k,h}\widehat{\boldsymbol{V}}_{k,h+1},$\\
    $\boldsymbol{\phi}(\cdot,\cdot)\rangle+\widehat{\beta}_{k}\|\boldsymbol{\phi}(\cdot,\cdot)\|_{\widehat{\mathbf{\Lambda}}_{k,h}^{-1}},\widehat{Q}_{k-1,h}(\cdot,\cdot),H\}$
   \STATE $k_0\leftarrow k$\textcolor{blue}{~// Last updating episode}}
			\ELSE
			{\STATE $\widehat{Q}_{k,h}(\cdot,\cdot)\leftarrow\widehat{Q}_{k-1,h}(\cdot,\cdot)$}
			\ENDIF
                \STATE $\widecheck{Q}_{k,h}(\cdot,\cdot)=r_h(\cdot,\cdot)+\langle\widehat{\boldsymbol{\mu}}_{k,h}\widecheck{\boldsymbol{V}}_{k,h+1},\boldsymbol{\phi}(\cdot,\cdot)\rangle-\widecheck{\beta}_{k}\|\boldsymbol{\phi}(\cdot, \cdot)\|_{\widehat{\mathbf{\Lambda}}_{k,h}^{-1}}$ \textcolor{blue}{~// Pessimistic Q function}
                \STATE $\widehat{V}_{k,h}(\cdot)\leftarrow\max_{a\in\mathcal{A}}\widehat{Q}_{k,h}(\cdot,a)$
			\STATE $\widecheck{V}_{k,h}(\cdot)\leftarrow\max\{\max_{a\in\mathcal{A}}\widecheck{Q}_{k,h}(\cdot,a),0\}$
                \STATE $\pi_h^k(\cdot)\leftarrow\arg\max_{a\in\mathcal{A}}\widehat{Q}(\cdot,a)$
			\ENDFOR
			\STATE Receive the initial state $s_1^k$.
			\FOR {step $h=1,...,H$}
			\STATE $a_{h}^{k} \leftarrow \pi_h^{k}(s_{h}^{k})$, and observe $s_{h+1}^{k}\sim\mathbb{P}_h(\cdot|s_h^k,a_h^k)$.
			\STATE \hspace{-.2cm}$\widetilde{\sigma}_{k,h}\leftarrow$\\$\sqrt{\max\{H,Hd^3E_{k,h},[\widehat{\mathbb{V}}_{k,h}\widehat{V}_{k,h+1}](s_h^k,a_h^k)+U_{k,h}\}}$
			\STATE $\widetilde{\mathbf{\Lambda}}_{k+1,h}\leftarrow\widetilde{\mathbf{\Lambda}}_{k,h}+\widetilde{\sigma}_{k,h}^{-2}\boldsymbol{\phi}(s_{h}^{k}, a_{h}^{k})\boldsymbol{\phi}(s_{h}^{k}, a_{h}^{k})^{\top}$
			\IF{$\|\widetilde{\sigma}_{k,h}^{-1}\boldsymbol{\phi}(s_{h}^k,a_{h}^k)\|_{\widetilde{\mathbf{\Lambda}}_{k,{h}}^{-1}}\le1/(H^3d^5)$}
			\STATE \vspace{-.1cm}$\varsigma_{k,h}\leftarrow\sqrt{H}$
			\ELSE
			\STATE $\varsigma_{k,h}\leftarrow H^2\sqrt{d^5}$\textcolor{blue}{~// Enlarge $\varsigma_k$ if necessary}
			\ENDIF
			\STATE \hspace{-.1cm}$\widehat{\sigma}_{k,h}\leftarrow$\\\hspace{-.2cm}$\sqrt{\max\{\varsigma_{k,h}^2,d^3HE_{k,h},[\widehat{\mathbb{V}}_{k,h}\widehat{V}_{k,h+1}](s_h^k,a_h^k)+U_{k,h}\}}$
			\STATE $\widehat{\mathbf{\Lambda}}_{k+1,h}\leftarrow\widehat{\mathbf{\Lambda}}_{k,h}+\widehat{\sigma}_{k,h}^{-2}\boldsymbol{\phi}(s_{h}^{k}, a_{h}^{k})\boldsymbol{\phi}(s_{h}^{k}, a_{h}^{k})^{\top}$
			\STATE $\widehat{\boldsymbol{\mu}}_{k+1,h} \leftarrow\widehat{\mathbf{\Lambda}}_{k+1,h}^{-1}\sum_{i=1}^{k}\widehat{\sigma}_{i,h}^{-2}\boldsymbol{\phi}(s_{h}^{i}, a_{h}^{i})\boldsymbol{\delta}(s_{h+1}^i)^\top$
			\ENDFOR
			\ENDFOR
		\end{algorithmic}
	\end{algorithm}
	\vspace{-.7cm}
 
 In this section, we present the proposed LSVI-UCB$^+$ algorithm (Algorithm~\ref{alg:plus}\footnote{\bug{
  In our original version \cite{hu2022nearly}, there is a technical issue in building the over-optimistic value function $\widedot{V}_{k,h}(\cdot)$ (pointed out by \cite{he2022nearly,agarwal2022vo}) such that the theoretical results do not hold.
In this version, we build on \cite{he2022nearly} by replacing the over-optimistic value function $\widedot{V}_{k,h}(\cdot)$ with the ``rare-switching'' value function first proposed in \cite{he2022nearly}.
In this way, our result still achieves minimax optimal regret for linear MDPs, with minor modifications of constant terms, compared to our original version.}}), where the optimistic value iteration is performed in Lines 5-18, and the learned policy is executed in Line 21.
	The remaining parts of Algorithm~\ref{alg:plus} are responsible for estimating parameter $\boldsymbol{\mu}_{h}$ by linear weighted ridge regression.
	Specifically, the estimated variance is given in Line 29, whose lower bound is controlled in Lines 24-28, and the solution to the regression is given in Line 31.
    
	LSVI-UCB$^+$ is an optimistic algorithm similar to existing works \cite{yang2019sample, jin2020provably, ayoub2020model}, but upgrading the Hoeffding-type bonus to a carefully designed Bernstein-type one.
	The exploration radius in LSVI-UCB$^+$ is proportional to the standard deviation of the optimal value function conditioned on some state-action pair, which accounts for two key novelties of LSVI-UCB$^+$:
	
	(\romannumeral1) We replace the linear ridge regression in prior works \cite{yang2019sample, jin2020provably, ayoub2020model} with a carefully designed weighted version such that LTV can be applied.
	Note that the linear weighted ridge regression estimator was originally built for linear bandits with heteroscedastic noises, e.g., \cite{lattimore2015linear, kirschner2018information}. 
	Besides, the regression is performed to estimate $\boldsymbol{\mu}_h$, i.e., transition matrix, instead of estimating indirect variables, e.g., $\boldsymbol{\omega}_h^*$ of LSVI-UCB in \cite{jin2020provably}.
	
	(\romannumeral2) 
	A variance estimator, based on the estimated parameter $\widehat{\boldsymbol{\mu}}_{k,h}$, is built for the optimal value function to determine the weights in regression.
	UCRL-VTR$^+$ in \cite{zhou2021nearly} also introduces weighted ridge regression for linear mixture MDPs, and weights are determined by variances of the constructed optimistic value function. 
	However, the weights in LSVI-UCB$^+$ are very different from those in UCRL-VTR$^+$, since our variances are estimated with respect to the optimal value function, not the constructed value function.
	\vspace{-.2cm}
	\subsection{Linear Weighted Ridge Regression}

	Denote $\boldsymbol{\delta}(s)\in\mathbb{R}^{|\mathcal{S}|}$ as a one-hot vector that is zero everywhere except that the entry corresponding to state $s$ is one, and define $\boldsymbol{\epsilon}_{h}^{k}:=\mathbb{P}_h(\cdot \mid s_{h}^{k}, a_{h}^{k})-\boldsymbol{\delta}(s_{h+1}^{k})$.
	Since $\mathbb{E}[\boldsymbol{\epsilon}_{h}^{k} \mid \mathcal{F}_{k,h}]=0$, $\boldsymbol{\delta}(s_{h+1}^{k})$ is an unbiased estimate of $\mathbb{P}_{h}(\cdot\mid s_h^k,a_h^k)=\boldsymbol{\mu}_{h}^\top\boldsymbol{\phi}(s_h^k, a_h^k)$.
	Thus, $\boldsymbol{\mu}_{h}$ can be learned via regression from $\boldsymbol{\phi}(s_{h}^{k}, a_{h}^{k})$ to $\boldsymbol{\delta}(s_{h+1}^{k})$.
	In addition, samples are normalized by the estimated standard deviation $\widehat{\sigma}_{k,h}$.
	Thus, the estimated parameter $\widehat{\boldsymbol{\mu}}_{k,h}$ in Line 31 of Algorithm~\ref{alg:plus} is the solution to the following weighted ridge regression problem:
    \vspace{-.3cm}
	\begin{align*}
	    \operatorname{min}_{\boldsymbol{\mu} \in \mathbb{R}^{d\times|\mathcal{S}|}} \sum_{i=1}^{k-1}\left\|\left[\boldsymbol{\mu}_{h}^\top\boldsymbol{\phi}(s_h^k, a_h^k)-\boldsymbol{\delta}(s_{h+1}^{i})\right]\widehat{\sigma}_{i, h}^{-1}\right\|_{2}^{2}+\lambda\|\boldsymbol{\mu}\|_{F}^{2},
	\end{align*}
	where $\|\cdot\|_{F}$ denotes Frobenius norm. The solution is
	\begin{align}\label{eq:solution}
		\widehat{\boldsymbol{\mu}}_{k,h}=\widehat{\mathbf{\Lambda}}_{k,h}^{-1}\sum_{i=1}^{k-1}\widehat{\sigma}_{i, h}^{-2}\boldsymbol{\phi}(s_{h}^{i}, a_{h}^{i})\boldsymbol{\delta}(s_{h+1}^{i})^{\top},
	\end{align}
	where $\widehat{\mathbf{\Lambda}}_{k,h}=\sum_{i=1}^{k-1} \widehat{\sigma}_{i, h}^{-2}\boldsymbol{\phi}(s_{h}^{i}, a_{h}^{i}) \boldsymbol{\phi}(s_{h}^{i}, a_{h}^{i})^{\top}+\lambda \mathbf{I}$.
	Thus, the estimated transition probability is denoted as
	\begin{align*}
	    \widehat{\mathbb{P}}_{k,h}(\cdot\mid s,a)=\widehat{\boldsymbol{\mu}}_{k,h}^\top\boldsymbol{\phi}(s,a)
	\end{align*}
	for any $(s,a)\in\mathcal{S}\times\mathcal{A}$.
	After estimating the transition matrix, Lines 8-13 in Algorithm~\ref{alg:plus}
	constructs an optimistic state-action function, which is equivalent to
	\begin{align*}
		\widehat{Q}_{k,h}(\cdot,\cdot)=&\min\{\max_{\boldsymbol{\mu}\in\widehat{\mathcal{C}}_{k_0,h}}r_h(\cdot,\cdot)+\langle\boldsymbol{\mu}\widehat{\boldsymbol{V}}_{k_0,h+1},\boldsymbol{\phi}(\cdot, \cdot)\rangle,H\},
	\end{align*}
	where $k_0$ is the updating episode and the optimistic confidence set is given by
	\begin{align*}
		\widehat{\mathcal{C}}_{k_0, h}:=\big\{\boldsymbol{\mu}:\|(\boldsymbol{\mu}-\widehat{\boldsymbol{\mu}}_{k_0, h})\widehat{\boldsymbol{V}}_{k_0,h+1}\|_{\widehat{\boldsymbol{\Lambda}}_{k_0, h}} \leq \widehat{\beta}\big\},
	\end{align*}
	and $\widehat{\beta}$ is the exploration radius.
    \bug{The construction of the optimistic state-action function in Lines 8-13 and the updating condition in Line 8 of Algorithm~\ref{alg:plus} are proposed by \cite{he2022nearly}, which utilize a ``rare switching" mechanism (detailed in Lemma~\ref{lm:numupdate} and \ref{lm:coverhatv} in Appendix) to ensure a small covering number of considered optimistic value function classes.
    This ``rare switching" mechanism avoids the issue of building the over-optimistic value function in our original version \cite{hu2022nearly}.}
    In addition, the pessimistic state-action function in Line 14 is equivalent to
	\begin{align*}
		\widecheck{Q}_{k,h}(\cdot, \cdot)=\min_{\boldsymbol{\mu}\in\widecheck{\mathcal{C}}_{k,h}}r_h(\cdot,\cdot)+\langle\boldsymbol{\mu}\widecheck{\boldsymbol{V}}_{k,h+1},\boldsymbol{\phi}(\cdot, \cdot)\rangle,
	\end{align*}
	where the pessimistic confidence set is 
	\begin{align*}
		\widecheck{\mathcal{C}}_{k, h}:=\big\{\boldsymbol{\mu}:\|(\boldsymbol{\mu}-\widehat{\boldsymbol{\mu}}_{k, h})\widecheck{\boldsymbol{V}}_{k,h+1}\|_{\widecheck{\boldsymbol{\Lambda}}_{k, h}} \leq \widecheck{\beta}\big\}.
	\end{align*}
	Subsequently, optimistic value function $\widehat{V}_{k,h}(\cdot)$ and pessimistic value function $\widecheck{V}_{k,h}(\cdot)$ can be defined.
	Note that $\widehat{V}_{k,h}(\cdot)$ in Algorithm~\ref{alg:plus} is strictly decreasing in $k$, which ensures that the optimistic value function approaches the optimal value function $V_{h}^*(\cdot)$ almost surely.
	Besides, the pessimistic value function $\widecheck{V}_{k,h}(\cdot)$ is required for estimating the variance upper bound later.
	\vspace{-.2cm}
	\subsection{Variance Estimation}
	After estimating the transition matrix in Eq.~(\ref{eq:solution}), LSVI-UCB$^+$ estimates the variance of the optimal value function $[\mathbb{V}_{h}V_{h+1}^*](s_h^k,a_h^k)$ and the variance of sub-optimality gap $[\mathbb{V}_{h}(\widehat{V}_{k,h+1}-V_{h+1}^*)](s_h^k,a_h^k)$.
	This is a major difference with prior UCRL-VTR$^+$ algorithm in \cite{zhou2021nearly} for linear mixture MDPs, which only estimates the variance of the constructed optimistic value function $[\mathbb{V}_{h}\widehat{V}_{k,h+1}](s_h^k,a_h^k)$.
	The purpose to estimate these two variances remains that we utilize Bernstein self-normalized tail inequality in Theorem~\ref{th:self} to bound the dominant term $[(\widehat{\mathbb{P}}_{k,h}-\mathbb{P}_{h})V^*_{h+1}](s_h^k,a_h^k)$ and the correction term $[(\widehat{\mathbb{P}}_{k,h}-\mathbb{P}_{h})(\widehat{V}_{k,h+1}-V^*_{h+1})](s_h^k,a_h^k)$ separately to remove the extra dependency of regrets, such that we need to estimate these two variance, which are illusated below.
	\vspace{-.2cm}
	\paragraph{Variance of Optimal Value Function}
	We first consider the case where the transition matrix and optimal value function $V_{h+1}^*(\cdot)$ were given. In this case, the variance of the optimal value function is given by 
	\begin{align*}
	    [\mathbb{V}_{h}V_{h+1}^*](s_{h}^{k},a_{h}^{k})=\mathbb{P}_h{V_{h+1}^*}^2(s_{h}^{k},a_{h}^{k})-[\mathbb{P}_hV_{h+1}^*(s_{h}^{k},a_{h}^{k})]^2
	\end{align*}
	
	However, only empirical estimation $\widehat{\boldsymbol{\mu}}_{k,h}$ and  optimistic value function $\widehat{V}_{k,h}$ are obtainable, which means we only have the empirical variance of the optimistic value function:
	\begin{align}\label{eq:hatvar}
	    \begin{split}
	        [\widehat{\mathbb{V}}_{k, h} \widehat{V}_{k, h+1}](s_{h}^k, a_{h}^k)=\widehat{\mathbb{P}}_{k,h}\widehat{V}_{k,h+1}^2(s_{h}^{k},a_{h}^{k})_{[0,H^2]}\\
	    -[\widehat{\mathbb{P}}_{k,h}\widehat{V}_{k,h+1}(s_{h}^{k},a_{h}^{k})_{[0,H]}]^2.
	    \end{split}
	\end{align}
	To ensure the accuracy of the estimation, we introduce an offset term $U_{k, h}$ to guarantee that $|[\mathbb{V}_{h}V_{h+1}^*](s_{h}^{k},a_{h}^{k})-[\widehat{\mathbb{V}}_{k, h} \widehat{V}_{k, h+1}](s_{h}^k, a_{h}^k)|\le U_{k,h}$ with high probability.
	Moreover, the exact form of offset term $U_{k, h}$ is specified in Lemma~\ref{lm:cs}, which requires accessing the pessimistic value functions as detailed in Lemma~\ref{lm:var} in Appnedix.
	
	\vspace{-.2cm}
	\paragraph{Variance of Sub-optimality Gap}
	In particular, we try to build a upper bound for the variance of the sub-optimality gap, which is given as
	\begin{align*}
	    &[\mathbb{V}_{h}(\widehat{V}_{k,h+1}-V_{h+1}^*)](s_{h}^{k},a_{h}^{k})\\
	    =&\resizebox{\columnwidth}{!}{$[\mathbb{P}_h(\widehat{V}_{k,h+1}-V_{h+1}^*)^2](s_{h}^{k},a_{h}^{k})-\left[[\mathbb{P}_h(\widehat{V}_{k,h+1}-V_{h+1}^*)](s_{h}^{k},a_{h}^{k})\right]^2$}\\
	    \le&\resizebox{\columnwidth}{!}{$[\mathbb{P}_h(\widehat{V}_{k,h+1}-V_{h+1}^*)^2](s_{h}^{k},a_{h}^{k})\le H[\mathbb{P}_h(\widehat{V}_{k,h+1}-V_{h+1}^*)](s_{h}^{k},a_{h}^{k})$}\\
	    \le&H[\mathbb{P}_h(\widehat{V}_{k,h+1}-\widecheck{V}_{k,h+1})](s_{h}^{k},a_{h}^{k})
	\end{align*}
	where the second and last inequalities holds by the optimism and pessimism of $\widehat{V}_{k,h+1}$ and $\widecheck{V}_{k,h+1}$, respectively.
	Thus, it suffices to upper bound the deviation $[\mathbb{P}_h(\widehat{V}_{k,h+1}-\widehat{V}_{k,h+1})](s_{h}^{k},a_{h}^{k})$.
	In addition, the upper bound of the variance of the sub-optimality gap is denoted as $E_{k,h}$, specified in Lemma~\ref{lm:cs}.
	
	Putting two variances together, the weight $\widehat{\sigma}_{k,h}$ in Algorithm~\ref{alg:plus} is given by
	\begin{align*}
		\widehat{\sigma}_{k,h}=\sqrt{\max \{\varsigma_{k,h}^2,dE_{k,h},[\widehat{\mathbb{V}}_{k, h} \widehat{V}_{k, h+1}](s_{h}^{k}, a_{h}^{k})+U_{k, h}\}},
	\end{align*}
	which is the maximum over the weight lower bound $\varsigma_{k,h}$, the variance upper bound of the optimal value function, the variance of the sub-optimality gap with a factor $d$ scaling.
	Here $\varsigma_{k,h}$ controls the lower bound of $\widehat{\sigma}_{k,h}$ and is dynamically determined in Lines 24-28.
	In particular, we try to keep the magnitude of the considered MDS to be small by adaptively enlarging $\varsigma_{k,h}$, which is detailed in Remark~\ref{rm:small}.
	\vspace{-.2cm}
	\section{Main Results}\label{sec:results}
	This section presents the results of the statistical, space, and computational complexities of the LSVI-UCB$^+$ algorithm.
	In particular, LSVI-UCB$^+$ reaches nearly minimax optimal regret in linear MDPs, while the space and computational complexities are no worse than prior works.
	\vspace{-.2cm}
	\subsection{Statistical Complexity}
	We first present the regret upper bound of   LSVI-UCB$^+$ in Theorem~\ref{th:regkr}.
	\begin{theorem}[Regret Upper Bound]\label{th:regkr}
		Set $\lambda=1/(H^2\sqrt{d})$.
		Then, with probability at least $1-10\delta$, the regret of LSVI-UCB$^{+}$ is upper bounded by
		\begin{align}\label{eq:regret}
			\begin{split}
				\operatorname{Regret}(K)=&\widetilde{O}\left(Hd\sqrt{T}+H^6d^9\right),
			\end{split}
		\end{align}
		where $T=KH$.
		\vspace{-.5cm}
		\begin{proof}[Proof Sketch]
			We prove the result conditioning on the conclusion of Lemma~\ref{lm:cs}.
			Initially, with the standard regret decomposition, we can show that the total regret is bounded by the summation of the exploration bonus, i.e.,
			\vspace{-.2cm}
			\begin{align}\label{eq:regd}
				\begin{split}
					&\operatorname{Regret}(K)\le\sum_{k=1}^{K}\sum_{h=1}^{H}\widehat{\beta}\|\boldsymbol{\phi}(s_h^k,a_h^k)\|_{\widehat{\mathbf{\Lambda}}_{k,h}^{-1}}\\
					=&\sum_{k=1}^{K}\sum_{h=1}^{H}\widehat{\beta}\widehat{\sigma}_{k,h}\|\widehat{\sigma}_{k,h}^{-1}\boldsymbol{\phi}(s_h^k,a_h^k)\|_{\widehat{\mathbf{\Lambda}}_{k,h}^{-1}}\\	\le&\widehat{\beta}\underbrace{\sqrt{\sum_{k=1}^{K}\sum_{h=1}^{H}\widehat{\sigma}_{k,h}^{2}}}_{\widetilde{O}\left(\sqrt{HT+c\sqrt{T}}\right)}\underbrace{\sqrt{\sum_{k=1}^{K}\sum_{h=1}^{H}\|\widehat{\sigma}_{k,h}^{-1}\boldsymbol{\phi}(s_h^k,a_h^k)\|_{\widehat{\mathbf{\Lambda}}_{k,h}^{-1}}^2}}_{\widetilde{O}\left(\sqrt{Hd}\right)}
				\end{split}
			\end{align}
			where the second inequality holds by Cauchy-Schwarz inequality and $c$ is a constant.
			The summation of $\|\widehat{\sigma}_{k,h}^{-1}\boldsymbol{\phi}(s_h^k,a_h^k)\|_{\widehat{\mathbf{\Lambda}}_{k,h}^{-1}}^2$ can be addressed by Elliptical Potential Lemma (Lemma~\ref{lm:ablog} in Appendix), and the summation of $\widehat{\sigma}_{k,h}^2$ can be bounded by
			\vspace{-.2cm}
			\begin{align*}
					\sum_{k=1}^{K}\sum_{h=1}^{H}\widehat{\sigma}_{k,h}^{2}\le\sum_{k=1}^{K}\sum_{h=1}^{H}\varsigma_{k,h}^2 +\sum_{k=1}^{K}\sum_{h=1}^{H}\left[dE_{k,h}+U_{k,h}\right]\\
					+\sum_{k=1}^{K}\sum_{h=1}^{H}[\widehat{\mathbb{V}}_{k, h} \widehat{V}_{k,h+1}](s_{h}^k,a_{h}^k)\le\widetilde{O}\left(HT+c\sqrt{T}\right),
			\end{align*}
			where the first inequality holds by definition of $\widehat{\sigma}_{k,h}$, and the second inequality holds by
			$\sum_{k=1}^{K}\sum_{h=1}^{H}\varsigma_{k,h}^2\le\widetilde{O}(HT)$ due to the conservatism of elliptical potentials,
			$\sum_{k=1}^{K}\sum_{h=1}^{H}[dE_{k,h}+U_{k,h}]\le\widetilde{O}(c\sqrt{T})$ due to the Elliptical Potential Lemma,
			and $\sum_{k=1}^{K} \sum_{h=1}^{H}[\widehat{\mathbb{V}}_{k,h} \widehat{V}_{k,h+1}](s_{h}^{k}, a_{h}^{k})\le\widetilde{O}(HT)$ due to the LTV.
			Besides, the exploration radius $\widehat{\beta}=\widetilde{O}(\sqrt{d})$, which determined by the upper bound of $(\widehat{\mathbb{P}}_{k,h}-\mathbb{P}_{h})\widehat{V}_{k,h+1}(s_h^k,a_h^k)$, detailed in Section~\ref{sec:corvering}.
			The full proof is given in Appendix~\ref{sec:appregret}.
		\end{proof}
	\end{theorem}
	Theorem~\ref{th:regkr} is proved under the event that the optimistic confidence set $\widehat{\mathcal{C}}_{k,h}$ holds, which is built in Lemma~\ref{lm:cs}.
	In addition, we find that the exploration radius $\widehat{\beta}$ of the optimistic confidence set $\widehat{\mathcal{C}}_{k,h}$ determines the sharpness of the final regret, as shown in Eq.~(\ref{eq:regd}) .
	
	\begin{remark}
		When\footnote{Large-sample regime conditions are required in many RL algorithms to obtain satisfactory statistical complexities, e.g. UCRL-VTR$^+$ in \cite{zhou2021nearly} requires $T \geq d^{4} H^{2}+d^{3} H^{3}$.}
		$T \ge H^{10}d^{16}$, the regret in Eq. $(\ref{eq:regret})$ can be simplified to $\widetilde{O}(Hd\sqrt{T})$, which improves the $\widetilde{O}\left(\sqrt{H^3d^3T}\right)$ regret  of LSVI-UCB
		\cite{jin2020provably} by a factor of $\sqrt{Hd}$.
		Moreover, our algorithm design an analytical tools including Theorem~\ref{th:self} and Lemma~\ref{lm:ep} in next sections can further improve the regret bound of UCRL-VTR$^+$ in \cite{zhou2021nearly} for linear mixture MDPs to $\widetilde{O}(Hd\sqrt{T})$ from existing  $\widetilde{O}(\sqrt{H^2d^2T+H^3dT})$, such that it is minimax optimal up to logarithmic factor without large dimension assumption that $d\ge H$ in \cite{zhou2021nearly}.
	\end{remark}
	\vspace{-.2cm}
	\paragraph{Lower Bound}
	We formalize a linear MDP instance in Appendix~\ref{sec:applower} to establish an $\Omega(Hd\sqrt{T})$ regret lower bound of linear MDPs.
	This linear MDP instance is firstly proposed in Remark 23 in \cite{zhou2021nearly}, which shares the same regret lower bound of a linear mixture MDP instance.
	This class of MDP is hard due to the intrinsical sparsity of reward and indistinguishability of large action space, which can be regarded as an extension of hard instances in linear bandits literature \cite{dani2008stochastic,lattimore2020bandit}.
	According to Theorem 8 in \cite{zhou2021nearly}, linear mixture MDPs have regret lower bound of $\Omega(Hd\sqrt{T})$. Thus linear MDPs have the same regret lower bound.
	The lower bound, together with the upper bound of LSVI-UCB$^{+}$ in Theorem~\ref{th:regkr} show that LSVI-UCB$^{+}$ is minimax optimal up to logarithmic factors when $T \geq\max\{H^4d^{10},H^5d^6\}$.
	
	\subsection{Space and Computational Complexities}
	As stated above, LSVI-UCB$^+$ reaches minimax optimal regret up to logarithmic factors, which is also computationally efficient.
	In particular, the space and computational complexities of LSVI-UCB$^+$ are briefly stated below, which are both the same as LSVI-UCB in \cite{jin2020provably}.
    \vspace{-.2cm}
	\paragraph{Space Complexity}
	Although $\widehat{\boldsymbol{\mu}}_{k,h}\in\mathbb{R}^{d\times|\mathcal{S}|}$ and $|\mathcal{S}|$ can be infinitely large, we do not store it explicitly, as Algorithm~\ref{alg:plus} only utilizes indirect variables $\widehat{\boldsymbol{\mu}}_{k,h}\boldsymbol{V}=\widehat{\mathbf{\Lambda}}_{k,h}^{-1}\sum_{i=1}^{k-1}\widehat{\sigma}_{i, h}^{-2}\boldsymbol{\phi}(s_{h}^{i}, a_{h}^{i})V(s_{h+1}^{i})$ where $V\in\{\widehat{V}_{k,h+1},\widehat{V}^2_{k,h+1},\widecheck{V}_{k,h+1}\}$.
	In episode $k\in[K]$, Algorithm~\ref{alg:plus} only stores $\widehat{\boldsymbol{\mu}}_{k,h}\boldsymbol{V}$, $\widehat{\mathbf{\Lambda}}_{k,h}$, $\widetilde{\mathbf{\Lambda}}_{k,h}$, $\widehat{\sigma}_{k,h}$, $\widetilde{\sigma}_{k,h}$ for all $h\in[H]$, and $\{\boldsymbol{\phi}(s_h^{k'},a)\}_{a\in\mathcal{A}}$ for all $(k',h)\in[k]\times[H]$, which means LSVI-UCB$^+$ requires $O(d^2H+d|\mathcal{A}|T)$ space.
	\vspace{-.2cm}
	\paragraph{Computational Complexity}
	Assume $\widehat{\boldsymbol{\mu}}_{k,h}\boldsymbol{V}$ is given for some $V\in\{\widehat{V}_{k,h+1},\widehat{V}^2_{k,h+1},\widecheck{V}_{k,h+1}\}$, then each evaluation of $V(s)$ takes $O(d^2|\mathcal{A}|)$ operations. Thus, calculating $\widehat{\boldsymbol{\mu}}_{k,h}\boldsymbol{V}$ takes $O(d^2|\mathcal{A}|)K$ operations.
	Besides, $\widehat{\mathbf{\Lambda}}_{k,h},\widetilde{\mathbf{\Lambda}}_{k,h}$ can be computed by Sherman-Morrison formula \cite{hager1989updating} with $O(d^2)$ operations and other steps take less operations.
	Thus, LSVI-UCB$^+$ has a running time of $O(d^2|\mathcal{A}|KT)$, which is computationally efficient since its running time is polynomial on $d,K,H,|\mathcal{A}|$, and does not depend on $|\mathcal{S}|$, which can be possibly infinite.
	
	\section{Mechanism Towards Minimax Optimality}\label{sec:tow}
	In this section, we highlight our technical contributions in building the sharp optimistic confidence set $\widehat{\mathcal{C}}_{k,h}$.
	We first present two novel analytical tools, a sharp Bernstein self-normalized tail inequality for vector-valued martingales in Section~\ref{sec:bernbound}, the conservatism of elliptical potentials in Section~\ref{sec:ep}.
	Together, these two analytical tools remove the additional dependency of regret on $H$ in the regret of the LSVI-UCB$^+$ algorithm.
	In addition, we also upper bounds the correction term of the form $(\widehat{\mathbb{P}}-\mathbb{P})(\widehat{V}-V^*)$ to avoid extra cost from the covering net such that the additional dependency of regret on $d$ is removed as well.
	Consequently, the sharp confidence set $\widehat{\mathcal{C}}_{k,h}$ is built in Lemma~\ref{lm:cs} in Section~\ref{sec:corvering}.
	These technical contributions together enable LSVI-UCB$^+$ to achieve nearly minimax optimal regret and have the potential to improve other statistical results of algorithms for RL with linear function approximation.
	
	\subsection{Sharp Bernstein Self-normalized Bound}\label{sec:bernbound}
	Most existing self-normalized concentrations used in prior works for RL with linear function approximation \cite{jin2020provably,wang2020optimism,wang2020reinforcement,ayoub2020model} are all Hoeffding-type, i.e., they consider sub-Gaussian noises.
	Our self-normalized bound below considers sub-exponential noises, which is a Bernstein-type one.
	
	\begin{theorem}[Bernstein self-normalized bound]\label{th:self}
		Let $\{\mathcal{G}_{t}\}_{t=1}^{\infty}$ be a filtration, and $\{\boldsymbol{x}_{t}, \eta_{t}\}_{t \geq 1}$ be a stochastic process such that $\boldsymbol{x}_{t} \in \mathbb{R}^{d}$ is $\mathcal{G}_{t}$-measurable and $\eta_{t} \in \mathbb{R}$ is $\mathcal{G}_{t+1}$-measurable.
		Define $\mathbf{Z}_{t}=\lambda \mathbf{I}+\sum_{i=1}^{t} \boldsymbol{x}_{i} \boldsymbol{x}_{i}^{\top}$ for $t\ge1$ and $\mathbf{Z}_{0}=\lambda \mathbf{I}$.
		If $\|\boldsymbol{x}_t\|_2\le L$, and $\eta_{t}$ satisfies $\mathbb{E}[\eta_{t} \mid \mathcal{G}_{t}]=0$, $\mathbb{E}[\eta_{t}^{2} \mid \mathcal{G}_{t}] \leq \sigma^{2}$, and	$\quad|\eta_t\cdot\min\{1,\|\boldsymbol{x}_t\|_{\mathbf{Z}_{t-1}^{-1}}\}|\le R$ for all $t\ge1$. Then, for any $0<\delta<1$, with probability at least $1-\delta$, we have:
		\begin{align*}
		    \forall t>0,\left\|\sum_{i=1}^{t} \boldsymbol{x}_{i} \eta_{i}\right\|_{\mathbf{Z}_{t}^{-1}}\leq\widetilde{O}(\sigma\sqrt{d}+R).
		\end{align*}
		\begin{proof}
			Please refer to Appendix~\ref{sec:appself}.
		\end{proof}
	\end{theorem}
	
	\begin{remark}\label{rm:th}
		The proof of Theorem~\ref{th:self} in Appendix B shows that bounding the self-normalized vector-valued martingales is equivalent to bounding a scalar-valued MDS $\{\eta_t\cdot\min\{1,\|\boldsymbol{x}_t\|_{\mathbf{Z}_{t-1}^{-1}}\},\mathcal{G}_{t+1}\}$, where $\eta_{t}$ is scaled by the factor of $\min\{1,\|\boldsymbol{x}_t\|_{\mathbf{Z}_{t-1}^{-1}}\}$.
		In particular, $\|\boldsymbol{x}_t\|_{\mathbf{Z}_{t-1}^{-1}}=\sqrt{\boldsymbol{x}_t^\top\mathbf{Z}_{t-1}^{-1}\boldsymbol{x}_t}$ is denoted as the elliptical potential, which is common in online learning literature \cite{cesa2006prediction}.
		Notice that Elliptical Potential Lemma shows that $\min\{1,\|\boldsymbol{x}_t\|_{\mathbf{Z}_{t-1}^{-1}}\}$ can be roughly regarded as an attenuated sequence.
		Theorem~\ref{th:self} looks similar to but is sharper than Theorem 2 in \cite{zhou2021nearly}, because it pay extra attentions on elliptical potentials $\|\boldsymbol{x}_t\|_{\mathbf{Z}_{t-1}^{-1}}$.
		However, the scaling factor $\min\{1,\|\boldsymbol{x}_t\|_{\mathbf{Z}_{t-1}^{-1}}\}$ is crudely deflated to $1$ in Theorem 2 in \cite{zhou2021nearly}, such that the attenuation of the MDS is neglected, which is highlighted in Lemma~\ref{lm:od}.
	\end{remark}
	
	\subsection{Conservatism of Elliptical Potentials}\label{sec:ep}
	Notice the self-normalized bound in Theorem~\ref{th:self} will determine the order of exploration radius $\widehat{\beta}$.
	We try to keep the second term $R$, the magnitude of the MDS, in Theorem~\ref{th:self} smaller than the first $\sigma\sqrt{d}$ by utilizing the conservatism of elliptical potentials.
	Specifically, the following lemma characterizes the conservatism of elliptical potentials, i.e., elliptical potentials are usually small.
	This lemma is firstly proposed at Exercise 19.3 in \cite{lattimore2020bandit} for case $c=1$, and we generalize it to case $c>0$.
	
	\begin{lemma}[Elliptical Potentials are Usually Small]\label{lm:ep}
		Given $\lambda>0$ and sequence $\left\{\boldsymbol{x}_{t}\right\}_{t=1}^{T} \subset$ $\mathbb{R}^{d}$ with $\left\|\boldsymbol{x}_{t}\right\|_{2} \leq L$ for all $t\in[T]$, define $\mathbf{Z}_{t}=\lambda \mathbf{I}+\sum_{i=1}^{t} \boldsymbol{x}_{i} \boldsymbol{x}_{i}^{\top}$ for $t\ge1$ and $\mathbf{Z}_{0}=\lambda \mathbf{I}$. The number of times $\left\|\boldsymbol{x}_{t}\right\|_{\mathbf{Z}_{t-1}^{-1}}\geq c$ is at most
		$$
		\frac{3 d}{\log (1+c^2)} \log \left(1+\frac{L^{2}}{\lambda \log (1+c^2)}\right)
		$$
	    for any $t\in[T]$, where $c>0$ is a constant.
		\begin{proof}
			Please refer to Lemma~\ref{lm:od} in Appendix. 
		\end{proof}
	\end{lemma}
	On the one hand, for some stage $h\in[H]$, the noise $\eta_{k}=\widehat{\sigma}_{k,h}^{-1}\boldsymbol{V}^\top\boldsymbol{\epsilon}_h^k$ for some value function $V$ as detailed in Appendix~\ref{sec:apphpe}.
	On the other hand, in Theorem~\ref{th:self}, $R$ is the absolute bound of $|\eta_{k}\cdot\min\{1,\|\widehat{\sigma}_{k,h}^{-1}\boldsymbol{\phi}(s_h^k,a_h^k)\|_{\widehat{\mathbf{\Lambda}}_{k,h}^{-1}}\}|$.
	Lemma~\ref{lm:ep} reveals that $R$ is intrinsically small since the elliptical potential $\|\widehat{\sigma}_{k,h}^{-1}\boldsymbol{\phi}(s_h^k,a_h^k)\|_{\widehat{\mathbf{\Lambda}}_{k,h}^{-1}}^2$ is small in most episodes.
	In addition, we only need to enlarge the lower bound of $\widehat{\sigma}_{k,h}$,i.e., $\varsigma_{k,h}$, when $\|\widehat{\sigma}_{k,h}^{-1}\boldsymbol{\phi}(s_h^k,a_h^k)\|_{\widehat{\mathbf{\Lambda}}_{k,h}^{-1}}^2$ is large such that $R$ can remain small uniformly, which is detailed in the following remark.
	
	\begin{remark}\label{rm:small}
		Lines 22-29 of Algorithm~\ref{alg:plus} ensure the following facts for any $(k,h)\in[K]\times[H]$ by introducing indicator variable $\|\widetilde{\sigma}_{i, h}^{-1}\boldsymbol{\phi}(s_{h}^i,a_{h}^i)\|_{\widetilde{\mathbf{\Lambda}}_{i,h}^{-1}}$:
		
		(\romannumeral1) In most cases, we have $\|\widetilde{\sigma}_{i, h}^{-1}\boldsymbol{\phi}(s_{h}^i,a_{h}^i)\|_{\widetilde{\mathbf{\Lambda}}_{i,h}^{-1}}\le1/(Hd^3)$, then $\varsigma_{k,h}=\sqrt{H}$ such that $\widehat{\sigma}_{i,h}=\widetilde{\sigma}_{i,h}$.
		We can prove that the elliptical potential $\|\widehat{\sigma}_{k,h}^{-1}\boldsymbol{\phi}(s_h^i,a_h^i)\|_{\widehat{\mathbf{\Lambda}}_{k,h}^{-1}}$ is small;
		
		(\romannumeral2) Otherwise, $\varsigma_{k,h}=H\sqrt{d^3}$, such that $\widehat{\sigma}_{k,h}\ge H\sqrt{d^3}$.
		In this case, the $R$ is still small since $\widehat{\sigma}_{k,h}$ is large.
			
		Notice that always enlarging $\varsigma_{k,h}$ for any $k\in[K]$ is a simple method to keep $R$ small, but contributing to final regret linearly since it is an additive term in $\sum_{k=1}^K\sum_{h=1}^H\widehat{\sigma}_{k,h}$.
		Nevertheless, the enlarging operation in case (\romannumeral2) only contributes an additive constant term to the regret, since the elliptical potential $\|\widetilde{\sigma}_{k,h}^{-1}\boldsymbol{\phi}(s_{h}^k,a_{h}^k)\|_{\widetilde{\mathbf{\Lambda}}_{k,h}^{-1}}^2$ is small in most episodes such that the enlarging operation in case (\romannumeral2) happens rarely.
	\end{remark}
	
	\vspace{-.2cm}
	As a consequence, the $R$ in Theorem~\ref{th:self} can be controlled to be smaller than the $\sigma\sqrt{d}$ in the LSVI-UCB$^+$ algorithm, such that the exploration radius $\widehat{\beta}$ in LSVI-UCB$^+$ is $\widetilde{O}(\sqrt{d})$.
	However, analog Bernstein self-normalized bounds, such as Theorem 2 in \cite{zhou2021nearly} and Theorem 1 in \cite{faury2020improved}, cannot lead to such exploration radius in the SVI-UCB$^+$ algorithm, while that of Theorem 2 in \cite{zhou2021nearly} is $\widetilde{O}(\sqrt{d}+\sqrt{H}d^2)$, and Theorem 1 in \cite{faury2020improved} is $\widetilde{O}(\sqrt{H\sqrt{d^5}})$.
	
	\subsection{Building Confidence Set with Correction Term}\label{sec:corvering}
	This subsection explains critical steps of building a sharp optimistic confidence set $\widehat{\mathcal{C}}_{k,h}$ with the correction term.
	Specifically, the exploration bonus $\widehat{\beta}\|\boldsymbol{\phi}(s_h^k,a_h^k)\|_{\widehat{\mathbf{\Lambda}}_{k,h}^{-1}}$ is the upper bound of the deviation term $[(\widehat{\mathbb{P}}_{k,h}-\mathbb{P}_{h})\widehat{V}_{k,h+1}](s_h^k,a_h^k)$, which is the basis of optimistic learning.
	Specifically, the deviation term can be decomposed by triangle inequality as the sum of dominant term and correction term:
	\vspace{-.3cm}
	\begin{align*}
			&\underbrace{[(\widehat{\mathbb{P}}_{k,h}-\mathbb{P}_{h})V_{h+1}^*](s_h^k,a_h^k)}_{\text{Dominant term with respect to}\widehat{V}_{h+1}^*}\\
			&+\underbrace{[(\widehat{\mathbb{P}}_{k,h}-\mathbb{P}_{h})(\widehat{V}_{k,h+1}-V_{h+1}^*)](s_h^k,a_h^k)}_{\text{Correction Term}}\\
			\le&\|(\boldsymbol{\mu}-\widehat{\boldsymbol{\mu}}_{k,h}){\boldsymbol{V}}_{h+1}^*\|_{\widehat{\boldsymbol{\Lambda}}_{k,h}}\|\boldsymbol{\phi}(s_h^k,a_h^k)\|_{\widehat{\mathbf{\Lambda}}_{k,h}^{-1}}\\
			&+\|(\boldsymbol{\mu}-\widehat{\boldsymbol{\mu}}_{k,h})(\widehat{\boldsymbol{V}}_{k,h+1}-{\boldsymbol{V}}_{h+1}^*)\|_{\widehat{\boldsymbol{\Lambda}}_{k,h}}\|\boldsymbol{\phi}(s_h^k,a_h^k)\|_{\widehat{\mathbf{\Lambda}}_{k,h}^{-1}}\\
			\le&\widehat{\beta}^{(1)}\|\boldsymbol{\phi}(s_h^k,a_h^k)\|_{\widehat{\mathbf{\Lambda}}_{k,h}^{-1}}+\widehat{\beta}^{(2)}\|\boldsymbol{\phi}(s_h^k,a_h^k)\|_{\widehat{\mathbf{\Lambda}}_{k,h}^{-1}}\\
			=&\widehat{\beta}\|\boldsymbol{\phi}(s_h^k,a_h^k)\|_{\widehat{\mathbf{\Lambda}}_{k,h}^{-1}},
	\end{align*}
	where the first inequality holds by Cauchy-Schwarz inequality, the second inequality holds since
	\vspace{-.2cm}
	\begin{align*}
		\widehat{\mathcal{C}}^{(1)}_{k,h}=&\{\boldsymbol{\mu}:\|(\boldsymbol{\mu}-\widehat{\boldsymbol{\mu}}_{k,h}){\boldsymbol{V}}_{h+1}^*\|_{\widehat{\boldsymbol{\Lambda}}_{k, h}} \leq \widehat{\beta}^{(1)}\}\\
		\widehat{\mathcal{C}}^{(2)}_{k,h}=&\{\boldsymbol{\mu}:\|(\boldsymbol{\mu}-\widehat{\boldsymbol{\mu}}_{k,h})(\widehat{\boldsymbol{V}}_{k,h+1}-{\boldsymbol{V}}_{h+1}^*)\|_{\widehat{\boldsymbol{\Lambda}}_{k, h}} \leq \widehat{\beta}^{(2)}\},
	\end{align*}
	and the last equality holds since $\widehat{\beta}=\widehat{\beta}^{(1)}+\widehat{\beta}^{(2)}$.
	In the following, we briefly illustrate how to use Theorem~\ref{th:self} to build confidence sets $\widehat{\mathcal{C}}^{(1)}_{k,h}$ and $\widehat{\mathcal{C}}^{(2)}_{k,h}$.
	Initially, by
	\vspace{-.3cm}
	\begin{equation}\label{eq:cstmp}
		\left\|\left(\widehat{\boldsymbol{\mu}}_{k,h}-\boldsymbol{\mu}_h\right)\boldsymbol{V}\right\|_{\widehat{\boldsymbol{\Lambda}}_{k,h}}\lesssim\|\sum_{i=1}^{k-1}\widehat{\sigma}_{i,h}^{-2}\boldsymbol{\phi}(s_{h}^{i},a_{h}^{i}){\boldsymbol{\epsilon}_h^i}^\top\boldsymbol{V}\|_{\widehat{\mathbf{\Lambda}}_{k,h}^{-1}}
	\end{equation}
	for some fixed function $V$ in Lemma~\ref{lm:mud}, building a confidence set with respect to $V$ is equivalent to building a self-normalized bound for $\|\sum_{i=1}^{k-1}\widehat{\sigma}_{i,h}^{-2}\boldsymbol{\phi}(s_{h}^{i},a_{h}^{i}){\boldsymbol{\epsilon}_h^i}^\top\boldsymbol{V}\|_{\widehat{\mathbf{\Lambda}}_{k,h}^{-1}}$. 
	\vspace{-.2cm}
	\paragraph{Building $\widehat{\mathcal{C}}^{(1)}_{k,h}$:}
	We build confidence set $\widehat{\mathcal{C}}^{(1)}_{k,h}$ by applying the Bernstein self-normalized inequality in Theorem~\ref{th:self} with dynamic control of MDS magnitude, highlighted in Remark~\ref{sec:ep}.
	Thus, $U_{k,h}$ is specified in Lemma~\ref{lm:cs} to guarantee that $\widehat{\sigma}_{k, h}$ upper bounds $\mathbb{V}_hV_{h + 1}^*(s_h^k, a_h^k)$, and $\varsigma_{k,h}$ is set dynamically.
	Besides, the uniform convergence argument by covering net in \cite{jin2020provably,wang2020optimism,wang2020reinforcement} is not required, since now $V=V_{h+1}^*$ in Eq.~(\ref{eq:cstmp}) is a fixed function and there is no measurability issue.
	Consequently, we get $\widehat{\beta}^{(1)}_k = \widetilde{O}(\sqrt{d})$, which is detailed in Lemma~\ref{lm:rightarrowbeta}.
	
    \vspace{-.3cm}
    \paragraph{Building $\widehat{\mathcal{C}}^{(2)}_{k,h}$:}
	We apply also Theorem~\ref{th:self} with dynamic control of MDS magnitude to build $\widehat{\mathcal{C}}^{(2)}_{k,h}$ as well.
	Similarly, $E_{k,h}$ is specified in Lemma~\ref{lm:cs} to guarantee that $\widehat{\sigma}_{k, h}$ upper bounds $[\mathbb{V}_h(\widehat{V}_{k,h+1}-V_{h+1}^*)](s_h^k, a_h^k)$, and $\varsigma_{k,h}$ is set dynamically to keep the MDS magnitude $R$ in Theorem~\ref{th:self} small.
	Since now $V=\widehat{V}_{k,h+1}-V_{h+1}^*$ in Eq.~(\ref{eq:cstmp}) suffers from the measurability issue, a uniform convergence argument by covering net is still required, which bring extra dependency on $d$ in the exploration radius $\widehat{\beta}^{(2)}$.
	That is why we enlarge $E_{k,h}$ with a $d$ factor in $\widehat{\sigma}_{k,h}$.
	
	Putting everything together gives the following key technical lemma that builds the sharp optimistic confidence set $\widehat{\mathcal{C}}_{k,h}$.
	
	\begin{lemma}\label{lm:cs}
		Set $\widehat{\beta}=\widehat{\beta}^{(1)}+\widehat{\beta}^{(2)}$, then there exists an absolute constant $c>0$ such that for any $\delta\in(0,1)$, with probability at least $1-7\delta$, we have that simultaneously for any $k \in[K]$ and any $h \in[H]$,
		$$\boldsymbol{\mu}_{h}\in \widehat{\mathcal{C}}_{k,h}\cap\widecheck{\mathcal{C}}_{k, h},$$
		and $\big|[\widehat{\mathbb{V}}_{k, h} \widehat{V}_{k, h+1}](s_{h}^{k}, a_{h}^{k})-[\mathbb{V}_{h} V_{h+1}^*](s_{h}^{k}, a_{h}^{k})\big|\leq U_{k, h}$, $[\mathbb{V}_h(\widehat{V}_{k,h+1}-V_{h+1}^*)](s_h^k, a_h^k)\le E_{k,h}$
		where
		$\widehat{\beta}^{(1)}$, $\widehat{\beta}^{(2)}$, $U_{k,h}$, $E_{k,h}$ are specified in Lemma~\ref{lm:csf} in Appendix.
		\begin{proof}
			Please refer to Appendix~\ref{sec:appcs}.
		\end{proof}
	\end{lemma}
	\vspace{-.2cm}
	\section{Conclusion}
	This paper presents a computationally and statistically efficient algorithm, LSVI-UCB$^+$, which builds on linear weighted ridge regression and upper confidence value iteration with a Bernstein-type exploration bonus.
	LSVI-UCB$^+$ reaches minimax optimal regret bound up to logarithmic factors for linear MDPs.
	Our sharp result builds on a novel Bernstein self-normalized bound with the conservatism of elliptical potentials, and refined analysis of the correction term, which serve as new analytical tools for RL with linear function approximation.
	\vspace{-.2cm}
	\section*{Acknowledgements}
    \bug{
In our original version \cite{hu2022nearly}, there is a technical issue in building the over-optimistic value function $\widedot{V}_{k,h}(\cdot)$ (pointed out by \cite{he2022nearly,agarwal2022vo}) such that the theoretical results do not hold.
We thank \cite{he2022nearly,agarwal2022vo} for pointing out and addressing the technical issue of building the over-optimistic value function in our original version \cite{hu2022nearly} of the work.
In this version, we build on \cite{he2022nearly} by replacing the over-optimistic value function $\widedot{V}_{k,h}(\cdot)$ with the ``rare-switching'' value function first proposed in \cite{he2022nearly} such that our result still achieves minimax optimal regret for linear MDPs, with minor modifications of constant terms.
    }

	\bibliography{example_paper}

\begin{thebibliography}{37}
\providecommand{\natexlab}[1]{#1}
\providecommand{\url}[1]{\texttt{#1}}
\expandafter\ifx\csname urlstyle\endcsname\relax
  \providecommand{\doi}[1]{doi: #1}\else
  \providecommand{\doi}{doi: \begingroup \urlstyle{rm}\Url}\fi

\bibitem[Abbasi-Yadkori et~al.(2011)Abbasi-Yadkori, P{\'a}l, and
  Szepesv{\'a}ri]{abbasi2011improved}
Abbasi-Yadkori, Y., P{\'a}l, D., and Szepesv{\'a}ri, C.
\newblock Improved algorithms for linear stochastic bandits.
\newblock \emph{Advances in neural information processing systems},
  24:\penalty0 2312--2320, 2011.

\bibitem[Agarwal et~al.(2019)Agarwal, Jiang, Kakade, and
  Sun]{agarwal2019reinforcement}
Agarwal, A., Jiang, N., Kakade, S.~M., and Sun, W.
\newblock Reinforcement learning: Theory and algorithms.
\newblock \emph{CS Dept., UW Seattle, Seattle, WA, USA, Tech. Rep}, 2019.

\bibitem[Agarwal et~al.(2022)Agarwal, Jin, and Zhang]{agarwal2022vo}
Agarwal, A., Jin, Y., and Zhang, T.
\newblock Vo $ q $ l: Towards optimal regret in model-free rl with nonlinear
  function approximation.
\newblock \emph{arXiv preprint arXiv:2212.06069}, 2022.

\bibitem[Ayoub et~al.(2020)Ayoub, Jia, Szepesvari, Wang, and
  Yang]{ayoub2020model}
Ayoub, A., Jia, Z., Szepesvari, C., Wang, M., and Yang, L.
\newblock Model-based reinforcement learning with value-targeted regression.
\newblock In \emph{International Conference on Machine Learning}, pp.\
  463--474. PMLR, 2020.

\bibitem[Azar et~al.(2013)Azar, Munos, and Kappen]{azar2013minimax}
Azar, M.~G., Munos, R., and Kappen, H.~J.
\newblock Minimax pac bounds on the sample complexity of reinforcement learning
  with a generative model.
\newblock \emph{Machine learning}, 91\penalty0 (3):\penalty0 325--349, 2013.

\bibitem[Azar et~al.(2017)Azar, Osband, and Munos]{azar2017minimax}
Azar, M.~G., Osband, I., and Munos, R.
\newblock Minimax regret bounds for reinforcement learning.
\newblock In \emph{International Conference on Machine Learning}, pp.\
  263--272. PMLR, 2017.

\bibitem[Cai et~al.(2020)Cai, Yang, Jin, and Wang]{cai2020provably}
Cai, Q., Yang, Z., Jin, C., and Wang, Z.
\newblock Provably efficient exploration in policy optimization.
\newblock In \emph{International Conference on Machine Learning}, pp.\
  1283--1294. PMLR, 2020.

\bibitem[Cesa-Bianchi \& Lugosi(2006)Cesa-Bianchi and
  Lugosi]{cesa2006prediction}
Cesa-Bianchi, N. and Lugosi, G.
\newblock \emph{Prediction, learning, and games}.
\newblock Cambridge university press, 2006.

\bibitem[Dani et~al.(2008)Dani, Hayes, and Kakade]{dani2008stochastic}
Dani, V., Hayes, T.~P., and Kakade, S.~M.
\newblock Stochastic linear optimization under bandit feedback.
\newblock In \emph{Conference on Learning Theory-colt}, 2008.

\bibitem[Dann et~al.(2018)Dann, Jiang, Krishnamurthy, Agarwal, Langford, and
  Schapire]{dann2018oracle}
Dann, C., Jiang, N., Krishnamurthy, A., Agarwal, A., Langford, J., and
  Schapire, R.~E.
\newblock On oracle-efficient pac rl with rich observations.
\newblock In \emph{Proceedings of the 32nd International Conference on Neural
  Information Processing Systems}, pp.\  1429--1439, 2018.

\bibitem[Faury et~al.(2020)Faury, Abeille, Calauz{\`e}nes, and
  Fercoq]{faury2020improved}
Faury, L., Abeille, M., Calauz{\`e}nes, C., and Fercoq, O.
\newblock Improved optimistic algorithms for logistic bandits.
\newblock In \emph{International Conference on Machine Learning}, pp.\
  3052--3060. PMLR, 2020.

\bibitem[Freedman(1975)]{freedman1975tail}
Freedman, D.~A.
\newblock On tail probabilities for martingales.
\newblock \emph{the Annals of Probability}, pp.\  100--118, 1975.

\bibitem[Hager(1989)]{hager1989updating}
Hager, W.~W.
\newblock Updating the inverse of a matrix.
\newblock \emph{SIAM review}, 31\penalty0 (2):\penalty0 221--239, 1989.

\bibitem[He et~al.(2021)He, Zhou, and Gu]{he2020minimax}
He, J., Zhou, D., and Gu, Q.
\newblock Minimax optimal reinforcement learning for discounted mdps.
\newblock In \emph{Proceedings of the 32nd International Conference on Neural
  Information Processing Systems}, 2021.

\bibitem[He et~al.(2022)He, Zhao, Zhou, and Gu]{he2022nearly}
He, J., Zhao, H., Zhou, D., and Gu, Q.
\newblock Nearly minimax optimal reinforcement learning for linear markov
  decision processes.
\newblock \emph{arXiv preprint arXiv:2212.06132}, 2022.

\bibitem[Hu et~al.(2022)Hu, Chen, and Huang]{hu2022nearly}
Hu, P., Chen, Y., and Huang, L.
\newblock Nearly minimax optimal reinforcement learning with linear function
  approximation.
\newblock In \emph{International Conference on Machine Learning}, pp.\
  8971--9019. PMLR, 2022.

\bibitem[Jia et~al.(2020)Jia, Yang, Szepesvari, and Wang]{jia2020model}
Jia, Z., Yang, L., Szepesvari, C., and Wang, M.
\newblock Model-based reinforcement learning with value-targeted regression.
\newblock In \emph{Learning for Dynamics and Control}, pp.\  666--686. PMLR,
  2020.

\bibitem[Jiang et~al.(2017)Jiang, Krishnamurthy, Agarwal, Langford, and
  Schapire]{jiang2017contextual}
Jiang, N., Krishnamurthy, A., Agarwal, A., Langford, J., and Schapire, R.~E.
\newblock Contextual decision processes with low bellman rank are
  pac-learnable.
\newblock In \emph{International Conference on Machine Learning}, pp.\
  1704--1713. PMLR, 2017.

\bibitem[Jin et~al.(2018)Jin, Allen-Zhu, Bubeck, and Jordan]{jin2018q}
Jin, C., Allen-Zhu, Z., Bubeck, S., and Jordan, M.~I.
\newblock Is q-learning provably efficient?
\newblock In \emph{Proceedings of the 32nd International Conference on Neural
  Information Processing Systems}, pp.\  4868--4878, 2018.

\bibitem[Jin et~al.(2020)Jin, Yang, Wang, and Jordan]{jin2020provably}
Jin, C., Yang, Z., Wang, Z., and Jordan, M.~I.
\newblock Provably efficient reinforcement learning with linear function
  approximation.
\newblock In \emph{Conference on Learning Theory}, pp.\  2137--2143. PMLR,
  2020.

\bibitem[Kirschner \& Krause(2018)Kirschner and
  Krause]{kirschner2018information}
Kirschner, J. and Krause, A.
\newblock Information directed sampling and bandits with heteroscedastic noise.
\newblock In \emph{Conference On Learning Theory}, pp.\  358--384. PMLR, 2018.

\bibitem[Lattimore \& Hutter(2012)Lattimore and Hutter]{lattimore2012pac}
Lattimore, T. and Hutter, M.
\newblock Pac bounds for discounted mdps.
\newblock In \emph{International Conference on Algorithmic Learning Theory},
  pp.\  320--334. Springer, 2012.

\bibitem[Lattimore \& Szepesv{\'a}ri(2020)Lattimore and
  Szepesv{\'a}ri]{lattimore2020bandit}
Lattimore, T. and Szepesv{\'a}ri, C.
\newblock \emph{Bandit algorithms}.
\newblock Cambridge University Press, 2020.

\bibitem[Lattimore et~al.(2015)Lattimore, Crammer, and
  Szepesv{\'a}ri]{lattimore2015linear}
Lattimore, T., Crammer, K., and Szepesv{\'a}ri, C.
\newblock Linear multi-resource allocation with semi-bandit feedback.
\newblock In \emph{NIPS}, pp.\  964--972, 2015.

\bibitem[Li et~al.(2021)Li, Wang, Chen, and Zhou]{li2021tight}
Li, Y., Wang, Y., Chen, X., and Zhou, Y.
\newblock Tight regret bounds for infinite-armed linear contextual bandits.
\newblock In \emph{International Conference on Artificial Intelligence and
  Statistics}, pp.\  370--378. PMLR, 2021.

\bibitem[Modi et~al.(2020)Modi, Jiang, Tewari, and Singh]{modi2020sample}
Modi, A., Jiang, N., Tewari, A., and Singh, S.
\newblock Sample complexity of reinforcement learning using linearly combined
  model ensembles.
\newblock In \emph{International Conference on Artificial Intelligence and
  Statistics}, pp.\  2010--2020. PMLR, 2020.

\bibitem[Sun et~al.(2019)Sun, Jiang, Krishnamurthy, Agarwal, and
  Langford]{sun2019model}
Sun, W., Jiang, N., Krishnamurthy, A., Agarwal, A., and Langford, J.
\newblock Model-based rl in contextual decision processes: Pac bounds and
  exponential improvements over model-free approaches.
\newblock In \emph{Conference on learning theory}, pp.\  2898--2933. PMLR,
  2019.

\bibitem[Tossou et~al.(2019)Tossou, Basu, and Dimitrakakis]{tossou2019near}
Tossou, A., Basu, D., and Dimitrakakis, C.
\newblock Near-optimal optimistic reinforcement learning using empirical
  bernstein inequalities.
\newblock \emph{arXiv preprint arXiv:1905.12425}, 2019.

\bibitem[Victor et~al.(2009)Victor, la~Pe{\~n}a, Lai, and Shao]{victor2009self}
Victor, H., la~Pe{\~n}a, D., Lai, T.~L., and Shao, Q.-M.
\newblock \emph{Self-normalized processes: Limit theory and Statistical
  Applications}, volume 204.
\newblock Springer, 2009.

\bibitem[Wang et~al.(2020{\natexlab{a}})Wang, Salakhutdinov, and
  Yang]{wang2020reinforcement}
Wang, R., Salakhutdinov, R.~R., and Yang, L.
\newblock Reinforcement learning with general value function approximation:
  Provably efficient approach via bounded eluder dimension.
\newblock \emph{Advances in Neural Information Processing Systems}, 33,
  2020{\natexlab{a}}.

\bibitem[Wang et~al.(2020{\natexlab{b}})Wang, Wang, Du, and
  Krishnamurthy]{wang2020optimism}
Wang, Y., Wang, R., Du, S.~S., and Krishnamurthy, A.
\newblock Optimism in reinforcement learning with generalized linear function
  approximation.
\newblock In \emph{International Conference on Learning Representations},
  2020{\natexlab{b}}.

\bibitem[Yang \& Wang(2019)Yang and Wang]{yang2019sample}
Yang, L. and Wang, M.
\newblock Sample-optimal parametric q-learning using linearly additive
  features.
\newblock In \emph{International Conference on Machine Learning}, pp.\
  6995--7004. PMLR, 2019.

\bibitem[Yang \& Wang(2020)Yang and Wang]{yang2020reinforcement}
Yang, L. and Wang, M.
\newblock Reinforcement learning in feature space: Matrix bandit, kernels, and
  regret bound.
\newblock In \emph{International Conference on Machine Learning}, pp.\
  10746--10756. PMLR, 2020.

\bibitem[Zanette \& Brunskill(2019)Zanette and Brunskill]{zanette2019tighter}
Zanette, A. and Brunskill, E.
\newblock Tighter problem-dependent regret bounds in reinforcement learning
  without domain knowledge using value function bounds.
\newblock In \emph{International Conference on Machine Learning}, pp.\
  7304--7312. PMLR, 2019.

\bibitem[Zanette et~al.(2020{\natexlab{a}})Zanette, Brandfonbrener, Brunskill,
  Pirotta, and Lazaric]{zanette2020frequentist}
Zanette, A., Brandfonbrener, D., Brunskill, E., Pirotta, M., and Lazaric, A.
\newblock Frequentist regret bounds for randomized least-squares value
  iteration.
\newblock In \emph{International Conference on Artificial Intelligence and
  Statistics}, pp.\  1954--1964. PMLR, 2020{\natexlab{a}}.

\bibitem[Zanette et~al.(2020{\natexlab{b}})Zanette, Lazaric, Kochenderfer, and
  Brunskill]{zanette2020learning}
Zanette, A., Lazaric, A., Kochenderfer, M., and Brunskill, E.
\newblock Learning near optimal policies with low inherent bellman error.
\newblock In \emph{International Conference on Machine Learning}, pp.\
  10978--10989. PMLR, 2020{\natexlab{b}}.

\bibitem[Zhou et~al.(2021)Zhou, Gu, and Szepesvari]{zhou2021nearly}
Zhou, D., Gu, Q., and Szepesvari, C.
\newblock Nearly minimax optimal reinforcement learning for linear mixture
  markov decision processes.
\newblock In \emph{Conference on Learning Theory}, pp.\  4532--4576. PMLR,
  2021.

\end{thebibliography}
	\bibliographystyle{icml2022}

	\newpage
	\appendix
	\onecolumn

    \appendixpage
	\startcontents[section]
    \printcontents[section]{l}{1}{\setcounter{tocdepth}{2}}
	
	In the appendix, we present some additional results and supporting materials to supplement the statements, theorems and proofs in the main papers. There are 6 sections in appendix:
	\begin{itemize}
		\item Appendix~\ref{sec:apptable} presents additional comparisons with related works.
		\item Appendix~\ref{sec:appself} presents the proof of our proposed sharp Bernstein tail inequality for self-normalized vector-valued martingales.
		\item Appendix~\ref{sec:apphpe} presents the construction of several high probability confidences sets.
		\item Appendix~\ref{sec:appregret} presents the former proof of the regret upper bound of LSVI-UCB$^+$ algorithm, with conclusion given in Theorem~\ref{th:regkr}.
		\item Appendix~\ref{sec:applower} constructs a hard-to-learn MDP to build a regret lower bound for linear MDPs.
		\item Appendix~\ref{sec:appaux} presents auxiliary lemmas necessary for proofs in above sections and important properties that will be helpful in algorithm design.
	\end{itemize}
	\section{Additional Comparisons of Related Works}\label{sec:apptable}
	Table~\ref{tb:compf} serves as a more complete table compared to Table~\ref{tb:comp} in the main paper, which lists some representative works in RL with linear function approximation.
	The top part of Table~\ref{tb:compf} lists representative works for the linear MDP and its generalizations and the buttom part is for the linear mixture MDP and its generalizations.
	
	\begin{table}[H]
		\caption{Theoretical results on RL with linear function approximation, where $\dagger$ denotes that rewards are adversarial, and $d_h$ is the dimension of the feature mapping at the $h$-th stage within the episodes and $K$ is the number of episodes.}
		\label{tb:compf}
		\begin{center}
			\begin{small}
				\begin{tabular}{lllll}
					\toprule
					Setting & Algorithm & Technique & Regret \\
					\midrule
					Linear MDP & OPT-RLSVI \cite{zanette2020frequentist} & Hoeffding+Covering & $\widetilde{O}(H^2d^2\sqrt{T})$ \\
					Linear MDP & LSVI-UCB \cite{jin2020provably} & Hoeffding+Covering & $\widetilde{O}(\sqrt{H^3d^3T})$ \\
					Linear MDP &  LSVI-UCB$^+$ (\textbf{this paper}) & Bernstein+\emph{Covering} & $\widetilde{O}(Hd\sqrt{T})$ \\
					Linear Q Function&  LSVI-UCB$^*$ \cite{wang2020optimism} & Hoeffding+Covering & $\widetilde{O}(\sqrt{d^3T})$ \\
					Low Bellman Error & ELEANOR \cite{zanette2020learning} & Hoeffding+Covering  & $\widetilde{O}(\sum_{h=1}^{H}d_h\sqrt{K})$ \\
					Bounded Eluder Dimension & $\mathcal{F}\operatorname{-LSVI}(\delta)$ \cite{wang2020reinforcement} & Hoeffding+Covering & $\widetilde{O}(\operatorname{poly}(dH)\sqrt{T})$ \\
					\midrule
					Linear Mixture MDP & UCRL-VTR \cite{jia2020model,ayoub2020model} & Hoeffding & $\widetilde{O}(\sqrt{H^3d^2T})$ \\
					Linear Mixture MDP & UCRL-VTR$^+$ \cite{zhou2021nearly} & Bernstein & $\widetilde{O}(\sqrt{H^2d^2T+H^3dT})$ \\
					Feature Space & MatrixRL \cite{yang2020reinforcement} & Hoeffding & $\widetilde{O}(H^2d\log T\sqrt{T})$ \\
					Linear Mixture MDP$^\dagger$ & OPPO \cite{cai2020provably} & Hoeffding & $\widetilde{O}(\sqrt{d^2H^3T})$ \\
					\bottomrule
				\end{tabular}
			\end{small}
		\end{center}
	\end{table}
	
	From Table~\ref{tb:compf}, we can find that existing algorithms for the linear MDP and its generalizations all use a classical Hoeffding self-normalized bound such as Theorem 1 in \cite{abbasi2011improved} with the covering net argument, while our work introduces a Bernstein self-normalized bound with a covering net argument.
	Moreover, building a covering net argument in our work does not brings extra dependency on feature space dimension $d$ since we only consider covering net argument in bounding the correction term which can be made small.
	
	As for the linear mixture MDP and its generalizations, the covering net argument is not required due to the structure of the linear mixture MDP.
	In addition, prior works \cite{yang2020reinforcement, jia2020model, ayoub2020model, cai2020provably} utilize Hoeffding self-normalized bound to build confidence sets, while \cite{zhou2021nearly} consider the Bernstein self-normalized bound for the first time in the setting of linear mixture MDP.
	Compared with regret bound of $\widetilde{O}(\sqrt{H^3d^2T})$ obtained in \cite{jia2020model, ayoub2020model} for linear mixture MDPs, the regret bound $\widetilde{O}(\sqrt{H^2d^2T+H^3dT})$ in \cite{zhou2021nearly} is better and a $\sqrt{H}$ factor is further saved if $d\ge H$.
	
	\section{Sharp Bernstein Self-Normalized Bound}\label{sec:appself}
	In this section, we prove the proposed sharp Bernstein tail inequality for self-normalized vector-valued martingales.
	Our proof diagram is based on the proof of Theorem 1 in \cite{zhou2021nearly}, which is firstly proposed in the proof of Lemma 14 in  \cite{dani2008stochastic}.
	However, our Bernstein self-normalized bound is sharper than Theorem 1 in \cite{zhou2021nearly} with critical changes of the attenuation of the martingale difference sequence. 
	
	Specifically, $\{\min\{1,\|\boldsymbol{x}_t\|_{\mathbf{Z}_{t-1}^{-1}}\}\}_{t\in[T]}$ can be roughly considered as an attenuated sequence since we can prove $\sum_{t=1}^{T}\min\{1,\|\boldsymbol{x}_t\|_{\mathbf{Z}_{t-1}^{-1}}\}=O(d\log T)$ by Elliptical Potential Lemma.
	On the contrary, $\min\{1,\|\boldsymbol{x}_t\|_{\mathbf{Z}_{t-1}^{-1}}\}$ is deflated to $1$ in Theorem 1, \cite{zhou2021nearly} such that the bound is looser than ours.
	In the following proof, we do not deflate $\min\{1,\|\boldsymbol{x}_t\|_{\mathbf{Z}_{t-1}^{-1}}\}$ to $1$, and we take into account the elliptical potential $\|\boldsymbol{x}_t\|_{\mathbf{Z}_{t-1}^{-1}}$ in our algorithm design, which is one of the major contributions in this paper.

	Firstly, we give the following definitions to simplifying notations during the proof. 
	\begin{definition}\label{def:selfp}
		\begin{align*}
        	    \boldsymbol{d}_{t}:=&\sum_{i=1}^{t} \boldsymbol{x}_{i} \eta_{i},\\ Z_{t}:=&\left\|\boldsymbol{d}_{t}\right\|_{\mathbf{Z}_{t}^{-1}},\\ w_{t}:=&\left\|\boldsymbol{x}_{t}\right\|_{\mathbf{Z}_{t-1}^{-1}},\\
        				\beta_{t}:=&8 \sigma \sqrt{d \log \left(1+t L^{2} /(d \lambda)\right) \log \left(4 t^{2} / \delta\right)}+4 R \log \left(4 t^{2} / \delta\right),\\
        				\mathcal{E}_{t}:=&\mathds{1}\left\{0 \leq s \leq t, Z_{s} \leq \beta_{s}\right\},
    	\end{align*}
		for $t\ge1$ and $\boldsymbol{d}_{0}=0, Z_{0}=0, \beta_{0}=0$ for $t=0$.
	\end{definition}
	
	\paragraph{Measurability}
	With the assumptions in Theorem~\ref{th:self}, $x_{t}$ is $\mathcal{G}_{t}$-measurable and $\eta_{t}$ is $\mathcal{G}_{t+1}$-measurable. Thus, $w_{t}$ is $\mathcal{G}_{t}$-measurable, and $d_{t}, Z_{t}$ and $\mathcal{E}_{t}$ are $\mathcal{G}_{t+1}$-measurable.
	
	Our goal is to upper bound $Z_{t}$. By definition of $Z_t$, we have

	\begin{align*}
		Z_{t}^{2} &=\left(\boldsymbol{d}_{t-1}+\boldsymbol{x}_{t} \eta_{t}\right)^{\top} \mathbf{Z}_{t}^{-1}\left(\boldsymbol{d}_{t-1}+\boldsymbol{x}_{t} \eta_{t}\right) \\
		&=\boldsymbol{d}_{t-1}^{\top} \mathbf{Z}_{t}^{-1} \boldsymbol{d}_{t-1}+2 \eta_{t} \boldsymbol{x}_{t}^{\top} \mathbf{Z}_{t}^{-1} \boldsymbol{d}_{t-1}+\eta_{t}^{2} \boldsymbol{x}_{t}^{\top} \mathbf{Z}_{t}^{-1} \boldsymbol{x}_{t} \\
		& \leq Z_{t-1}^{2}+\underbrace{2 \eta_{t} \boldsymbol{x}_{t}^{\top} \mathbf{Z}_{t}^{-1} \boldsymbol{d}_{t-1}}_{I_{1}}+\underbrace{\eta_{t}^{2} \boldsymbol{x}_{t}^{\top} \mathbf{Z}_{t}^{-1} \boldsymbol{x}_{t}}_{I_{2}}
	\end{align*}

	where the inequality holds since $\mathbf{Z}_{t} \succeq \mathbf{Z}_{t-1}$.
	
	Since $\mathbf{Z}_{t}=\mathbf{Z}_{t-1}+\boldsymbol{x}_{t} \boldsymbol{x}_{t}^{\top}$, by the Sherman–Morrison formula \cite{hager1989updating}, we obtain
	$$
	\mathbf{Z}_{t}^{-1}=\mathbf{Z}_{t-1}^{-1}-\frac{\mathbf{Z}_{t-1}^{-1} \boldsymbol{x}_{t} \boldsymbol{x}_{t}^{\top} \mathbf{Z}_{t-1}^{-1}}{1+w_{t}^{2}},\quad\text{for }t\ge1.
	$$
	
	Subsequently,
	$$
	\begin{aligned}
		I_{1}&=2 \eta_{t}\left(\boldsymbol{x}_{t}^{\top} \mathbf{Z}_{t-1}^{-1} \boldsymbol{d}_{t-1}-\frac{\boldsymbol{x}_{t}^{\top} \mathbf{Z}_{t-1}^{-1} \boldsymbol{x}_{t} \boldsymbol{x}_{t}^{\top} \mathbf{Z}_{t-1}^{-1} \boldsymbol{d}_{t-1}}{1+w_{t}^{2}}\right)=2 \eta_{t}\left(\boldsymbol{x}_{t}^{\top} \mathbf{Z}_{t-1}^{-1} \boldsymbol{d}_{t-1}-\frac{w_{t}^{2} \boldsymbol{x}_{t}^{\top} \mathbf{Z}_{t-1}^{-1} \boldsymbol{d}_{t-1}}{1+w_{t}^{2}}\right)= \frac{2 \eta_{t} \boldsymbol{x}_{t}^{\top} \mathbf{Z}_{t-1}^{-1} \boldsymbol{d}_{t-1}}{1+w_{t}^{2}}\\
		I_{2}&=\eta_{t}^{2}\left(\boldsymbol{x}_{t}^{\top} \mathbf{Z}_{t-1}^{-1} \boldsymbol{x}_{t}^{\top}-\frac{\boldsymbol{x}_{t}^{\top} \mathbf{Z}_{t-1}^{-1} \boldsymbol{x}_{t} \boldsymbol{x}_{t}^{\top} \mathbf{Z}_{t-1}^{-1} \boldsymbol{x}_{t}}{1+w_{t}^{2}}\right)=\eta_{t}^{2}\left(w_{t}^{2}-\frac{w_{t}^{4}}{1+w_{t}^{2}}\right)=\frac{\eta_{t}^{2} w_{t}^{2}}{1+w_{t}^{2}} 
	\end{aligned}
	$$
	
	Therefore, we have
	\begin{equation}\label{eq:ztdp}
		Z_{t}^{2} \leq \sum_{i=1}^{t} \frac{2 \eta_{i} \boldsymbol{x}_{i}^{\top} \mathbf{Z}_{i-1}^{-1} \boldsymbol{d}_{i-1}}{1+w_{i}^{2}}+\sum_{i=1}^{t} \frac{\eta_{i}^{2} w_{i}^{2}}{1+w_{i}^{2}}
	\end{equation}
	Now we try to bound the two summation terms on  the r.h.s. of Eq.~(\ref{eq:ztdp}) in Lemma~\ref{lm:ztdp1} and Lemma~\ref{lm:ztdp2}, respectively. Before that, we present a uniform Bernstein bound required for proving Lemma~\ref{lm:ztdp1} and Lemma~\ref{lm:ztdp2}.
	
	\begin{lemma}[Uniform Bernstein Bound]\label{lm:unibernstein}
		Let $\left\{x_{i}, \mathcal{F}_{i}\right\}$ be a martingale difference sequence with $\forall i\ge1$, $\mathbb{E}\left(x_{i}\mid \mathcal{F}_{i-1}\right)=0$, $\mathbb{E}\left(x_{i}^{2} \mid \mathcal{F}_{i-1}\right)=\sigma_{i}^{2}$, $V_{i}^{2}=\sum_{j=1}^{i} \sigma_{j}^{2} .$ Furthermore, assume that $\mathbb{P}\left(\left|x_{i}\right| \leq c \mid \mathcal{F}_{i-1}\right)=1$ for any $0<c<\infty$.
		
		Then, for any $\delta>0$, with probability at least $1-\delta$, simultaneously for any $t\ge1$, it holds that
		$$
		\sum_{i=1}^{t} d_{i} \leq \sqrt{2V_t^2\log (2t^2/ \delta)}+\frac{2c \log (2t^2 / \delta)}{3}.
		$$
		\begin{proof}
			By Freedman's inequality in Lemma~\ref{lm:freedman}, for any $\delta>0$ and some $t\ge1$, with probability at least $1-\delta/(2t^2)$, we have:
			\begin{equation}\label{eq:unip}
				\sum_{i=1}^{t} d_{i} \leq \sqrt{2V_t^2\log (1 / \delta)}+\frac{2}{3}c \log (1 / \delta).
			\end{equation}
			Taking a union bound for Eq.~(\ref{eq:unip}) from $t=1$ to $\infty$ and using the fact that $\sum_{t=1}^{\infty} t^{-2}<2$ complete the proof.
		\end{proof}
	\end{lemma}
	
	Next, we bound the first summation term on the r.h.s. of Eq.~(\ref{eq:ztdp}) in Lemma~\ref{lm:ztdp1}.
	
	\begin{lemma}\label{lm:ztdp1}
		Under assumptions in Theorem~\ref{th:self} and Definition~\ref{def:selfp}, with probability at least $1-\delta / 2$, simultaneously for all $t \geq 1$ it holds that
		$$
		\sum_{i=1}^{t} \frac{2 \eta_{i} \boldsymbol{x}_{i}^{\top} \mathbf{Z}_{i-1}^{-1} \boldsymbol{d}_{i-1}}{1+w_{i}^{2}} \mathcal{E}_{i-1} \leq 3 \beta_{t}^{2} / 4
		$$
		\begin{proof}
			Define
			$$
			\ell_{i}=\frac{2 \eta_{i} \boldsymbol{x}_{i}^{\top} \mathbf{Z}_{i-1}^{-1} \boldsymbol{d}_{i-1}}{1+w_{i}^{2}} \mathcal{E}_{i-1},
			$$
			and we will give a uniform upper bound of  $\sum_{i=1}^tl_i$ by Lemma~\ref{lm:unibernstein} below.
			
			Firstly, for any $1\le i\le t$,
			$$\mathbb{E}[\ell_i\mid\mathcal{G}_i]=\frac{2\boldsymbol{x}_{i}^{\top} \mathbf{Z}_{i-1}^{-1} \boldsymbol{d}_{i-1}}{1+w_{i}^{2}} \mathcal{E}_{i-1}\mathbb{E}[\eta_i\mid\mathcal{G}_i]=0.$$
			Besides, we have
			\begin{equation}\label{eq:ztdp1}
				\left|\frac{2 \boldsymbol{x}_{i}^{\top} \mathbf{Z}_{i-1}^{-1} \boldsymbol{d}_{i-1}}{1+w_{i}^{2}} \mathcal{E}_{i-1}\right| \leq \frac{2\left\|\boldsymbol{x}_{i}\right\|_{\mathbf{Z}_{i-1}^{-1}}\left[\left\|\boldsymbol{d}_{i-1}\right\|_{\mathbf{Z}_{i-1}^{-1}} \mathcal{E}_{i-1}\right]}{1+w_{i}^{2}} \leq \frac{2 w_{i} \beta_{i-1}}{1+w_{i}^{2}} \leq 2\min\left\{1,w_{i}\right\} \beta_{i-1}
			\end{equation}
			where the first inequality holds due to Cauchy-Schwarz inequality, the second inequality holds due to the definition of $\mathcal{E}_{i-1}$, and the last inequality holds by algebra. Thus,
			$$\left|\ell_{i}\right|=\left|2\eta_{i}\cdot\min\left\{1,w_{i}\right\}\beta_{i}\right|\le2R \beta_{i}\le2R \beta_{t}$$ where the last inequality holds since $\left\{\beta_{i}\right\}_{i\in[t]}$ is an increasing sequence.
			
			Secondly, we also have
			$$
			\begin{aligned}
				\sum_{i=1}^{t} \mathbb{E}\left[\ell_{i}^{2} \mid \mathcal{G}_{i}\right] & \leq \sigma^{2} \sum_{i=1}^{t}\left(\frac{2 \boldsymbol{x}_{i}^{\top} \mathbf{Z}_{i-1}^{-1} \boldsymbol{d}_{i-1}}{1+w_{i}^{2}} \mathcal{E}_{i-1}\right)^{2} \\
				& \leq \sigma^{2} \sum_{i=1}^{t}\left[2\min \left\{1,w_{i}\right\} \beta_{i-1}\right]^{2} \\
				& \leq 4 \sigma^{2} \beta_{t}^{2} \sum_{i=1}^{t} \min \left\{1, w_{i}^{2}\right\} \\
				& \leq 8 \sigma^{2} \beta_{t}^{2} d \log \left(1+t L^{2} /(d \lambda)\right)
			\end{aligned}
			$$
			where the first inequality holds since $\mathbb{E}\left[\eta_{i}^{2} \mid \mathcal{G}_{i}\right] \leq \sigma^{2}$, the second inequality holds due to Eq.~(\ref{eq:ztdp1}), the third inequality holds  since $\left\{\beta_{i}\right\}_{i\in[t]}$ is an increasing sequence, and the last inequality holds due to Lemma~\ref{lm:ablog}.
			
			Therefore, using Lemma~\ref{lm:unibernstein}, simultaneously for all $t\ge1$, with probability at least $1-\delta /\left(4 t^{2}\right)$, it holds that
			$$
			\begin{aligned}
				\sum_{i=1}^{t} \ell_{i} & \leq \sqrt{16 \sigma^{2} \beta_{t}^{2} d \log \left(1+t L^{2} /(d \lambda)\right) \log \left(4 t^{2} / \delta\right)}+2 / 3 \cdot 2R \beta_{t} \log \left(4 t^{2} / \delta\right) \\
				& \leq \frac{\beta_{t}^{2}}{4}+16 \sigma^{2} d \log \left(1+t L^{2} /(d \lambda)\right) \log \left(4 t^{2} / \delta\right)+\frac{\beta_{t}^{2}}{4}+4 R^{2} \log ^{2}\left(4 t^{2} / \delta\right) \\
				& \leq \beta_{t}^{2} / 2+\frac{1}{4}\left(8 \sigma \sqrt{d \log \left(1+t L^{2} /(d \lambda)\right) \log \left(4 t^{2} / \delta\right)}+4 R \log \left(4 t^{2} / \delta\right)\right)^{2} \\
				&=3 \beta_{t}^{2} / 4
			\end{aligned}
			$$
			where the first inequality holds due to Lemma~\ref{lm:unibernstein}, the second inequality holds due to $2 \sqrt{|a b|} \leq$ $|a|+|b|$, and the last equality follows from the definition of $\beta_{t}$.
		\end{proof}
	\end{lemma}
	
	Accordingly, the second summation terms on the r.h.s. of Eq.~(\ref{eq:ztdp}) is bounded in Lemma~\ref{lm:ztdp2}.
	
	\begin{lemma}\label{lm:ztdp2}
		Under assumptions in Theorem~\ref{th:self} and Definition~\ref{def:selfp}, with probability at least $1-\delta / 2$, simultaneously for all $t \geq 1$ it holds that
		$$
		\sum_{i=1}^{t} \frac{\eta_{i}^{2} w_{i}^{2}}{1+w_{i}^{2}} \leq \beta_{t}^{2} / 4
		$$
		\begin{proof}
			Define
			$$
			\ell_{i}=\frac{\eta_{i}^{2} w_{i}^{2}}{1+w_{i}^{2}}-\mathbb{E}\left[\frac{\eta_{i}^{2} w_{i}^{2}}{1+w_{i}^{2}} \mid \mathcal{G}_{i}\right] .
			$$
			We will give a uniform upper bound of  $\sum_{i=1}^tl_i$ by Lemma~\ref{lm:unibernstein}, similar to Lemma~\ref{lm:ztdp1}.

			Clearly, for any $1\le i\le t$, we have $\mathbb{E}\left[\ell_{i} \mid \mathcal{G}_{i}\right]=0$. We further have that
			$$
			\begin{aligned}
				\sum_{i=1}^{t} \mathbb{E}\left[\ell_{i}^{2} \mid \mathcal{G}_{i}\right] & \leq \sum_{i=1}^{t} \mathbb{E}\left[\frac{\eta_{i}^{4} w_{i}^{4}}{\left(1+w_{i}^{2}\right)^{2}} \mid \mathcal{G}_{i}\right] \\
				& \leq R^{2} \sum_{i=1}^{t} \mathbb{E}\left[\frac{\eta_{i}^{2} w_{i}^{2}}{1+w_{i}^{2}} \mid \mathcal{G}_{i}\right] \\
				& \leq R^{2} \sigma^{2} \sum_{i=1}^{t} \frac{w_{i}^{2}}{1+w_{i}^{2}} \\
				& \leq 2 R^{2} \sigma^{2} d \log \left(1+t L^{2} /(d \lambda)\right)
			\end{aligned}
			$$
			where the first inequality holds due to the fact $\mathbb{E}(X-\mathbb{E} X)^{2} \leq \mathbb{E} X^{2}$, the second inequality holds since $\left|\eta_{i}\cdot\min\left\{1,w_{i}\right\}\right| \le R$, the third inequality holds since $\mathbb{E}\left[\eta_{i}^{2} \mid \mathcal{G}_{i}\right] \leq \sigma^{2}$, and the fourth inequality holds due to the fact $w_{i}^{2} /\left(1+w_{i}^{2}\right) \leq \min \left\{1, w_{i}^{2}\right\}$ and Lemma~\ref{lm:ablog}. 
			
			Furthermore, using the fact that $\left|\eta_{i}\cdot\min\left\{1,w_{i}\right\}\right| \le R$ holds almost surely under filtration $\mathcal{G}_i$, we obtain 
			$$
			\left|\frac{\eta_{i}^{2} w_{i}^{2}}{1+w_{i}^{2}}\right|\le\left|\eta_{i}^2\cdot\min\left\{1,w_{i}^2\right\}\right|=\left|\eta_{i}\cdot\min\left\{1,w_{i}\right\}\right|^2=R^2,
			$$
			and
			$$
			\left|\ell_{i}\right| \leq\left|\frac{\eta_{i}^{2} w_{i}^{2}}{1+w_{i}^{2}}\right|+\left|\mathbb{E}\left[\frac{\eta_{i}^{2} w_{i}^{2}}{1+w_{i}^{2}} \mid \mathcal{G}_{i}\right]\right| \leq 2 R^{2}.
			$$
			
			Therefore, by Lemma~\ref{lm:unibernstein}, simultaneously for all $t\ge1$, with probability at least $1-\delta /\left(4 t^{2}\right)$, it holds that
			$$
			\begin{aligned}
				\sum_{i=1}^{t} \frac{\eta_{i}^{2} w_{i}^{2}}{1+w_{i}^{2}} & \leq \sum_{i=1}^{t} \mathbb{E}\left[\frac{\eta_{i}^{2} w_{i}^{2}}{1+w_{i}^{2}} \mid \mathcal{G}_{i}\right]+\sqrt{4 R^{2} \sigma^{2} d \log \left(1+t L^{2} /(d \lambda)\right) \log \left(4 t^{2} / \delta\right)}+4 / 3 \cdot R^{2} \log \left(4 t^{2} / \delta\right) \\
				& \leq \sigma^{2} \sum_{i=1}^{t} \frac{w_{i}^{2}}{1+w_{i}^{2}}+2 R \sigma \sqrt{d \log \left(1+t L^{2} /(d \lambda)\right) \log \left(4 t^{2} / \delta\right)}+2 R^{2} \log \left(4 t^{2} / \delta\right) \\
				& \leq 2 \sigma^{2} d \log \left(1+t L^{2} /(d \lambda)\right)+2 R \sigma \sqrt{d \log \left(1+t L^{2} /(d \lambda)\right) \log \left(4 t^{2} / \delta\right)}+2 R^{2} \log \left(4 t^{2} / \delta\right) \\
				& \leq 1 / 4 \cdot\left(8 \sigma \sqrt{d} \sqrt{\log \left(1+t L^{2} /(d \lambda)\right) \log \left(4 t^{2} / \delta\right)}+4 R \log \left(4 t^{2} / \delta\right)\right)^{2} \\
				&=\beta_{t}^{2} / 4
			\end{aligned}
			$$
			where the first inequality holds due to Lemma~\ref{lm:unibernstein}, the second inequality holds due to $\mathbb{E}\left[\eta_{i}^{2} \mid \mathcal{G}_{i}\right] \leq \sigma^{2}$, the third inequality holds due to the fact $w_{i}^{2} /\left(1+w_{i}^{2}\right) \leq \min \left\{1, w_{i}^{2}\right\}$ and Lemma 12, and the last inequality holds due to the definition of $\beta_{t}$.
		\end{proof}
	\end{lemma}
	
	\subsection{Proof of Theorem~\ref{th:self}}
	
	\begin{proof}[Proof of Theorem~\ref{th:self}]
		
		Consider the case when conclusions of Lemma~\ref{lm:ztdp1} and Lemma~\ref{lm:ztdp2} hold.
		Conditioning on this event, we claim $Z_{t} \leq \beta_{t}$ for any $t\geq 0$. 
		
		We prove this by induction on $t$. Initially, the base case of $t=0$ holds since $\beta_{0}=0=Z_{0}$ by definition. Now fix some $t \geq 1$ and assume that for all $0 \leq i\leq t-1$, we have $Z_{i} \leq \beta_{i}$. This implies that $\mathcal{E}_{1}=\mathcal{E}_{2}=\cdots=\mathcal{E}_{t-1}=1$. Then, by Eq.~(\ref{eq:ztdp}), we have
		\begin{equation}\label{eq:thselfp1}
			Z_{t}^{2} \leq \sum_{i=1}^{t} \frac{2 \eta_{i} \boldsymbol{x}_{i}^{\top} \mathbf{Z}_{i-1}^{-1} \boldsymbol{d}_{i-1}}{1+w_{i}^{2}}+\sum_{i=1}^{t} \frac{\eta_{i}^{2} w_{i}^{2}}{1+w_{i}^{2}}=\sum_{i=1}^{t} \frac{2 \eta_{i} \boldsymbol{x}_{i}^{\top} \mathbf{Z}_{i-1}^{-1} \boldsymbol{d}_{i-1}}{1+w_{i}^{2}} \mathcal{E}_{i-1}+\sum_{i=1}^{t} \frac{\eta_{i}^{2} w_{i}^{2}}{1+w_{i}^{2}}.
		\end{equation}
		Since the conclusions of Lemma~\ref{lm:ztdp1} and Lemma~\ref{lm:ztdp2} hold, we have
		\begin{eqnarray}
			\sum_{i=1}^{t} \frac{2 \eta_{i} \boldsymbol{x}_{i}^{\top} \mathbf{Z}_{i-1}^{-1} \boldsymbol{d}_{i-1}}{1+w_{i}^{2}} \mathcal{E}_{i-1} \leq& 3 \beta_{t}^{2} / 4,\label{eq:thselfp21}\\
			\sum_{i=1}^{t} \frac{\eta_{i}^{2} w_{i}^{2}}{1+w_{i}^{2}} \leq &\beta_{t}^{2} / 4.\label{eq:thselfp22}
		\end{eqnarray}
		Therefore, substituting Eq.~(\ref{eq:thselfp21}) and (\ref{eq:thselfp22}) into Eq.~(\ref{eq:thselfp1}), we have $Z_{t} \leq \beta_{t}$, which ends the induction. Taking the union bound of the events in Lemma~\ref{lm:ztdp1} and Lemma~\ref{lm:ztdp2} implies that with probability at least $1-\delta$, $Z_{t} \leq \beta_{t}$ holds for any $t\ge1$.
	\end{proof}

	\section{High Probability Events}\label{sec:apphpe}
	In this section, we define some high probability events, i.e., confidence sets concerning the parameter $\boldsymbol{\mu}_h$, and show how to build them.
	The goal of this section is to build the sharp optimistic confidence set $\widehat{\mathcal{C}}_{k,h}$ in Lemma~\ref{lm:cs} for all $(k,h)\in[K]\times[H]$.
	
	We lists all confidence sets encountered during the proof in the following. Confidence sets $\widebar{\mathcal{C}}_{k,h},\widetilde{\mathcal{C}}_{k,h},\widecheck{\mathcal{C}}_{k,h}$ are called independent confidence sets, since they can be built by applying self-normalized concentration inequality directly without conditioning on other events.
	Instead, confidence sets $\widehat{\mathcal{C}}^{(1)}_{k,h},\widehat{\mathcal{C}}^{(2)}_{k,h},\widehat{\mathcal{C}}_{k,h}$ are called dependent confidence sets since they can only be built by conditioning on other confidence sets, apart from self-normalized concentration inequalities.
	
	\begin{definition}[Confidence Set]
    	\begin{itemize}
    		\item Independent Confidence Sets:
    		\begin{align*}
    			\widebar{\mathcal{C}}_{k, h}=&\left\{\boldsymbol{\mu}:\left\|\left(\boldsymbol{\mu}-\widehat{\boldsymbol{\mu}}_{k, h}\right)\widehat{\boldsymbol{V}}_{k,h+1}\right\|_{\widehat{\boldsymbol{\Lambda}}_{k, h}} \leq\widebar{\beta}\right\}\\
    			\widetilde{\mathcal{C}}_{k, h}=&\left\{\boldsymbol{\mu}:\left\|\left(\boldsymbol{\mu}-\widehat{\boldsymbol{\mu}}_{k, h}\right)\widehat{\boldsymbol{V}}_{k,h+1}^2\right\|_{\widehat{\boldsymbol{\Lambda}}_{k, h}} \leq\widetilde{\beta}\right\}\\
    			\widecheck{\mathcal{C}}_{k, h}=&\left\{\boldsymbol{\mu}:\left\|\left(\boldsymbol{\mu}-\widehat{\boldsymbol{\mu}}_{k, h}\right)\widecheck{\boldsymbol{V}}_{k,h+1}\right\|_{\widehat{\boldsymbol{\Lambda}}_{k, h}} \leq \widecheck{\beta}\right\}\\
    		\end{align*}
    		\item Dependent Confidence Sets:
    		\begin{align*}
    			\widehat{\mathcal{C}}^{(1)}_{k, h}=&\left\{\boldsymbol{\mu}:\left\|\left(\boldsymbol{\mu}-\widehat{\boldsymbol{\mu}}_{k, h}\right){\boldsymbol{V}}_{h+1}^*\right\|_{\widehat{\boldsymbol{\Lambda}}_{k, h}} \leq \widehat{\beta}^{(1)}\right\}\\
    			\widehat{\mathcal{C}}^{(2)}_{k, h}=&\left\{\boldsymbol{\mu}:\left\|\left(\boldsymbol{\mu}-\widehat{\boldsymbol{\mu}}_{k, h}\right)\left(\widehat{\boldsymbol{V}}_{k,h+1}-{\boldsymbol{V}}_{h+1}^*\right)\right\|_{\widehat{\boldsymbol{\Lambda}}_{k, h}} \leq \widehat{\beta}^{(2)}\right\}\\
    			\widehat{\mathcal{C}}_{k, h}=&\left\{\boldsymbol{\mu}:\left\|\left(\boldsymbol{\mu}-\widehat{\boldsymbol{\mu}}_{k, h}\right)\widehat{\boldsymbol{V}}_{k,h+1}\right\|_{\widehat{\boldsymbol{\Lambda}}_{k, h}} \leq\widehat{\beta}^{(1)}+\widehat{\beta}^{(2)}=\widehat{\beta}\right\}
    		\end{align*}
    	\end{itemize}
    \end{definition}
	
	To simplify notations during the proof, we further define the following events that optimistic and pessimistic confidence sets hold in multiple stages under some episode $k\in[K]$ or all episodes.
	
	\begin{definition}[Optimism Event]
		\begin{equation}
			\begin{aligned} \widehat{\Psi}_{k,h}:=&\left\{\forall h\le h'\le H:\boldsymbol{\mu}_{h'}\in\widehat{\mathcal{C}}_{k,h'}\right\}\\
			\widehat{\Psi}_{h}:=&\left\{\forall i\in[K],\forall h\le h'\le H:\boldsymbol{\mu}_{h'}\in\widehat{\mathcal{C}}_{i,h'}\right\}
			\end{aligned}
		\end{equation}
	\end{definition}
	
	\begin{definition}[Pessimism Event]
		\begin{equation}
			\begin{aligned} \widecheck{\Psi}_{k,h}:=&\left\{\forall h\le h'\le H:\boldsymbol{\mu}_{h'}\in\widecheck{\mathcal{C}}_{k,h'}\right\}\\
			\widecheck{\Psi}_{h}:=&\left\{\forall i\in[K],\forall h\le h'\le H:\boldsymbol{\mu}_{h'}\in\widecheck{\mathcal{C}}_{i,h'}\right\}
			\end{aligned}
		\end{equation}
	\end{definition}
	
	In this section, independent confidence sets $\widebar{\mathcal{C}}_{k,h},\widetilde{\mathcal{C}}_{k,h},\widecheck{\mathcal{C}}_{k,h}$ are built in Lemma~\ref{lm:barbeta}, \ref{lm:tildebeta}, and \ref{lm:checkbeta} respectively in Appendix~\ref{sec:appics}.
	These independent confidence sets are built to upper bound the variance of the considered value function.
	Specifically, the difference between the estimated variance of the constructed optimistic value function and the real variance of the optimal value function, i.e., $\big|[\mathbb{V}_{h} V_{h+1}^*](s_{h}^{k}, a_{h}^{k})-[\widehat{\mathbb{V}}_{k, h} \widehat{V}_{k, h+1}](s_{h}^{k}, a_{h}^{k})\big|$, is upper bounded in high probability in Lemma~\ref{lm:var}.
	In addition, the variance $[\mathbb{V}_h(\widehat{V}_{k,h+1}-V_{h+1}^*)](s_h^k,a_h^k)$ is also upper bounded in Lemma~\ref{lm:var2}.
	Subsequently, dependent confidence sets $\widehat{\mathcal{C}}^{(1)}_{k,h},\widehat{\mathcal{C}}^{(2)}_{k,h}$ can be built based on the independent confidence sets in Lemma~\ref{lm:rightarrowbeta},~\ref{lm:leftarrowbeta} respectively in Appendix~\ref{sec:appdcs}.
	Thus, the confidence set $\widehat{\mathcal{C}}_{k,h}$, the goal of this section, holds trivially if $\widehat{\mathcal{C}}^{(1)}_{k,h},\widehat{\mathcal{C}}^{(2)}_{k,h}$ both hold. 
	Finally,  Lemma~\ref{lm:cs} in the main paper is proved in Appendix~\ref{sec:appcs}.
	
	Before the formal proof begins, we give some necessary definitions.
	We first give definitions about-measurable space and filtration required for our proofs.
	\paragraph{Measurable Space}
	Note that the stochasticity in the transition probability of the MDP are the only source of randomness.
	Denote $\mathbb{P}$ as the gather of the distributions over state-action pair sequence $(\mathcal{S}\times\mathcal{A})^\mathbb{N}$, induced by the interconnection of policy obtained from LSVI-UCB$^+$ algorithm and the episodic linear MDP $\mathcal{M}$.
	Denote $\mathbb{E}$ as the corresponding expectation operator.
	Hence, all random variables can be defined over the sample space $\Omega=(\mathcal{S} \times \mathcal{A})^\mathbb{N}$.
	Thus, we work with the probability space given by the triplet $(\Omega, \mathcal{F}, \mathbb{P})$, where $\mathcal{F}$ is the product $\sigma$-algebra generated by the discrete $\sigma$-algebras underlying $\mathcal{S}$ and $\mathcal{A}$.
	
	\begin{definition}[Filtration]\label{def:ft}
		For any $k\in[K]$ and any $h\in[H]$, let $\mathcal{F}_{k, h}$ be the $\sigma$-algebra generated by the random variables representing the state-action pairs up to and including that appears in stage $h$ of episode $k$.
	\end{definition}
	
	\paragraph{Measurability}
	Thus, $[\widehat{\mathbb{V}}_{k,h}\widehat{V}_{k,h+1}](s_{h}^{k}, a_{h}^{k}), U_{k, h}, E_{k,h},\varsigma_{k,h},\widetilde{\sigma}_{k, h}, \widehat{\sigma}_{k, h},\widehat{\boldsymbol{\Lambda}}_{k+1, h}$
	are $\mathcal{F}_{k, h}$-measurable, $\widehat{\boldsymbol{\mu}}_{k+1, h}$ is $\mathcal{F}_{k, h+1}$-measurable, $\widehat{Q}_{k,h},\widecheck{Q}_{k,h},\widehat{V}_{k,h},\widecheck{V}_{k,h},\pi_h^k$ are $\mathcal{F}_{k-1, H}$-measurable, but not $\mathcal{F}_{k-1, h}$-measurable due to their backwards construction.
	
	\subsection{Independent Confidence Sets}\label{sec:appics}
	In this subsection, independent confidence sets $\widebar{\mathcal{C}}_{k,h},\widetilde{\mathcal{C}}_{k,h},\widecheck{\mathcal{C}}_{k,h}$ are built in Lemma~\ref{lm:barbeta}, \ref{lm:tildebeta}, \ref{lm:checkbeta}, respectively.
	During the construction of these confidence sets, it is unavoidable to build a uniform convergence argument by covering net of the encountered function class.
	Thus, we also present the definition of possibly encountered function classes in the following.
	
	\begin{definition}[Optimistic Value Function Class]\label{def:hatv}
		For fixed $J$ updating episode, let $\widehat{\mathcal{V}}$ denote a class of functions mapping from $\mathcal{S}$ to $\mathbb{R}$ with following parametric form
		$$
		\widehat{V}(\cdot)=\max _{a} \min_{1\le i\le J}\min\left\{\boldsymbol{w}_i^{\top} \boldsymbol{\phi}(\cdot, a)+\beta \sqrt{\boldsymbol{\phi}(\cdot, a)^{\top} \mathbf{\Lambda}_i^{-1} \boldsymbol{\phi}(\cdot, a)},H\right\},
		$$
		where the parameters $(\boldsymbol{w}_i, \beta, \Lambda_i)$ satisfy $\|\boldsymbol{w}_i\|_2 \leq L, \beta \in[0, B]$, the minimum eigenvalue satisfies $\lambda_{\min }(\mathbf{\Lambda}_i) \geq \lambda$, and $\sup_{s,a}\|\boldsymbol{\phi}(s,a)\|_2\le1$.
	\end{definition}
	
	\begin{definition}[Squared Optimistic Value Function Class]\label{def:hatv2}
		For fixed $J$ updating episode, let $\widehat{\mathcal{V}}^2$ denote a class of functions mapping from $\mathcal{S}$ to $\mathbb{R}$ with following parametric form
		$$
		\widehat{V}^2(\cdot)=\max_{a}\min_{1\le i\le J}\left[\min\left\{\boldsymbol{w}_i^{\top} \boldsymbol{\phi}(\cdot, a)+\beta \sqrt{\boldsymbol{\phi}(\cdot, a)^{\top} \mathbf{\Lambda}_i^{-1} \boldsymbol{\phi}(\cdot, a)},H\right\}\right]^2,
		$$
		where the parameters $(\boldsymbol{w}_i, \beta, \Lambda_i)$ satisfy $\|\boldsymbol{w}_i\|_2 \leq L, \beta \in[0, B]$, the minimum eigenvalue satisfies $\lambda_{\min }(\mathbf{\Lambda}_i) \geq \lambda$, and $\sup_{s,a}\|\boldsymbol{\phi}(s,a)\|_2\le1$.
	\end{definition}
	
	\begin{definition}[Pessimistic Value Function Class]\label{def:checkv}
		Let $\widecheck{\mathcal{V}}$ denote a class of functions mapping from $\mathcal{S}$ to $\mathbb{R}$ with following parametric form
		$$
		\widecheck{V}(\cdot)=\max\left\{\max _{a} \boldsymbol{w}^{\top} \boldsymbol{\phi}(\cdot, a)-\beta \sqrt{\boldsymbol{\phi}(\cdot, a)^{\top} \mathbf{\Lambda}^{-1} \boldsymbol{\phi}(\cdot, a)},0\right\},
		$$
		where the parameters $(\boldsymbol{w}, \beta, \Lambda)$ satisfy $\|\boldsymbol{w}\|_2 \leq L, \beta \in[0, B]$, the minimum eigenvalue satisfies $\lambda_{\min }(\mathbf{\Lambda}) \geq \lambda$, and $\sup_{s,a}\|\boldsymbol{\phi}(s,a)\|_2\le1$.
	\end{definition}
	
	Now we are ready to build four independent confidence sets $\widebar{\mathcal{C}}_{k,h},\widetilde{\mathcal{C}}_{k,h},\widecheck{\mathcal{C}}_{k,h}$.
	Since radius of independent confidence sets, i.e., $\widebar{\beta},\widetilde{\beta},\widecheck{\beta}$, will not become dominant terms in the final regret bound, we build these four confidence sets with traditional Hoeffding inequality (Lemma~\ref{lm:vectorhoeffding}) with covering net arguments.
	
	\begin{lemma}\label{lm:barbeta}
		In Algorithm~\ref{alg:plus}, for any $\delta\in(0,1)$, any $k\in[K]$ fixed $h\in[H]$, with probability at least $1-\delta/H$: 
		\begin{align*}
		    \boldsymbol{\mu}_{h}\in \widebar{\mathcal{C}}_{k, h}=\left\{\boldsymbol{\mu}:\left\|\left(\boldsymbol{\mu}-\widehat{\boldsymbol{\mu}}_{k, h}\right)\widehat{\boldsymbol{V}}_{k,h+1}\right\|_{\widehat{\boldsymbol{\Lambda}}_{k, h}} \leq \widebar{\beta}\right\},
		\end{align*}
		where
		\begin{align*}
		    \widebar{\beta}=\sqrt{H}\sqrt{d\log\left(1+\frac{K}{Hd\lambda}\right)+\log\left(\frac{H}{\delta}\right) + dJ\log\left(1+\frac{4KL}{H\sqrt{\lambda}}\right) + d^2J\log\left(1 + \frac{8K^2\widehat{B}^2\sqrt{d}}{H^2\lambda^2}\right)}+ H\sqrt{\lambda d} + 2.
		\end{align*}
		Here $J=dH\log(1 + K),L=W+K/\lambda$ and $\widehat{B}$ is a constant satisfying $\widehat{\beta}\le \widehat{B}$ with $\widehat{\beta}$ given in Lemma~\ref{lm:csf}.
		
		\begin{proof}
		    Initially, note that we have
			\begin{align}\label{eq:betatmp}
			    \begin{split}
			        &\left\|\left(\widehat{\boldsymbol{\mu}}_{k, h}-\boldsymbol{\mu}_h\right)\widehat{\boldsymbol{V}}_{k,h+1}\right\|_{\widehat{\boldsymbol{\Lambda}}_{k,h}}\\
			        =&\left\|\widehat{\mathbf{\Lambda}}_{k,h}^{-1}\left[-\lambda\boldsymbol{\mu}_h+\sum_{i=1}^{k-1}\widehat{\sigma}_{i, h}^{-2}\boldsymbol{\phi}(s_{h}^{i},a_{h}^{i}){\boldsymbol{\epsilon}_h^i}^\top\right]\widehat{\boldsymbol{V}}_{k,h+1}\right\|_{\widehat{\boldsymbol{\Lambda}}_{k, h}}\\
					=&\left\|-\lambda\boldsymbol{\mu}_h\widehat{\boldsymbol{V}}_{k,h+1}+\sum_{i=1}^{k-1}\widehat{\sigma}_{i, h}^{-2}\boldsymbol{\phi}(s_{h}^{i},a_{h}^{i}){\boldsymbol{\epsilon}_h^i}^\top\widehat{\boldsymbol{V}}_{k,h+1}\right\|_{\widehat{\mathbf{\Lambda}}_{k,h}^{-1}}\\
					\le&\left\|-\lambda\boldsymbol{\mu}_h\widehat{\boldsymbol{V}}_{k,h+1}\right\|_{\widehat{\mathbf{\Lambda}}_{k,h}^{-1}}+\left\|\sum_{i=1}^{k-1}\widehat{\sigma}_{i, h}^{-2}\boldsymbol{\phi}(s_{h}^{i},a_{h}^{i}){\boldsymbol{\epsilon}_h^i}^\top\widehat{\boldsymbol{V}}_{k,h+1}\right\|_{\widehat{\mathbf{\Lambda}}_{k,h}^{-1}}\\
					\le&\frac{1}{\sqrt{\lambda}}\cdot \lambda H\sqrt{d}+\left\|\sum_{i=1}^{k-1}\widehat{\sigma}_{i, h}^{-2}\boldsymbol{\phi}(s_{h}^{i},a_{h}^{i}){\boldsymbol{\epsilon}_h^i}^\top\widehat{\boldsymbol{V}}_{k,h+1}\right\|_{\widehat{\mathbf{\Lambda}}_{k,h}^{-1}},
			    \end{split}
			\end{align}
			where the first equality is due to Eq.~(\ref{eq:mud2}) in Lemma~\ref{lm:mud}, the first inequality is due to triangle inequality, and the second inequality holds since $\|\boldsymbol{\mu}_{ h}\widehat{\boldsymbol{V}}_{k,h+1}\|_2\le H\sqrt{d}$ and the minimum eigenvalue of $\widehat{\mathbf{\Lambda}}_{k,h}$ is no less than $\lambda$.
			
			Thus, we bound $\|\sum_{i=1}^{k-1}\widehat{\sigma}_{i, h}^{-2}\boldsymbol{\phi}(s_{h}^{i}, a_{h}^{i}){\boldsymbol{\epsilon}_h^i}^\top\widehat{\boldsymbol{V}}_{k,h+1}\|_{\widehat{\mathbf{\Lambda}}_{k,h}^{-1}}$ in the following.
			However, $\widehat{V}_{k,h+1}$ is $\mathcal{F}_{k,h}$-measurable, which brings obstacles in directly applying self-normalized bound for martingales. We need to build a uniform convergence argument for $\widehat{V}_{k,h+1}$.
			
			For any $(k,h)\in\times[K]\times[H]$, $\widehat{V}_{k,h}(\cdot)=\min\{\max_a\langle\boldsymbol{\theta}_h+\widehat{\boldsymbol{\mu}}_{k,h}\widehat{\boldsymbol{V}}_{k,h+1},\boldsymbol{\phi}(\cdot,a)\rangle+\widehat{\beta}\|\boldsymbol{\phi}(\cdot, a)\|_{\widehat{\mathbf{\Lambda}}_{k,h}^{-1}},H\}$ in Algorithm~\ref{alg:plus}.
			Moreover, we have
			\begin{align*}
			    \left\|\boldsymbol{\theta}_h+\widehat{\boldsymbol{\mu}}_{k,h}\widehat{\boldsymbol{V}}_{k,h+1}\right\|_2=&\left\|\boldsymbol{\theta}_h+\widehat{\mathbf{\Lambda}}_{k,h}^{-1}\sum_{i=1}^{k-1}\widehat{\sigma}_{i, h}^{-2}\boldsymbol{\phi}(s_{h}^{i},a_{h}^{i})\widehat{V}_{k,h+1}(s_{h+1}^{i})\right\|_2\\
				\le&W+\left\|\widehat{\mathbf{\Lambda}}_{k,h}^{-1}\sum_{i=1}^{k-1}\widehat{\sigma}_{i, h}^{-2}\boldsymbol{\phi}(s_{h}^{i},a_{h}^{i})\widehat{V}_{k,h+1}(s_{h+1}^{i})\right\|_2\\
				\le&W+\frac{H}{(\sqrt{H})^2}\left\|\widehat{\mathbf{\Lambda}}_{k,h}^{-1}\sum_{i=1}^{k-1}\boldsymbol{\phi}(s_{h}^{i},a_{h}^{i})\right\|_2
				\le W+K/\lambda,
			\end{align*}
			where the first inequality holds due to triangle inequality, the second inequality holds since $\widehat{V}_{k,h+1}(\cdot)\le H$ and $\widehat{\sigma}_{i,h}\ge\sqrt{H}$ for any $i\in[k]$, and the last inequality holds since $\lambda_{\min }(\mathbf{\Lambda}_{k,h})\ge\lambda$ and $\sup_{s, a}\|\boldsymbol{\phi}(s,a)\|_2\le1$.
			Subsequently, we claim $\widehat{V}_{k,h}\in\mathcal{\widehat{V}}$, where $\widehat{\mathcal{V}}$ is defined in Definition~\ref{def:hatv}, with $L=W+K/\lambda$ and $B=\widehat{B}$.
			Here, $\widehat{B}$ is a constant satisfying $\widehat{\beta}\le \widehat{B}$ with $\widehat{\beta}$ specified in Lemma~\ref{lm:csf}.
			
			Then, we fix a function $V(\cdot)\in\widehat{\mathcal{V}}: \mathcal{S} \mapsto[0, H]$.
			Let $\mathcal{G}_{i}=\mathcal{F}_{i, h}$, $\boldsymbol{x}_{i}=\widehat{\sigma}_{i, h}^{-1}\boldsymbol{\phi}(s_{h}^{i}, a_{h}^{i})$ and $\eta_{i}=\widehat{\sigma}_{i, h}^{-1}{\boldsymbol{\epsilon}_h^i}^\top\boldsymbol{V}=\widehat{\sigma}_{i, h}^{-1}\langle\boldsymbol{\mu}_{h}{\boldsymbol{V}},\boldsymbol{\phi}(s_h^i,a_h^i)\rangle-\widehat{\sigma}_{i, h}^{-1}V(s_{h+1}^{i})$. It is clear that $\boldsymbol{x}_{i}$ is $\mathcal{G}_{i}$-measurable and $\eta_{i}$ is $\mathcal{G}_{i+1}$-measurable. Since $\varsigma_{i}\ge\sqrt{H}$,  we have $\widehat{\sigma}_{i,h}\ge\sqrt{H}$. Besides, we have $\|\boldsymbol{x}_i\|_2\le1/\sqrt{H}$, $\mathbb{E}[\eta_{i}\mid \mathcal{G}_i]=0$, $|\eta_{i}|\le \sqrt{H}$ and $\mathbb{E}[\eta_{i}^2\mid \mathcal{G}_i]\le H$.
			By Lemma~\ref{lm:vectorhoeffding}, we obtain that, with probability at least $1-\delta / H$, for any $k\in[K]$ and fixed $h\in[H]$,
			\begin{align*}
			    &\left\|\sum_{i=1}^{k-1}\widehat{\sigma}_{i, h}^{-2}\boldsymbol{\phi}(s_{h}^{i},a_{h}^{i}){\boldsymbol{\epsilon}_h^i}^\top\boldsymbol{V}\right\|_{\widehat{\mathbf{\Lambda}}_{k,h}^{-1}}\le\sqrt{H}\sqrt{d\log\left(1+\frac{K}{Hd\lambda}\right)+\log\left(\frac{H}{\delta}\right)}.
			\end{align*}
			
			Denote the $\varepsilon$-cover of function class $\widehat{\mathcal{V}}$ as $\widehat{\mathcal{N}}_{\varepsilon}$.
			Consider an arbitrary $f \in \widehat{\mathcal{V}}$. From the definition of $\varepsilon$-cover, we know that for $f$, there exists a $V \in \widehat{\mathcal{N}}_{\varepsilon}$, such that $\|\boldsymbol{f}-\boldsymbol{V}\|_{\infty} \leq \varepsilon$. Since $\|{\boldsymbol{\epsilon}_h^i}^\top(\boldsymbol{f}-\boldsymbol{V})\|_2\le\|\boldsymbol{\epsilon}_h^i\|_1\|\boldsymbol{f}-\boldsymbol{V}\|_\infty\le2\varepsilon$ and $\|\sum_{i=1}^{k-1}\widehat{\sigma}_{i, h}^{-2}\boldsymbol{\phi}(s_{h}^{i}, a_{h}^{i})\|_{\widehat{\mathbf{\Lambda}}_{k,h}^{-1}}\le{K}/(H\sqrt{\lambda})$, we have
			\begin{align}\label{eq:barbetatmp1}
			    \left\|\sum_{i=1}^{k-1}\widehat{\sigma}_{i, h}^{-2}\boldsymbol{\phi}(s_{h}^{i},a_{h}^{i}){\boldsymbol{\epsilon}_h^i}^\top\left(\boldsymbol{f}-\boldsymbol{V}\right)\right\|_{\widehat{\mathbf{\Lambda}}_{k,h}^{-1}} \leq\frac{2\varepsilon K}{H\sqrt{\lambda}}.
			\end{align}
		    This further implies the following inequality holds with probability at least $1-\delta/H$:
			\begin{align}\label{eq:anyf}
			    \begin{split}
			        &\left\|\sum_{i=1}^{k-1}\widehat{\sigma}_{i, h}^{-2}\boldsymbol{\phi}(s_{h}^{i},a_{h}^{i}){\boldsymbol{\epsilon}_h^i}^\top\boldsymbol{f}\right\|_{\widehat{\mathbf{\Lambda}}_{k,h}^{-1}}\\
					\le& \left\|\sum_{i=1}^{k-1}\widehat{\sigma}_{i, h}^{-2}\boldsymbol{\phi}(s_{h}^{i},a_{h}^{i}){\boldsymbol{\epsilon}_h^i}^\top\boldsymbol{V}\right\|_{\widehat{\mathbf{\Lambda}}_{k,h}^{-1}}+\left\|\sum_{i=1}^{k-1}\widehat{\sigma}_{i, h}^{-2}\boldsymbol{\phi}(s_{h}^{i},a_{h}^{i}){\boldsymbol{\epsilon}_h^i}^\top\left(\boldsymbol{f}-\boldsymbol{V}\right)\right\|_{\widehat{\mathbf{\Lambda}}_{k,h}^{-1}}\\
					\leq & \left\|\sum_{i=1}^{k-1}\widehat{\sigma}_{i, h}^{-2}\boldsymbol{\phi}(s_{h}^{i},a_{h}^{i}){\boldsymbol{\epsilon}_h^i}^\top\boldsymbol{V}\right\|_{\widehat{\mathbf{\Lambda}}_{k,h}^{-1}}+\frac{2\varepsilon K}{H\sqrt{\lambda}}\\
					\leq &\sqrt{H}\sqrt{d\log\left(1+\frac{K}{Hd\lambda}\right)+\log\left(\frac{H}{\delta}\right)+\log\left|\widehat{\mathcal{N}}_\varepsilon\right|} + \frac{2\varepsilon K}{H\sqrt{\lambda}}.
			    \end{split}
			\end{align}
			where the first inequality is due to the triangle inequality, the second one holds by Eq.~(\ref{eq:barbetatmp1}), and the third inequality holds by a union bound over all functions in $\widehat{\mathcal{N}}_{\varepsilon}$ with
			\begin{align*}
			    \log\left|\widehat{\mathcal{N}}_{\varepsilon}\right| \leq dJ \log\left(1+\frac{4 L}{\varepsilon}\right)+d^{2}J \log\left(1+\frac{8\widehat{B}^{2}\sqrt{d}}{\lambda \varepsilon^{2}}\right)
			\end{align*}
			according to Lemma~\ref{lm:coverhatv} with $J=dH\log(1 + K)$ by Lemma~\ref{lm:numupdate}.
			
			Note that Eq.~(\ref{eq:betatmp}) holds and $\widehat{V}_{k,h+1}(\cdot)\in\widehat{\mathcal{V}}$. We have with probability at least $1-\delta/H$, for any $k\in[K]$ and fixed $h\in[H]$:
			\begin{align*}
			    &\left\|\left(\widehat{\boldsymbol{\mu}}_{k, h}-\boldsymbol{\mu}_h\right)\widehat{\boldsymbol{V}}_{k,h+1}\right\|_{\widehat{\boldsymbol{\Lambda}}_{k,h}}\le\left\|\sum_{i=1}^{k-1}\widehat{\sigma}_{i, h}^{-2}\boldsymbol{\phi}(s_{h}^{i},a_{h}^{i}){\boldsymbol{\epsilon}_h^i}^\top\widehat{\boldsymbol{V}}_{k,h+1}\right\|_{\widehat{\mathbf{\Lambda}}_{k,h}^{-1}}+H\sqrt{\lambda d}\\
				\le& \sqrt{H}\sqrt{d\log\left(1+\frac{K}{Hd\lambda}\right)+\log\left(\frac{H}{\delta}\right) + dJ\log\left(1+\frac{4KL}{H\sqrt{\lambda}}\right) + d^2J\log\left(1 + \frac{8K^2\widehat{B}^2\sqrt{d}}{H^2\lambda^2}\right)}+ H\sqrt{\lambda d} + 2\\
				=&\widebar{\beta},
			\end{align*}
			where the last inequality holds by Eq.~(\ref{eq:anyf}) and setting $\varepsilon = H\sqrt{\lambda}/K$.
		\end{proof}
	\end{lemma}
	
	After building the confidence set $\widebar{\mathcal{C}}_{k,h}$ in Lemma~\ref{lm:barbeta}, confidence sets $\widetilde{\mathcal{C}}_{k,h}$ and $\widecheck{\mathcal{C}}_{k,h}$ can be built similarly.
	
	\begin{lemma}\label{lm:tildebeta}
		In Algorithm~\ref{alg:plus}, for any $\delta\in(0,1)$, any $k\in[K]$ and fixed $h\in[H]$, with probability at least $1-\delta/H$: 
		\begin{align*}
		    \boldsymbol{\mu}_{h}\in \widetilde{\mathcal{C}}_{k, h}=\left\{\boldsymbol{\mu}:\left\|\left(\boldsymbol{\mu}-\widehat{\boldsymbol{\mu}}_{k, h}\right)\widehat{\boldsymbol{V}}_{k,h+1}^2\right\|_{\widehat{\boldsymbol{\Lambda}}_{k, h}} \leq \widetilde{\beta}\right\},
		\end{align*}
		where
		\begin{align*}
		    \widetilde{\beta}=\sqrt{H^3}\sqrt{d\log\left(1+\frac{K}{Hd\lambda}\right)+\log\left(\frac{H}{\delta}\right) + dJ\log\left(1+\frac{8KL}{\sqrt{\lambda}}\right) + d^2J\log\left(1 + \frac{32K^2\widehat{B}^2\sqrt{d}}{\lambda^2}\right)}+ H^2\sqrt{\lambda d} + 2.
		\end{align*}
		Here $J=dH\log(1 + K),L=W+K/\lambda$ and and $\widehat{B}$ is a constant satisfying $\widehat{\beta}\le \widehat{B}$ with $\widehat{\beta}$ given in Lemma~\ref{lm:csf}.
		\begin{proof}
			The proof of this lemma is almost the same as that of Lemma~\ref{lm:barbeta}, except for replacing $\widehat{V}_{k,h+1}$, $\widehat{\mathcal{V}}$, and $\widehat{\mathcal{N}}_{\varepsilon}$ by $\widehat{V}_{k,h+1}^2$, $\widehat{\mathcal{V}}^2$, and $\widehat{\mathcal{N}}_{\varepsilon}^2$, respectively.
			Here
			\begin{align*}
			    \log\left|\widehat{\mathcal{N}}^2_{\varepsilon}\right|\leq dJ\log \left(1+\frac{8LH}{\varepsilon}\right)+d^{2}J\log \left(1+\frac{32\widehat{B}^{2}H^2\sqrt{d}}{\lambda \varepsilon^{2}}\right),
			\end{align*}
			where $L=W+k/\lambda$ and $\widehat{B}$ is a constant satisfying $\widehat{\beta}\le \widehat{B}$ with $\widehat{\beta}$ given in Lemma~\ref{lm:csf}.
			
			After setting $\varepsilon=H\sqrt{\lambda}/K$, it can be proved that with probability at least $1-\delta/H$, for any $k\in[K]$ and fixed $h\in[H]$:
			\begin{align*}
			    &\left\|\left(\widehat{\boldsymbol{\mu}}_{k, h}-\boldsymbol{\mu}_h\right)\widehat{\boldsymbol{V}}_{k,h+1}^2\right\|_{\widehat{\boldsymbol{\Lambda}}_{k,h}}\\
				\le&\sqrt{H^3}\sqrt{d\log\left(1+\frac{K}{Hd\lambda}\right)+\log\left(\frac{H}{\delta}\right) + dJ\log\left(1+\frac{8KL}{\sqrt{\lambda}}\right) + d^2J\log\left(1 + \frac{32K^2\widehat{B}^2\sqrt{d}}{\lambda^2}\right)}+ H^2\sqrt{\lambda d} + 2\\
				=&\widetilde{\beta}.
			\end{align*}
		\end{proof}
	\end{lemma}
	
	\begin{lemma}\label{lm:checkbeta}
		In Algorithm~\ref{alg:plus}, for any $\delta\in(0,1)$, any $k\in[K]$ and fixed $h\in[H]$, with probability at least $1-\delta/H$: 
		\begin{align*}
		    \boldsymbol{\mu}_{h}\in \widecheck{\mathcal{C}}_{k, h}=\left\{\boldsymbol{\mu}:\left\|\left(\boldsymbol{\mu}-\widehat{\boldsymbol{\mu}}_{k, h}\right)\widecheck{\boldsymbol{V}}_{k,h+1}\right\|_{\widehat{\boldsymbol{\Lambda}}_{k, h}} \leq \widecheck{\beta}\right\},
		\end{align*}
		where
		\begin{align}\label{eq:achechbeta}
		    \widecheck{\beta}=\sqrt{H}\sqrt{d\log\left(1+\frac{K}{Hd\lambda}\right)+\log\left(\frac{H}{\delta}\right) + d\log\left(1+\frac{4KL}{H\sqrt{\lambda}}\right) + d^2\log\left(1 + \frac{8K^2\widecheck{B}^2\sqrt{d}}{H^2\lambda^2}\right)}+ H\sqrt{\lambda d} + 2.
		\end{align}
		Here $L=W+K/\lambda$ and $\widecheck{B}$ is a constant satisfying $\widecheck{\beta}\le \widecheck{B}$.
		\begin{proof}
			The proof of this lemma is similar to that of Lemma~\ref{lm:barbeta} except for replacing $\widehat{V}_{k,h+1}$, $\widehat{\mathcal{V}}$, and $\widehat{\mathcal{N}}_{\varepsilon}$ by $\widecheck{V}_{k,h+1}$, $\widecheck{\mathcal{V}}$, and $\widecheck{\mathcal{N}}_{\varepsilon}$, respectively, where
			\begin{align*}
			    \log\left|\widecheck{\mathcal{N}}_{\varepsilon}\right|\leq d \log \left(1+\frac{4L}{\varepsilon}\right)+d^{2} \log \left(1+\frac{8\widecheck{B}^{2}\sqrt{d}}{\lambda \varepsilon^{2}}\right).
			\end{align*}
			Here $L=W+K/\lambda$ and $\widecheck{B}$ is a constant satisfying $\widecheck{\beta}\le \widecheck{B}$ with $\widecheck{\beta}$ specified in Eq.~(\ref{eq:achechbeta}).
			
			After setting $\varepsilon=H\sqrt{\lambda}/K$, it can be proved that with probability at least $1-\delta/H$, for any $k\in[K]$ and fixed $h\in[H]$:
			\begin{align*}
			    &\left\|\left(\widehat{\boldsymbol{\mu}}_{k, h}-\boldsymbol{\mu}_h\right)\widecheck{\boldsymbol{V}}_{k,h+1}\right\|_{\widehat{\boldsymbol{\Lambda}}_{k,h}}\\
				\le&\sqrt{H}\sqrt{d\log\left(1+\frac{K}{Hd\lambda}\right)+\log\left(\frac{H}{\delta}\right) + d\log\left(1+\frac{4KL}{H\sqrt{\lambda}}\right) + d^2\log\left(1 + \frac{8K^2\widecheck{B}^2\sqrt{d}}{H^2\lambda^2}\right)}+ H\sqrt{\lambda d} + 2\\
				=&\widecheck{\beta}.
			\end{align*}
		\end{proof}
	\end{lemma}
	
	\subsection{Variance Upper Bound}\label{sec:appvar}
	In this section, we prove some necessary lemmas to build upper bounds of value function variances, including variances of $V_{h+1}^*(\cdot)$ and $[\widehat{V}_{k,h+1}-V_{h+1}^*](\cdot)$.
	Specifically, we present Lemma~\ref{lm:var} in Appendix~\ref{app:vstar} to bound the difference between the estimated variance of the constructed optimistic value function and the real variance of the optimal value function, i.e.,  $\big|[\mathbb{V}_{h} V_{h+1}^*](s_{h}^{k}, a_{h}^{k})-[\widehat{\mathbb{V}}_{k, h} \widehat{V}_{k, h+1}](s_{h}^{k}, a_{h}^{k})\big|$, with high probability.
	In addition, we also present Lemma~\ref{lm:var2} in Appendix~\ref{app:vdiet} to upper bound the variance $[\mathbb{V}_h(\widehat{V}_{k,h+1}-V_{h+1}^*)](s_h^k,a_h^k)$.

    \subsubsection{Variance of $V_{h+1}^*(\cdot)$}\label{app:vstar}
    Before proving Lemma~\ref{lm:var}, we first present Lemma~\ref{lm:sm}, which upper bounds $|[\mathbb{P}_h(\widehat{V}_{k,h+1}-V_{h+1}^*)](s_h^k,a_h^k)|$ under optimism and pessimism events $\widehat{\Psi}_{k,h+1}\cap\widecheck{\Psi}_{k,h+1}$, and serves as the building block for Lemma~\ref{lm:var}.
	\begin{lemma}\label{lm:sm}
		In Algorithm~\ref{alg:plus}, for any $k\in[K]$ and any $h\in[H]$, under $\widehat{\Psi}_{k,h+1}\cap\widecheck{\Psi}_{k,h+1}$, we have
		\begin{align*}
		    \left|\langle\boldsymbol{\mu}_{h}\widehat{\boldsymbol{V}}_{k,h+1},\boldsymbol{\phi}(s_h^k,a_h^k)\rangle-\langle\boldsymbol{\mu}_{h}\boldsymbol{V}_{h+1}^{*}, \boldsymbol{\phi}(s_{h}^{k},a_{h}^{k})\rangle\right|\le\left|\langle(\widehat{\boldsymbol{\mu}}_{k,h}-\boldsymbol{\mu}_{h})\widehat{\boldsymbol{V}}_{k,h+1},\boldsymbol{\phi}(s_h^k,a_h^k)\rangle\right|\\
			+\left|\langle\widehat{\boldsymbol{\mu}}_{k,h}\widehat{\boldsymbol{V}}_{k,h+1}-\widehat{\boldsymbol{\mu}}_{k, h}\widecheck{\boldsymbol{V}}_{k,h+1},\boldsymbol{\phi}(s_h^k,a_h^k)\rangle\right|+\left|\langle(\widehat{\boldsymbol{\mu}}_{k,h}-\boldsymbol{\mu}_{h})\widecheck{\boldsymbol{V}}_{k,h+1},\boldsymbol{\phi}(s_h^k,a_h^k)\rangle\right|.
		\end{align*}
		
		\begin{proof}
			By definition, we have
			\begin{align*}
			    &\left|\langle\boldsymbol{\mu}_{h}\widehat{\boldsymbol{V}}_{k,h+1},\boldsymbol{\phi}(s_h^k,a_h^k)\rangle-\langle\boldsymbol{\mu}_{h}\boldsymbol{V}_{h+1}^{*}, \boldsymbol{\phi}(s_{h}^{k},a_{h}^{k})\rangle\right|\\
				=&\left|\mathbb{P}_h\widehat{V}_{k,h+1}(s_h^k,a_h^k)-\mathbb{P}_hV_{h+1}^*(s_h^k,a_h^k)\right|=\mathbb{P}_h\widehat{V}_{k,h+1}(s_h^k,a_h^k)-\mathbb{P}_hV_{h+1}^*(s_h^k,a_h^k)\\
				\le&\mathbb{P}_h\widehat{V}_{k,h+1}(s_h^k,a_h^k)-\mathbb{P}_h\widecheck{V}_{k,h+1}(s_h^k,a_h^k)=\left|\mathbb{P}_h\widehat{V}_{k,h+1}(s_h^k,a_h^k)-\mathbb{P}_h\widecheck{V}_{k,h+1}(s_h^k,a_h^k)\right|,
			\end{align*}
			where the second equality holds since $\mathbb{P}_h$ is a valid distribution and $\widehat{V}_{k,h+1}(\cdot)\ge{V}_{h+1}^*(\cdot)$ under $\widehat{\Psi}_{k,h+1}$ by Lemma~\ref{lm:optimism}, the first inequality holds since $\mathbb{P}_h$ is a valid distribution plus ${V}_{h+1}^*(\cdot)\ge\widecheck{V}_{k,h+1}(\cdot)$ under $\widecheck{\Psi}_{k,h+1}$by Lemma~\ref{lm:optimism}, and the last equality holds since $\mathbb{P}_h$ is a valid distribution plus $\widehat{V}_{k,h+1}(\cdot)\ge\widecheck{V}_{k,h+1}(\cdot)$ under $\widehat{\Psi}_{k,h+1}\cap\widecheck{\Psi}_{k,h+1}$ by Lemma~\ref{lm:optimism}.
			
			Therefore,
			\begin{align*}
    			&\left|\mathbb{P}_h\widehat{V}_{k,h+1}(s_h^k,a_h^k)-\mathbb{P}_hV_{h+1}^*(s_h^k,a_h^k)\right|\le\left|\mathbb{P}_h\widehat{V}_{k,h+1}(s_h^k,a_h^k)-\mathbb{P}_h\widecheck{V}_{k,h+1}(s_h^k,a_h^k)\right|\\
    			=&\Big|\mathbb{P}_h\widehat{V}_{k,h+1}(s_h^k,a_h^k)-\widehat{\mathbb{P}}_{k,h}\widehat{V}_{k,h+1}(s_h^k,a_h^k)+\widehat{\mathbb{P}}_{k,h}\widehat{V}_{k,h+1}(s_h^k,a_h^k)-\widehat{\mathbb{P}}_{k,h}\widecheck{V}_{k,h+1}(s_h^k,a_h^k)\\
    			&+\widehat{\mathbb{P}}_{k,h}\widecheck{V}_{k,h+1}(s_h^k,a_h^k)-\mathbb{P}_h\widecheck{V}_{k,h+1}(s_h^k,a_h^k)\Big|\\
    			\le&\left|\mathbb{P}_h\widehat{V}_{k,h+1}(s_h^k,a_h^k)-\widehat{\mathbb{P}}_{k,h}\widehat{V}_{k,h+1}(s_h^k,a_h^k)\right|+\left|\widehat{\mathbb{P}}_{k,h}\left[\widehat{V}_{k,h+1}-\widecheck{V}_{k,h+1}\right](s_h^k,a_h^k)\right|\\
    			&+\left|\widehat{\mathbb{P}}_{k,h}\widecheck{V}_{k,h+1}(s_h^k,a_h^k)-\mathbb{P}_h\widecheck{V}_{k,h+1}(s_h^k,a_h^k)\right|\\
    			=&\left|\langle(\widehat{\boldsymbol{\mu}}_{k,h}-\boldsymbol{\mu}_{h})\widehat{\boldsymbol{V}}_{k,h+1},\boldsymbol{\phi}(s_h^k,a_h^k)\rangle\right|+\left|\langle\widehat{\boldsymbol{\mu}}_{k,h}(\widehat{\boldsymbol{V}}_{k,h+1}-\widecheck{\boldsymbol{V}}_{k,h+1}),\boldsymbol{\phi}(s_h^k,a_h^k)\rangle\right|\\
    			&+\left|\langle(\widehat{\boldsymbol{\mu}}_{k,h}-\boldsymbol{\mu}_{h})\widecheck{\boldsymbol{V}}_{k,h+1},\boldsymbol{\phi}(s_h^k,a_h^k)\rangle\right|.
			\end{align*}
		\end{proof}
	\end{lemma}
	
	Based on Lemma~\ref{lm:sm}, we are ready to present Lemma~\ref{lm:var}.
	
	\begin{lemma}\label{lm:var}
		In Algorithm~\ref{alg:plus}, for any $k\in[K]$ and any $h\in[H]$, under $\widehat{\Psi}_{k,h+1}\cap\widecheck{\Psi}_{k,h+1}$, we have
		\begin{align*}
			\left|\left[\mathbb{V}_{h} V_{h+1}^*\right](s_{h}^{k},a_{h}^{k})-\left[\widehat{\mathbb{V}}_{k, h} \widehat{V}_{k, h+1}\right](s_{h}^{k},a_{h}^{k})\right|\le&\min\left\{\left\|(\widehat{\boldsymbol{\mu}}_{k, h}-\boldsymbol{\mu}_h)\widehat{\boldsymbol{V}}_{k,h+1}^2\right\|_{\widehat{\boldsymbol{\Lambda}}_{k, h}}\left\|\boldsymbol{\phi}(s_{h}^{k},a_{h}^{k})\right\|_{\widehat{\mathbf{\Lambda}}_{k,h}^{-1}}+4H\Delta_{k,h},2H^2\right\},
		\end{align*}
		where
		\begin{align*}
			\Delta_{k,h}=&\left\|(\widehat{\boldsymbol{\mu}}_{k, h}-\boldsymbol{\mu}_h)\widehat{\boldsymbol{V}}_{k,h+1}\right\|_{\widehat{\boldsymbol{\Lambda}}_{k, h}}\left\|\boldsymbol{\phi}(s_{h}^{k},a_{h}^{k})\right\|_{\widehat{\mathbf{\Lambda}}_{k,h}^{-1}}+\left|\langle\widehat{\boldsymbol{\mu}}_{k,h}(\widehat{\boldsymbol{V}}_{k,h+1}-\widecheck{\boldsymbol{V}}_{k,h+1}),\boldsymbol{\phi}(s_h^k,a_h^k)\rangle\right|\\
			&+\left\|(\widehat{\boldsymbol{\mu}}_{k, h}-\boldsymbol{\mu}_h)\widecheck{\boldsymbol{V}}_{k,h+1}\right\|_{\widehat{\boldsymbol{\Lambda}}_{k, h}}\left\|\boldsymbol{\phi}(s_{h}^{k},a_{h}^{k})\right\|_{\widehat{\mathbf{\Lambda}}_{k,h}^{-1}}.
		\end{align*}
	\end{lemma}
	\begin{proof}
		
		By definition, we have
		\begin{align*}
		    &\left|\left[\mathbb{V}_{h} V_{h+1}^*\right](s_{h}^{k},a_{h}^{k})-\left[\widehat{\mathbb{V}}_{k, h} \widehat{V}_{k, h+1}\right](s_{h}^{k},a_{h}^{k})\right|\\
			=&\Big|\langle\boldsymbol{\mu}_{h}{\boldsymbol{V}_{h+1}^{*}}^2, \boldsymbol{\phi}(s_{h}^{k},a_{h}^{k})\rangle-\left[\langle\widehat{\boldsymbol{\mu}}_{k,h}\widehat{\boldsymbol{V}}^2_{k,h+1},\boldsymbol{\phi}(s_h^k,a_h^k)\rangle\right]_{[0,H^2]}\\
			&+\left\{\left[\langle\widehat{\boldsymbol{\mu}}_{k,h}\widehat{\boldsymbol{V}}_{k,h+1},\boldsymbol{\phi}(s_h^k,a_h^k)\rangle\right]_{[0,H]}\right\}^2-\left[\langle\boldsymbol{\mu}_{h}\boldsymbol{V}_{h+1}^*, \boldsymbol{\phi}(s_{h}^{k},a_{h}^{k})\rangle\right]^{2}\Big| \\
			\leq&\underbrace{\left|\left[\langle\widehat{\boldsymbol{\mu}}_{k,h}\widehat{\boldsymbol{V}}^2_{k,h+1},\boldsymbol{\phi}(s_h^k,a_h^k)\rangle\right]_{[0,H^2]}-\langle\boldsymbol{\mu}_{h}{\boldsymbol{V}_{h+1}^{*}}^2, \boldsymbol{\phi}(s_{h}^{k},a_{h}^{k})\rangle\right|}_{I_{1}}\\
			&+\underbrace{\left|\left\{\left[\langle\widehat{\boldsymbol{\mu}}_{k,h}\widehat{\boldsymbol{V}}_{k,h+1},\boldsymbol{\phi}(s_h^k,a_h^k)\rangle\right]_{[0,H]}\right\}^2-\left[\langle\boldsymbol{\mu}_{h}\boldsymbol{V}_{h+1}^*, \boldsymbol{\phi}(s_{h}^{k},a_{h}^{k})\rangle\right]^{2}\right|}_{I_{2}}
		\end{align*}
		where the inequality holds due to the triangle inequality. We bound $I_1$ first.
		\begin{align*}
		    I_1=&\left|\left[\langle\widehat{\boldsymbol{\mu}}_{k,h}\widehat{\boldsymbol{V}}^2_{k,h+1},\boldsymbol{\phi}(s_h^k,a_h^k)\rangle\right]_{[0,H^2]}-\langle\boldsymbol{\mu}_{h}{\boldsymbol{V}_{h+1}^{*}}^2, \boldsymbol{\phi}(s_{h}^{k},a_{h}^{k})\rangle\right|\\
			=&\Big|\left[\langle\widehat{\boldsymbol{\mu}}_{k,h}\widehat{\boldsymbol{V}}^2_{k,h+1},\boldsymbol{\phi}(s_h^k,a_h^k)\rangle\right]_{[0,H^2]}-\langle\boldsymbol{\mu}_{h}\widehat{\boldsymbol{V}}_{k,h+1}^2, \boldsymbol{\phi}(s_{h}^{k},a_{h}^{k})\rangle\\
			&+\langle\boldsymbol{\mu}_{h}\widehat{\boldsymbol{V}}^2_{k,h+1},\boldsymbol{\phi}(s_h^k,a_h^k)\rangle-\langle\boldsymbol{\mu}_{h}{\boldsymbol{V}_{h+1}^{*}}^2, \boldsymbol{\phi}(s_{h}^{k},a_{h}^{k})\rangle\Big|\\
			\le&\left|\left[\langle\widehat{\boldsymbol{\mu}}_{k,h}\widehat{\boldsymbol{V}}^2_{k,h+1},\boldsymbol{\phi}(s_h^k,a_h^k)\rangle\right]_{[0,H^2]}-\langle\boldsymbol{\mu}_{h}\widehat{\boldsymbol{V}}_{k,h+1}^2, \boldsymbol{\phi}(s_{h}^{k},a_{h}^{k})\rangle\right|\\
			&+\left|\langle\boldsymbol{\mu}_{h}\widehat{\boldsymbol{V}}^2_{k,h+1},\boldsymbol{\phi}(s_h^k,a_h^k)\rangle-\langle\boldsymbol{\mu}_{h}{\boldsymbol{V}_{h+1}^{*}}^2, \boldsymbol{\phi}(s_{h}^{k},a_{h}^{k})\rangle\right|\\
			\le&\left|\langle(\widehat{\boldsymbol{\mu}}_{k,h}-\boldsymbol{\mu}_{h})\widehat{\boldsymbol{V}}^2_{k,h+1},\boldsymbol{\phi}(s_h^k,a_h^k)\rangle\right|+2H\left|\langle\boldsymbol{\mu}_{h}\widehat{\boldsymbol{V}}_{k,h+1},\boldsymbol{\phi}(s_h^k,a_h^k)\rangle-\langle\boldsymbol{\mu}_{h}\boldsymbol{V}_{h+1}^{*}, \boldsymbol{\phi}(s_{h}^{k},a_{h}^{k})\rangle\right|\\
			\le&\left|\langle(\widehat{\boldsymbol{\mu}}_{k,h}-\boldsymbol{\mu}_{h})\widehat{\boldsymbol{V}}^2_{k,h+1},\boldsymbol{\phi}(s_h^k,a_h^k)\rangle\right|+2H\Big[\left|\langle(\widehat{\boldsymbol{\mu}}_{k,h}-\boldsymbol{\mu}_{h})\widehat{\boldsymbol{V}}_{k,h+1},\boldsymbol{\phi}(s_h^k,a_h^k)\rangle\right|\\
			&+\left|\langle\widehat{\boldsymbol{\mu}}_{k,h}(\widehat{\boldsymbol{V}}_{k,h+1}-\widecheck{\boldsymbol{V}}_{k,h+1}),\boldsymbol{\phi}(s_h^k,a_h^k)\rangle\right|+\left|\langle(\widehat{\boldsymbol{\mu}}_{k,h}-\boldsymbol{\mu}_{h})\widecheck{\boldsymbol{V}}_{k,h+1},\boldsymbol{\phi}(s_h^k,a_h^k)\rangle\right|\Big]\\
			\le&\left\|(\widehat{\boldsymbol{\mu}}_{k, h}-\boldsymbol{\mu}_h)\widehat{\boldsymbol{V}}_{k,h+1}^2\right\|_{\widehat{\boldsymbol{\Lambda}}_{k, h}}\left\|\boldsymbol{\phi}(s_{h}^{k},a_{h}^{k})\right\|_{\widehat{\mathbf{\Lambda}}_{k,h}^{-1}}+2H\Big[\left\|(\widehat{\boldsymbol{\mu}}_{k, h}-\boldsymbol{\mu}_h)\widehat{\boldsymbol{V}}_{k,h+1}\right\|_{\widehat{\boldsymbol{\Lambda}}_{k, h}}\left\|\boldsymbol{\phi}(s_{h}^{k},a_{h}^{k})\right\|_{\widehat{\mathbf{\Lambda}}_{k,h}^{-1}}\\
			&+\left|\langle\widehat{\boldsymbol{\mu}}_{k,h}(\widehat{\boldsymbol{V}}_{k,h+1}-\widecheck{\boldsymbol{V}}_{k,h+1}),\boldsymbol{\phi}(s_h^k,a_h^k)\rangle\right|+\left\|(\widehat{\boldsymbol{\mu}}_{k, h}-\boldsymbol{\mu}_h)\widecheck{\boldsymbol{V}}_{k,h+1}\right\|_{\widehat{\boldsymbol{\Lambda}}_{k, h}}\left\|\boldsymbol{\phi}(s_{h}^{k},a_{h}^{k})\right\|_{\widehat{\mathbf{\Lambda}}_{k,h}^{-1}}\Big]
		\end{align*}
		where the first inequality holds due to the triangle inequality, the second inequality holds since $\mathbb{P}_h(\cdot\mid s_h^k,a_h^k)=\langle\boldsymbol{\mu}_{h}(\cdot),\boldsymbol{\boldsymbol{\phi}}(s_h^k, a_h^k)\rangle$ is valid distribution and $\widehat{V}_{k,h+1}(\cdot),V_{h+1}^*(\cdot)\in[0,H]$, the third inequality holds due to Lemma~\ref{lm:sm} under $\widehat{\Psi}_{k,h+1}\cap\widecheck{\Psi}_{k,h+1}$, and the last inequality holds due to the Cauchy-Schwarz inequality.

		For $I_2$, we have
		\begin{align*}
			I_2=&\left|\left\{\left[\langle\widehat{\boldsymbol{\mu}}_{k,h}\widehat{\boldsymbol{V}}_{k,h+1},\boldsymbol{\phi}(s_h^k,a_h^k)\rangle\right]_{[0,H]}\right\}^2-\left[\langle\boldsymbol{\mu}_{h}\boldsymbol{V}_{h+1}^*, \boldsymbol{\phi}(s_{h}^{k},a_{h}^{k})\rangle\right]^{2}\right|\\
			=&\left|\left[\langle\widehat{\boldsymbol{\mu}}_{k,h}\widehat{\boldsymbol{V}}_{k,h+1},\boldsymbol{\phi}(s_h^k,a_h^k)\rangle\right]_{[0,H]}+\langle\boldsymbol{\mu}_{h}\boldsymbol{V}_{h+1}^*, \boldsymbol{\phi}(s_{h}^{k},a_{h}^{k})\rangle\right|\\
			&\cdot\left|\left[\langle\widehat{\boldsymbol{\mu}}_{k,h}\widehat{\boldsymbol{V}}_{k,h+1},\boldsymbol{\phi}(s_h^k,a_h^k)\rangle\right]_{[0,H]}-\langle\boldsymbol{\mu}_{h}\boldsymbol{V}_{h+1}^*, \boldsymbol{\phi}(s_{h}^{k},a_{h}^{k})\rangle\right|\\
			\le&2H\left|\left[\langle\widehat{\boldsymbol{\mu}}_{k,h}\widehat{\boldsymbol{V}}_{k,h+1},\boldsymbol{\phi}(s_h^k,a_h^k)\rangle\right]_{[0,H]}-\langle\boldsymbol{\mu}_{h}\boldsymbol{V}_{h+1}^*, \boldsymbol{\phi}(s_{h}^{k},a_{h}^{k})\rangle\right|\\
			\le&2H\Big[\left|\langle(\widehat{\boldsymbol{\mu}}_{k,h}-\boldsymbol{\mu}_{h})\widehat{\boldsymbol{V}}_{k,h+1},\boldsymbol{\phi}(s_h^k,a_h^k)\rangle\right|+\left|\langle\widehat{\boldsymbol{\mu}}_{k,h}(\widehat{\boldsymbol{V}}_{k,h+1}-\widecheck{\boldsymbol{V}}_{k,h+1}),\boldsymbol{\phi}(s_h^k,a_h^k)\rangle\right|\\
			&+\left|\langle(\widehat{\boldsymbol{\mu}}_{k,h}-\boldsymbol{\mu}_{h})\widecheck{\boldsymbol{V}}_{k,h+1},\boldsymbol{\phi}(s_h^k,a_h^k)\rangle\right|\Big]\\
			\le&2H\Big[\left\|(\widehat{\boldsymbol{\mu}}_{k, h}-\boldsymbol{\mu}_h)\widehat{\boldsymbol{V}}_{k,h+1}\right\|_{\widehat{\boldsymbol{\Lambda}}_{k, h}}\left\|\boldsymbol{\phi}(s_{h}^{k},a_{h}^{k})\right\|_{\widehat{\mathbf{\Lambda}}_{k,h}^{-1}}+\left|\langle\widehat{\boldsymbol{\mu}}_{k,h}(\widehat{\boldsymbol{V}}_{k,h+1}-\widecheck{\boldsymbol{V}}_{k,h+1}),\boldsymbol{\phi}(s_h^k,a_h^k)\rangle\right|\\
			&+\left\|(\widehat{\boldsymbol{\mu}}_{k, h}-\boldsymbol{\mu}_h)\widecheck{\boldsymbol{V}}_{k,h+1}\right\|_{\widehat{\boldsymbol{\Lambda}}_{k, h}}\left\|\boldsymbol{\phi}(s_{h}^{k},a_{h}^{k})\right\|_{\widehat{\mathbf{\Lambda}}_{k,h}^{-1}}\Big]
		\end{align*}
		where the first inequality holds since $0\le[\langle\widehat{\boldsymbol{\mu}}_{k,h}\widehat{\boldsymbol{V}}_{k,h+1},\boldsymbol{\phi}(s_h^k,a_h^k)\rangle]_{[0,H]}\le H$, $0\le\mathbb{P}_{h}V_{h+1}^*(s_h^k,a_h^k)=\langle\boldsymbol{\mu}_{h}\boldsymbol{V}_{h+1}^*, \boldsymbol{\phi}(s_{h}^{k}, a_{h}^{k})\rangle\le H$, the second inequality holds due to Lemma~\ref{lm:sm} under $\widehat{\Psi}_{k,h+1}\cap\widecheck{\Psi}_{k,h+1}$, and the third inequality holds due to the Cauchy-Schwarz inequality.
		Combining the upper bound of $I_1,I_2$ in above two inequalities and using the fact that $I_1$ and $I_2$ are both bounded by $H^2$ give the final result.
	\end{proof}
	
	\subsubsection{Variance of $[\widehat{V}_{k,h+1}-V_{h+1}^*](\cdot)$}\label{app:vdiet}
	\begin{lemma}\label{lm:var2}
	    In Algorithm~\ref{alg:plus}, for any $k\in[K]$ and any $h\in[H]$, under $\widehat{\Psi}_{k,h+1}\cap\widecheck{\Psi}_{k,h+1}$, we have
	    $V_{h+1}^*(\cdot)\le\widehat{V}_{k,h+1}(\cdot)$.
	    Moreover, for any function $V:\mathcal{S}\mapsto[0,H]$ satisfying $V_{h+1}^*(\cdot)-\zeta\le V(\cdot)\le\widehat{V}_{k,h+1}(\cdot)+\zeta$, where $0<\zeta\le H$ is a constant 
	    we have
	    \begin{align*}
	        \left[\mathbb{V}_{h}(V-V_{h+1}^*)\right](s_{h}^{k},a_{h}^{k})\le \min\Big\{2H\Big[&\langle\widehat{\boldsymbol{\mu}}_{k,h}\widehat{\boldsymbol{V}}_{k,h+1},\boldsymbol{\phi}(s_h^k,a_h^k)\rangle-\langle\widehat{\boldsymbol{\mu}}_{k,h}\widecheck{\boldsymbol{V}}_{k,h+1},\boldsymbol{\phi}(s_h^k,a_h^k)\rangle\\
	        &+\left\|(\widehat{\boldsymbol{\mu}}_{k, h}-\boldsymbol{\mu}_h)\widehat{\boldsymbol{V}}_{k,h+1}\right\|_{\widehat{\boldsymbol{\Lambda}}_{k, h}}\left\|\boldsymbol{\phi}(s_{h}^{k},a_{h}^{k})\right\|_{\widehat{\mathbf{\Lambda}}_{k,h}^{-1}}\\
	        &+\left\|(\widehat{\boldsymbol{\mu}}_{k, h}-\boldsymbol{\mu}_h)\widecheck{\boldsymbol{V}}_{k,h+1}\right\|_{\widehat{\boldsymbol{\Lambda}}_{k, h}}\left\|\boldsymbol{\phi}(s_{h}^{k},a_{h}^{k})\right\|_{\widehat{\mathbf{\Lambda}}_{k,h}^{-1}}+\zeta\Big],H^2\Big\}.
	    \end{align*}
	\begin{proof}
	    Initially, we have $V_{h+1}^*(\cdot)\le\widehat{V}_{k,h+1}(\cdot)$ by Lemma~\ref{lm:optimism} under $\widehat{\Psi}_{k,h+1}$.
	    Denote $\widetilde{V}(\cdot)=V(\cdot)-V_{h+1}^*(\cdot)$.
	    By definition of the variance, we have
	    \begin{align}\label{eq:var2t1}
	        \begin{split}
	            &[\mathbb{V}_{h}(V-V_{h+1}^*)](s_{h}^{k},a_{h}^{k})=[\mathbb{V}_{h}\widetilde{V}](s_{h}^{k},a_{h}^{k})=\mathbb{P}_{h}\widetilde{V}^2(s_h^k,a_h^k)-[\mathbb{P}_{h}\widetilde{V}(s_h^k,a_h^k)]^2\le\mathbb{P}_{h}\widetilde{V}^2(s_h^k,a_h^k)\\
	            \le&2H\cdot\mathbb{P}_{h}|\widetilde{V}|(s_h^k,a_h^k)\le 2H\cdot\left[\mathbb{P}_{h}\widehat{V}_{k,h+1}(s_h^k,a_h^k)-\mathbb{P}_{h}V_{h+1}^*(s_h^k,a_h^k)+\zeta\right]\\
	            \le&2H\cdot\left[\mathbb{P}_{h}\widehat{V}_{k,h+1}(s_h^k,a_h^k)-\mathbb{P}_{h}\widecheck{V}_{k,h+1}(s_h^k,a_h^k)+\zeta\right],
	        \end{split}
	    \end{align}
	    where the first inequality holds since $[\mathbb{P}_{h}\widetilde{V}(s_h^k,a_h^k)]^2\ge0$,
	    the second and third inequalities both holds since $-H\le-\zeta\le V(\cdot)-V_{h+1}^*(\cdot)\le \widehat{V}_{k,h+1}(\cdot)-V_{h+1}^*+\zeta\le 2H$, and the last inequality holds since $\widecheck{V}_{k,h+1}(\cdot)\le V_{h+1}^*(\cdot)$ under $\widecheck{\Psi}_{k,h+1}$ by Lemma~\ref{lm:optimism}.

	    On the other hand, by Cauchy-Schwarz inequality, we have
	    \begin{align}
	        \left|\mathbb{P}_{h}\widehat{V}_{k,h+1}(s_h^k,a_h^k)-\widehat{\mathbb{P}}_{k,h}\widehat{V}_{k,h+1}(s_h^k,a_h^k)\right|\le\left\|(\widehat{\boldsymbol{\mu}}_{k,h}-\boldsymbol{\mu}_h)\widehat{\boldsymbol{V}}_{k,h+1}\right\|_{\widehat{\boldsymbol{\Lambda}}_{k, h}}\left\|\boldsymbol{\phi}(s_{h}^{k},a_{h}^{k})\right\|_{\widehat{\mathbf{\Lambda}}_{k,h}^{-1}}\label{eq:var2t2}\\
	        \left|\mathbb{P}_{h}\widecheck{V}_{k,h+1}(s_h^k,a_h^k)-\widehat{\mathbb{P}}_{k,h}\widecheck{V}_{k,h+1}(s_h^k,a_h^k)\right|\le\left\|(\widehat{\boldsymbol{\mu}}_{k,h}-\boldsymbol{\mu}_h)\widecheck{\boldsymbol{V}}_{k,h+1}\right\|_{\widehat{\boldsymbol{\Lambda}}_{k, h}}\left\|\boldsymbol{\phi}(s_{h}^{k},a_{h}^{k})\right\|_{\widehat{\mathbf{\Lambda}}_{k,h}^{-1}}\label{eq:var2t3}
	    \end{align}
	    Combining Eq~(\ref{eq:var2t1}), Eq~(\ref{eq:var2t2}), Eq~(\ref{eq:var2t3}) and using the fact that  $[\mathbb{V}_{h}(V-V_{h+1}^*)](s_{h}^{k}, a_{h}^{k})\le H^2$ give the final result.
	\end{proof}
	\end{lemma}

	\subsection{Dependent Confidence Sets}\label{sec:appdcs}
	Based on independent confidence sets $\widebar{\mathcal{C}}_{k,h},\widetilde{\mathcal{C}}_{k,h},\widecheck{\mathcal{C}}_{k,h}$ built above and Lemma~\ref{lm:var}, \ref{lm:var2}, dependent confidence sets $\widehat{\mathcal{C}}^{(1)}_{k,h},\widehat{\mathcal{C}}^{(2)}_{k,h}$ are built in Lemma~\ref{lm:rightarrowbeta},~\ref{lm:leftarrowbeta},  respectively.
	As a results, the confidence set $\widehat{\mathcal{C}}_{k,h}$, the goal of this section, holds trivially if $\widehat{\mathcal{C}}^{(1)}_{k,h},\widehat{\mathcal{C}}^{(2)}_{k,h}$ both hold.
	We build confidence sets $\widehat{\mathcal{C}}^{(1)}_{k,h}$ and $\widehat{\mathcal{C}}^{(2)}_{k,h}$ elegantly because the radius of the confidence set $\widehat{\mathcal{C}}_{k,h}$ will exactly determine the sharpness of the regret obtained by LSVI-UCB$^+$ algorithm. 
	In particular, we utilize the conservatism of elliptical potentials, which is detailed in Remark~\ref{rm:small} in the main paper.
	To formally utilize this property, we first present Lemma~\ref{lm:smallR} to keep the magnitude of the considered MDS small with the conservatism of elliptical potentials.
	
	\begin{lemma}\label{lm:smallR}
	    In Algorithm~\ref{alg:plus}, for any $k\in[K]$ and any $h\in[H]$, we have
	    \begin{align*}
	        \widehat{\sigma}_{k,h}^{-1}\cdot\min\left\{1,\left\|\widehat{\sigma}_{k,h}^{-1}\boldsymbol{\phi}(s_h^i,a_h^i)\right\|_{\widehat{\boldsymbol{\Lambda}}_{k,h}^{-1}}\right\}\le\frac{1}{H^2\sqrt{d^5}}.
	    \end{align*}
	    \begin{proof}
	        In Algorithm~\ref{alg:plus}, for any $k\in[K]$ and any $h\in[H]$, we have following two cases:
			\begin{itemize}
				\item If $\left\|\widetilde{\sigma}_{i, h}^{-1}\boldsymbol{\phi}(s_{h}^i,a_{h}^i)\right\|_{\widetilde{\mathbf{\Lambda}}_{i,h}^{-1}}\le1/(H^3d^5)$, then $\varsigma_{k,h}=\sqrt{H}$ such that $\widehat{\sigma}_{i,h}=\widetilde{\sigma}_{i,h}$.
				In this case, we have
				\begin{align*}
				    \left\|\widehat{\sigma}_{k,h}^{-1}\boldsymbol{\phi}(s_h^i,a_h^i)\right\|_{\widehat{\mathbf{\Lambda}}_{k,h}^{-1}}=\left\|\widetilde{\sigma}_{k,h}^{-1}\boldsymbol{\phi}(s_h^i,a_h^i)\right\|_{\widehat{\mathbf{\Lambda}}_{k,h}^{-1}}	\le&\sqrt{H^3d^5}\left\|\widetilde{\sigma}_{k,h}^{-1}\boldsymbol{\phi}(s_h^i,a_h^i)\right\|_{\widetilde{\mathbf{\Lambda}}_{k,h}^{-1}}\le\sqrt{H^3d^5}/(H^3d^{5})\le1/\sqrt{H^3d^5}.
				\end{align*}
				where the inequality holds since $H^3d^5\cdot\widehat{\mathbf{\Lambda}}_{k,h}\succeq \widetilde{\mathbf{\Lambda}}_{k,h}$ by following facts:
				\begin{align*}
				    H^3d^5\cdot\widehat{\mathbf{\Lambda}}_{k,h}=H^3d^5\cdot\left(\sum_{i=1}^{k-1}\widehat{\sigma}_{i,h}^{-2}\boldsymbol{\phi}(s_h^i,a_h^i)\boldsymbol{\phi}(s_h^j,a_h^j)^\top+\lambda\mathbf{I}\right)\succeq\sum_{i=1}^{k-1}\widetilde{\sigma}_{i,h}^{-2}\boldsymbol{\phi}(s_h^i,a_h^i)\boldsymbol{\phi}(s_h^i,a_h^i)^\top+\lambda\mathbf{I}=\widetilde{\mathbf{\Lambda}}_{k,h},
				\end{align*}
			    where the inequality holds since $\widehat{\sigma}_{i,h}\le\sqrt{H^3d^5}\widetilde{\sigma}_{i,h}$ in Algorithm~\ref{alg:plus}, which implies $H^3d^5\cdot\widehat{\mathbf{\Lambda}}_{k,h}-\widetilde{\mathbf{\Lambda}}_{k,h}$ is a semi-positive definite matrix.
	
				Therefore, the conclusion holds in this case since $\widehat{\sigma}_{k,h}\ge\sqrt{H}$.
				
				\item Otherwise, $\varsigma_{k,h}=H^2\sqrt{d^5}$, such that $\widehat{\sigma}_{k,h}\ge H^2\sqrt{d^5}$.
				In this case, the conclusion still holds since $\min\{1,\|\widehat{\sigma}_{k,h}^{-1}\boldsymbol{\phi}(s_h^i,a_h^i)\|_{\widehat{\boldsymbol{\Lambda}}_{k,h}^{-1}}\}\le1$.
			\end{itemize}
	    \end{proof}
	\end{lemma}
	Now we are ready to prove Lemma~\ref{lm:rightarrowbeta} which builds the dependent confidence set $\widehat{\mathcal{C}}^{(1)}_{k,h}$, based on independent confidence sets $\widebar{\mathcal{C}}_{k,h}\cap\widetilde{\mathcal{C}}_{k,h}\cap\widecheck{\mathcal{C}}_{k,h}$ and Lemma~\ref{lm:var}.
	Indeed, the confidence set $\widehat{\mathcal{C}}^{(1)}_{k, h}$ corresponds to the deviation term of the form $[(\widehat{\mathbb{P}}_{k,h}-\mathbb{P}_{h})V_{h+1}^*](s_h^k,a_h^k)$ in the main paper.
	\begin{lemma}\label{lm:rightarrowbeta}
		In Algorithm~\ref{alg:plus}, for any $\delta\in(0,1)$, any $k\in[K]$ and fixed $h\in[H]$,  under $\widehat{\Psi}_{h+1}\cap\widecheck{\Psi}_{h+1}$, with probability at least $1-4\delta/H$:
		\begin{align*}
		    \boldsymbol{\mu}_h\in\widehat{\mathcal{C}}^{(1)}_{k,h}\cap\widebar{\mathcal{C}}_{k,h}\cap\widetilde{\mathcal{C}}_{k,h}\cap\widecheck{\mathcal{C}}_{k,h},
		\end{align*}
		where
		\begin{align}
			\widehat{\mathcal{C}}^{(1)}_{k, h}=&\left\{\boldsymbol{\mu}:\left\|\left(\boldsymbol{\mu}-\widehat{\boldsymbol{\mu}}_{k, h}\right){\boldsymbol{V}}_{h+1}^*\right\|_{\widehat{\boldsymbol{\Lambda}}_{k, h}} \leq \widehat{\beta}^{(1)}\right\},\notag\\
			\widehat{\beta}^{(1)}=&8 \sqrt{d\log\left(1+\frac{K}{Hd\lambda}\right) \log \left(\frac{4K^{2}H}{\delta}\right)}+4\log\left(\frac{4K^{2}H}{\delta}\right)+H\sqrt{\lambda d}.\label{eq:arightbeta}
		\end{align}
		\begin{proof}
			Let $\mathcal{G}_{i}=\mathcal{F}_{i, h}$, $\boldsymbol{x}_{i}=\widehat{\sigma}_{i, h}^{-1}\boldsymbol{\phi}(s_{h}^{i},a_{h}^{i})$, $\mathbf{Z}_i=\lambda\mathbf{I}+\sum_{j=1}^{i}\boldsymbol{x}_i\boldsymbol{x}_i^\top$, and
			$\eta_{i}=\widehat{\sigma}_{i,h}^{-1}{\boldsymbol{\epsilon}_h^i}^\top\boldsymbol{V}_{h+1}^*\cdot\mathds{1}\{\boldsymbol{\mu}_h\in\widebar{\mathcal{C}}_{i,h}\cap\widetilde{\mathcal{C}}_{i,h}\cap\widecheck{\mathcal{C}}_{i,h}\}\cdot\mathds{1}\{\widehat{\Psi}_{i,h+1}\cap\widecheck{\Psi}_{i,h+1}\}$.
			Now that $V_{h+1}^*$ is a fixed function and $\mathds{1}\{\boldsymbol{\mu}_h\in\widebar{\mathcal{C}}_{i,h}\cap\widetilde{\mathcal{C}}_{i,h}\cap\widecheck{\mathcal{C}}_{i,h}\}\cdot\mathds{1}\{\widehat{\Psi}_{i,h+1}\cap\widecheck{\Psi}_{i,h+1}\}$ is $\mathcal{G}_i$-measurable, it is clear that $\boldsymbol{x}_{i}$ are $\mathcal{G}_{i}$-measurable and $\eta_{i}$ is $\mathcal{G}_{i+1}$-measurable.
			
			Besides, we have $\mathbb{E}[\eta_{i}\mid \mathcal{G}_i]=0$.
			Since $\widehat{\sigma}_{i,h}\ge\varsigma_{i,h}\ge\sqrt{H}$, $|\eta_i| \le \sqrt{H}$ and $\|\boldsymbol{x}_i\|_2\le1/\sqrt{H}$.
			In particular, we claim $|\eta_i\cdot\min\{1,\|\boldsymbol{x}_i\|_{\mathbf{Z}_{i-1}^{-1}}\}|\le1$ because of the following three facts:
			(i) $|{\boldsymbol{\epsilon}_h^i}^\top\boldsymbol{V}_{h+1}^*\cdot\mathds{1}\{\boldsymbol{\mu}_h\in\widebar{\mathcal{C}}_{i,h}\cap\widetilde{\mathcal{C}}_{i,h}\cap\widecheck{\mathcal{C}}_{i,h}\}\cdot\mathds{1}\{\widehat{\Psi}_{i,h+1}\cap\widecheck{\Psi}_{i,h+1}\}|\le H$ holds by $|V_{h+1}*(\cdot)|\le H$; (ii) $\widehat{\sigma}_{i,h}^{-1}\cdot\min\{1,\|\boldsymbol{x}_i\|_{\boldsymbol{Z}_{i-1}^{-1}}\}\le1/(H^2\sqrt{d^5})$ holds by Lemma~\ref{lm:smallR}; and (iii) $|\mathds{1}\{\boldsymbol{\mu}_h\in\widebar{\mathcal{C}}_{i,h}\cap\widetilde{\mathcal{C}}_{i,h}\cap\widecheck{\mathcal{C}}_{i,h}\}\cdot\mathds{1}\{\widehat{\Psi}_{i,h+1}\cap\widecheck{\Psi}_{i,h+1}\}|\le1$.
			
			Furthermore, it holds that
			\begin{align*}
				\mathbb{E}[\eta_{i}^2\mid\mathcal{G}_i]=&\widehat{\sigma}_{i, h}^{-2}\cdot\mathds{1}\left\{\boldsymbol{\mu}_h\in\widebar{\mathcal{C}}_{i,h}\cap\widetilde{\mathcal{C}}_{i,h}\cap\widecheck{\mathcal{C}}_{i,h}\right\}\cdot\mathds{1}\left\{\widehat{\Psi}_{i,h+1}\cap\widecheck{\Psi}_{i,h+1}\right\}[\mathbb{V}_{h} V_{h+1}^*](s_{h}^{i},a_{h}^{i})\\
				\le&\widehat{\sigma}_{i, h}^{-2}\cdot\mathds{1}\left\{\boldsymbol{\mu}_h\in\widebar{\mathcal{C}}_{i,h}\cap\widetilde{\mathcal{C}}_{i,h}\cap\widecheck{\mathcal{C}}_{i,h}\right\}\Big[\left[\widehat{\mathbb{V}}_{i, h} \widehat{V}_{i, h+1}\right](s_{h}^{i},a_{h}^{i})\\
				&+\min\left\{\left\|\left(\widehat{\boldsymbol{\mu}}_{i,h}-\boldsymbol{\mu}_h\right)\widehat{\boldsymbol{V}}_{i,h+1}^2\right\|_{\widehat{\boldsymbol{\Lambda}}_{i,h}}\left\|\boldsymbol{\phi}(s_{h}^{i},a_{h}^{i})\right\|_{\widehat{\mathbf{\Lambda}}_{i,h}^{-1}}+4H\Delta_{i,h},2H^2\right\}\Big]\\
				\le&\widehat{\sigma}_{i,h}^{-2}\Big[\left[\widehat{\mathbb{V}}_{i,h}\widehat{V}_{i,h+1}\right](s_{h}^{i},a_{h}^{i})+\min\left\{\widetilde{\beta}\left\|\boldsymbol{\phi}(s_{h}^{i},a_{h}^{i})\right\|_{\widehat{\mathbf{\Lambda}}_{i,h}^{-1}}+4H\delta_{i,h},2H^2\right\}\Big]\\
				\le&1,
			\end{align*}
			where the first inequality holds due to Lemma~\ref{lm:var} under $\widehat{\Psi}_{i,h+1}\cap\widecheck{\Psi}_{i,h+1}$, the second inequality holds due to the definition of indicator function, and the last inequality holds due to the definition of $\widehat{\sigma}_{i,h}$. Here $\Delta_{i,h}$ and $\delta_{i,h}$ are given by
			\begin{align*}
				\Delta_{i,h}=&\left\|\left(\widehat{\boldsymbol{\mu}}_{i,h}-\boldsymbol{\mu}_h\right)\widehat{\boldsymbol{V}}_{i,h+1}\right\|_{\widehat{\boldsymbol{\Lambda}}_{i,h}}\left\|\boldsymbol{\phi}(s_{h}^{i},a_{h}^{i})\right\|_{\widehat{\mathbf{\Lambda}}_{i,h}^{-1}}+\left|\langle\widehat{\boldsymbol{\mu}}_{i,h}\left(\widehat{\boldsymbol{V}}_{i,h+1}-\widecheck{\boldsymbol{V}}_{i,h+1}\right),\boldsymbol{\phi}(s_h^i,a_h^i)\rangle\right|\\
				&+\left\|\left(\widehat{\boldsymbol{\mu}}_{i, h}-\boldsymbol{\mu}_h\right)\widecheck{\boldsymbol{V}}_{i,h+1}\right\|_{\widehat{\boldsymbol{\Lambda}}_{i,h}}\left\|\boldsymbol{\phi}(s_{h}^{i},a_{h}^{i})\right\|_{\widehat{\mathbf{\Lambda}}_{i,h}^{-1}},\\
				\delta_{i,h}=&\widebar{\beta}\left\|\boldsymbol{\phi}(s_{h}^{i},a_{h}^{i})\right\|_{\widehat{\mathbf{\Lambda}}_{i,h}^{-1}}+\left|\langle\widehat{\boldsymbol{\mu}}_{i,h}\left(\widehat{\boldsymbol{V}}_{i,h+1}-\widecheck{\boldsymbol{V}}_{i,h+1}\right),\boldsymbol{\phi}(s_h^i,a_h^i)\rangle\right|+\widecheck{\beta}\left\|\boldsymbol{\phi}(s_{h}^{i},a_{h}^{i})\right\|_{\widehat{\mathbf{\Lambda}}_{i,h}^{-1}}. 
			\end{align*}
			Then, by Lemma~\ref{lm:selffull}, with probability at least $1-\delta / H$, for all $k\in[K]$ and fixed $h\in[H]$,
			\begin{align}\label{eq:rightarrowbeta1}
				\begin{split}
					&\left\|\sum_{i=1}^{k-1}\widehat{\sigma}_{i,h}^{-2}\boldsymbol{\phi}(s_{h}^{i},a_{h}^{i}){\boldsymbol{\epsilon}_h^i}^\top\boldsymbol{V}_{h+1}^*\mathds{1}\left\{\boldsymbol{\mu}_h\in\widebar{\mathcal{C}}_{i,h}\cap\widetilde{\mathcal{C}}_{i,h}\cap\widecheck{\mathcal{C}}_{i,h}\right\}\mathds{1}\left\{\widehat{\Psi}_{i,h+1}\cap\widecheck{\Psi}_{i,h+1}\right\}\right\|_{\widehat{\mathbf{\Lambda}}_{k,h}^{-1}}\\
					\leq& 8 \sqrt{d\log\left(1+\frac{K}{Hd\lambda}\right) \log \left(\frac{4K^{2}H}{\delta}\right)}+4\log\left(\frac{4K^{2}H}{\delta}\right).
				\end{split}
			\end{align}
			
			Denote $\mathcal{E}_h^{(1)}$ as the event that $\boldsymbol{\mu}_h\in\bigcap_{k\in[K]}\widebar{\mathcal{C}}_{k,h}\cap\widetilde{\mathcal{C}}_{k,h}\cap\widecheck{\mathcal{C}}_{k,h}$ and Eq.~(\ref{eq:rightarrowbeta1}) hold, which happens with probability at least $1-4\delta / H$ by taking a union bound.
			In addition, we claim that with probability at least $1-4\delta / H$, for all $k\in[K]$ and fixed $h\in[H]$,
			\begin{align*}
			    &\left\|\sum_{i=1}^{k-1}\widehat{\sigma}_{i,h}^{-2}\boldsymbol{\phi}(s_{h}^{i},a_{h}^{i}){\boldsymbol{\epsilon}_h^i}^\top\boldsymbol{V}_{h+1}^*\cdot\mathds{1}\left\{\boldsymbol{\mu}_h\in\widebar{\mathcal{C}}_{i,h}\cap\widetilde{\mathcal{C}}_{i,h}\cap\widecheck{\mathcal{C}}_{i,h}\right\}\cdot\mathds{1}\left\{\widehat{\Psi}_{i,h+1}\cap\widecheck{\Psi}_{i,h+1}\right\}\right\|_{\widehat{\mathbf{\Lambda}}_{k,h}^{-1}}\\
				=&\left\|\sum_{i=1}^{k-1}\widehat{\sigma}_{i,h}^{-2}\boldsymbol{\phi}(s_{h}^{i},a_{h}^{i}){\boldsymbol{\epsilon}_h^i}^\top\boldsymbol{V}_{h+1}^*\cdot\mathds{1}\left\{\widehat{\Psi}_{i,h+1}\cap\widecheck{\Psi}_{i,h+1}\right\}\right\|_{\widehat{\mathbf{\Lambda}}_{k,h}^{-1}}\\
				\leq& 8 \sqrt{d\log\left(1+\frac{K}{Hd\lambda}\right) \log \left(\frac{4K^{2}H}{\delta}\right)}+4\log\left(\frac{4K^{2}H}{\delta}\right),
			\end{align*}
			where the equality holds since under event $\mathcal{E}_h^{(1)}$, for any $i\in[K]$, $\mathds{1}\{\boldsymbol{\mu}_h\in\widebar{\mathcal{C}}_{i,h}\cap\widetilde{\mathcal{C}}_{i,h}\cap\widecheck{\mathcal{C}}_{i,h}\}=1$.
			Moreover, if we further assume $\widehat{\Psi}_{h+1}\cap\widecheck{\Psi}_{h+1}$ holds, which means $\mathds{1}\{\widehat{\Psi}_{i,h+1}\cap\widecheck{\Psi}_{i,h+1}\}=1$ for any $i\in[k]$, then with probability at least $1-4\delta/H$, for any $k\in[K]$ and fixed $h\in[H]$:
			\begin{align*}
			    \left\|\left(\widehat{\boldsymbol{\mu}}_{k, h}-\boldsymbol{\mu}_h\right)\boldsymbol{V}_{h+1}^*\right\|_{\widehat{\boldsymbol{\Lambda}}_{k,h}} \leq 8 \sqrt{d\log\left(1+\frac{K}{Hd\lambda}\right) \log \left(\frac{4K^{2}H}{\delta}\right)}+4\log\left(\frac{4K^{2}H}{\delta}\right)+H\sqrt{\lambda d}=\widehat{\beta}^{(1)}.
			\end{align*}
			since
			\begin{align*}
			    \left\|\left(\widehat{\boldsymbol{\mu}}_{k, h}-\boldsymbol{\mu}_h\right)\boldsymbol{V}_{h+1}^*\right\|_{\widehat{\boldsymbol{\Lambda}}_{k,h}}
			    \le H\sqrt{\lambda d}+\left\|\sum_{i=1}^{k-1}\widehat{\sigma}_{i,h}^{-2}\boldsymbol{\phi}(s_{h}^{i},a_{h}^{i}){\boldsymbol{\epsilon}_h^i}^\top\boldsymbol{V}_{h+1}^*\right\|_{\widehat{\mathbf{\Lambda}}_{k,h}^{-1}}
			\end{align*}
			with a similar argument as in Eq.~(\ref{eq:betatmp}).
			Thus, we conclude that for any $k\in[K]$ and fixed $h\in[H]$,  under $\widehat{\Psi}_{h+1}\cap\widecheck{\Psi}_{h+1}$, with probability at least $1-4\delta/H$:
			\begin{align*}
			    \boldsymbol{\mu}_h\in\widehat{\mathcal{C}}^{(1)}_{k,h}\cap\widebar{\mathcal{C}}_{k,h}\cap\widetilde{\mathcal{C}}_{k,h}\cap\widecheck{\mathcal{C}}_{k,h}.
			\end{align*}
		\end{proof}
	\end{lemma}
	
	Subsequently, we prove Lemma~\ref{lm:leftarrowbeta} which builds the dependent confidence set $\widehat{\mathcal{C}}^{(2)}_{k,h}$, based on independent confidence sets $\widebar{\mathcal{C}}_{k,h}\cap\widecheck{\mathcal{C}}_{k,h}$ and Lemma~\ref{lm:var2}.
	The confidence set $\widehat{\mathcal{C}}^{(2)}_{k, h}$ corresponds to the deviation term of the form $[(\widehat{\mathbb{P}}_{k,h}-\mathbb{P}_{h})(\widehat{V}_{k,h+1}-V_{h+1}^*)](s_h^k,a_h^k)$ in main paper, which is controlled to be small in LSVI-UCB$^+$.
	
	\begin{lemma}\label{lm:leftarrowbeta}
		In Algorithm~\ref{alg:plus}, for any $\delta\in(0,1)$, any $k\in[K]$ and fixed $h\in[H]$, under $\widehat{\Psi}_{h+1}\cap\widecheck{\Psi}_{h+1}$, with probability at least $1-3\delta/H$:
		\begin{align*}
		    \boldsymbol{\mu}_h\in\widehat{\mathcal{C}}^{(2)}_{k,h}\cap\widebar{\mathcal{C}}_{k,h}\cap\widecheck{\mathcal{C}}_{k,h}.
		\end{align*}
		where
		\begin{align}
		    \widehat{\mathcal{C}}^{(2)}_{k, h}=&\left\{\boldsymbol{\mu}:\left\|\left(\boldsymbol{\mu}-\widehat{\boldsymbol{\mu}}_{k, h}\right)\left(\widehat{\boldsymbol{V}}_{k,h+1}-{\boldsymbol{V}}_{h+1}^*\right)\right\|_{\widehat{\boldsymbol{\Lambda}}_{k, h}} \leq \widehat{\beta}^{(2)}\right\},\notag\\
		    \widehat{\beta}^{(2)}=&8\sqrt{\frac{2}{Hd^2}\log\left(1+\frac{K}{Hd\lambda}\right) \left[\log\left(\frac{4K^{2}H}{\delta}\right)+dJ \log\left(1+\frac{4KL}{H\sqrt{\lambda}}\right)+d^{2}J \log\left(1+\frac{8K^2\widehat{B}^{2}\sqrt{d}}{H^2\lambda^2}\right)\right]}\notag\\
		    &+\frac{4}{H\sqrt{d^5}}\left[\log\left(\frac{4K^{2}H}{\delta}\right)+dJ \log\left(1+\frac{4KL}{H\sqrt{\lambda}}\right)+d^{2}J \log\left(1+\frac{8K^2\widehat{B}^{2}\sqrt{d}}{H^2\lambda^2}\right)\right]+H\sqrt{\lambda d}+2\label{eq:aleftbeta}.
		\end{align}
		Here $J=dH\log(1 + K),L=W+K/\lambda$ and $\widehat{B}$ is a constant satisfying $\widehat{\beta}\le \widehat{B}$ with $\widehat{\beta}$ given in Lemma~\ref{lm:csf}.
		\begin{proof}
		    It suffices to upper bound  $\|\sum_{i=1}^{k-1}\widehat{\sigma}_{i, h}^{-2}\boldsymbol{\phi}(s_{h}^{i},a_{h}^{i}){\boldsymbol{\epsilon}_h^i}^\top(\widehat{\boldsymbol{V}}_{k,h+1}-\boldsymbol{V}_{h+1}^*)\|_{\widehat{\mathbf{\Lambda}}_{k,h}^{-1}}$ with a  similar argument as Eq.~(\ref{eq:betatmp}).
		    Besides, we need to  build a uniform convergence argument by covering net since $\widehat{\boldsymbol{V}}_{k,h+1}(\cdot)$ is $\mathcal{F}_{k-1,h}$-measurable. 
		    As stated in the proof of Lemma~\ref{lm:barbeta}, $\widehat{V}_{k,h}\in\mathcal{\widehat{V}}$, where $\widehat{\mathcal{V}}$ is defined in Definition~\ref{def:hatv}, with $J=dH\log(1 + K),L=W+K/\lambda$ and $B=\widehat{B}$.
			Here, $\widehat{B}$ is a constant satisfying $\widehat{\beta}\le \widehat{B}$ with $\widehat{\beta}$ specified in Lemma~\ref{lm:csf}.
		    
			Then, for a fixed function $V(\cdot)\in\widehat{\mathcal{V}}: \mathcal{S} \mapsto[0, H]$ and a constant $\zeta=H\sqrt{\lambda}/K$, let $\mathcal{G}_{i}=\mathcal{F}_{i, h}$, $\boldsymbol{x}_{i}=\widehat{\sigma}_{i, h}^{-1}\boldsymbol{\phi}(s_{h}^{i}, a_{h}^{i})$, $\mathbf{Z}_i=\lambda\mathbf{I}+\sum_{j=1}^{i}\boldsymbol{x}_i\boldsymbol{x}_i^\top$, and
			$\eta_{i}=\widehat{\sigma}_{i, h}^{-1}{\boldsymbol{\epsilon}_h^i}^\top(\boldsymbol{V}-\boldsymbol{V}_{h+1}^*)\cdot\mathds{1}\{\boldsymbol{\mu}_h\in\widebar{\mathcal{C}}_{i,h}\cap\widecheck{\mathcal{C}}_{i,h}\}\cdot\mathds{1}\{\boldsymbol{V}_{h+1}^*-\zeta\le\boldsymbol{V}\le\widehat{\boldsymbol{V}}_{i,h+1}+\zeta\}\cdot\mathds{1}\{\widehat{\Psi}_{i,h+1}\cap\widecheck{\Psi}_{i,h+1}\}$.
			
			Since $V(\cdot)$ and $V(\cdot)_{h+1}^*$ are fixed functions, and $\widehat{V}_{i,h+1}$ and $\mathds{1}\{\boldsymbol{\mu}_h\in\widebar{\mathcal{C}}_{i,h}\cap\widecheck{\mathcal{C}}_{i,h}\}\cdot\mathds{1}\{\widehat{\Psi}_{i,h+1}\cap\widecheck{\Psi}_{i,h+1}\}$ are $\mathcal{G}_i$-measurable, it is clear that $\boldsymbol{x}_{i}$ is $\mathcal{G}_{i}$-measurable and $\eta_{i}$ is $\mathcal{G}_{i+1}$-measurable.
			Besides, we have $\mathbb{E}[\eta_{i}\mid \mathcal{G}_i]=0$.
			Since $\widehat{\sigma}_{i,h}\ge\varsigma_{i,h}\ge\sqrt{H}$, $|\eta_i| \le \sqrt{H}$ and $\|\boldsymbol{x}_i\|_2\le1/\sqrt{H}$.

			Similar to the proof in Lemma~\ref{lm:rightarrowbeta}, we claim $|\eta_i\cdot\min\{1,\|\boldsymbol{x}_i\|_{\mathbf{Z}_{i-1}^{-1}}\}|\le1/H\sqrt{d^5}$ because of the following three facts: (i) $|{\boldsymbol{\epsilon}_h^i}^\top(\boldsymbol{V}-\boldsymbol{V}_{h+1}^*)\cdot\mathds{1}\{\boldsymbol{V}_{h+1}^*-\zeta\le\boldsymbol{V}\le\widehat{\boldsymbol{V}}_{i,h+1}+\zeta\}\cdot\mathds{1}\{\widehat{\Psi}_{i,h+1}\cap\widecheck{\Psi}_{i,h+1}\}|\le H$ holds by $|(V(\cdot)-V_{h+1}^*(\cdot))\cdot\mathds{1}\{\boldsymbol{V}_{h+1}^*-\zeta\le\boldsymbol{V}\le\widehat{\boldsymbol{V}}_{i,h+1}+\zeta\}|\le H$; (ii) $\widehat{\sigma}_{i,h}^{-1}\cdot\min\{1,\|\boldsymbol{x}_i\|_{\boldsymbol{Z}_{i-1}^{-1}}\}\le1/(H^2\sqrt{d^5})$ holds by Lemma~\ref{lm:smallR}; and (iii) $|\mathds{1}\{\boldsymbol{\mu}_h\in\widebar{\mathcal{C}}_{i,h}\cap\widecheck{\mathcal{C}}_{i,h}\}\cdot\mathds{1}\{\widehat{\Psi}_{i,h+1}\cap\widecheck{\Psi}_{i,h+1}\}|\le1$.
			Furthermore, it holds that
			\begin{align*}
				\mathbb{E}[\eta_{i}^2\mid\mathcal{G}_i]=&\widehat{\sigma}_{i, h}^{-2}\cdot\mathds{1}\left\{\boldsymbol{\mu}_h\in\widebar{\mathcal{C}}_{i,h}\cap\widecheck{\mathcal{C}}_{i,h}\right\}\cdot\mathds{1}\left\{\boldsymbol{V}_{h+1}^*-\zeta\le\boldsymbol{V}\le\widehat{\boldsymbol{V}}_{i,h+1}+\zeta\right\}\cdot\mathds{1}\left\{\widehat{\Psi}_{i,h+1}\cap\widecheck{\Psi}_{i,h+1}\right\}\\
				&\cdot[\mathbb{V}_{h}(V-V_{h+1}^*)](s_{h}^{i},a_{h}^{i})\\
				\le&\widehat{\sigma}_{i, h}^{-2}\cdot\mathds{1}\left\{\boldsymbol{\mu}_h\in\widebar{\mathcal{C}}_{i,h}\cap\widecheck{\mathcal{C}}_{i,h}\right\}\cdot 2H\Big[\widehat{\mathbb{P}}_{k,h}\widehat{V}_{k,h+1}(s_h^k,a_h^k)-\widehat{\mathbb{P}}_{k,h}\widecheck{V}_{k,h+1}(s_h^k,a_h^k)\\
	            &+\left\|\left(\widehat{\boldsymbol{\mu}}_{k, h}-\boldsymbol{\mu}_h\right)\widehat{\boldsymbol{V}}_{k,h+1}\right\|_{\widehat{\boldsymbol{\Lambda}}_{k, h}}\left\|\boldsymbol{\phi}(s_{h}^{k},a_{h}^{k})\right\|_{\widehat{\mathbf{\Lambda}}_{k,h}^{-1}}\\
	            &+\left\|\left(\widehat{\boldsymbol{\mu}}_{k, h}-\boldsymbol{\mu}_h\right)\widecheck{\boldsymbol{V}}_{k,h+1}\right\|_{\widehat{\boldsymbol{\Lambda}}_{k, h}}\left\|\boldsymbol{\phi}(s_{h}^{k},a_{h}^{k})\right\|_{\widehat{\mathbf{\Lambda}}_{k,h}^{-1}}+\zeta\Big]\\
				\le&\widehat{\sigma}_{i, h}^{-2}\cdot 2H\Big[\widehat{\mathbb{P}}_{k,h}\widehat{V}_{k,h+1}(s_h^k,a_h^k)-\widehat{\mathbb{P}}_{k,h}\widecheck{V}_{k,h+1}(s_h^k,a_h^k)+\widebar{\beta}\left\|\boldsymbol{\phi}(s_{h}^{k},a_{h}^{k})\right\|_{\widehat{\mathbf{\Lambda}}_{k,h}^{-1}}+\widecheck{\beta}\left\|\boldsymbol{\phi}(s_{h}^{k},a_{h}^{k})\right\|_{\widehat{\mathbf{\Lambda}}_{k,h}^{-1}}+\zeta\Big]\\
				\le&\frac{2}{Hd^3},
			\end{align*}
			where the first inequality holds due to Lemma~\ref{lm:var2} under $\widehat{\Psi}_{i,h+1}\cap\widecheck{\Psi}_{i,h+1}$, the second inequality holds due to the definition of indicator function, and the last inequality holds due to the definition of $\widehat{\sigma}_{i,h}$.
			
			Then, by Lemma~\ref{lm:selffull}, for all $k\in[K]$ and fixed $h\in[H]$, with probability at least $1-\delta / H$:
			\begin{align*}
				    &\Big\|\sum_{i=1}^{k-1}\widehat{\sigma}_{i,h}^{-2}\boldsymbol{\phi}(s_{h}^{i},a_{h}^{i}){\boldsymbol{\epsilon}_h^i}^\top\left(\boldsymbol{V}-\boldsymbol{V}_{h+1}^*\right)\\
				    &\cdot\mathds{1}\left\{\boldsymbol{\mu}_h\in\widebar{\mathcal{C}}_{i,h}\cap\widecheck{\mathcal{C}}_{i,h}\right\}\cdot\mathds{1}\left\{\boldsymbol{V}_{h+1}^*-\zeta\le\boldsymbol{V}\le\widehat{\boldsymbol{V}}_{i,h+1}+\zeta\right\}\cdot\mathds{1}\left\{\widehat{\Psi}_{i,h+1}\cap\widecheck{\Psi}_{i,h+1}\right\}\Big\|_{\widehat{\mathbf{\Lambda}}_{k,h}^{-1}}\\
				    \leq& 8 \sqrt{\frac{2}{Hd^2}\log\left(1+\frac{K}{Hd\lambda}\right) \log \left(\frac{4K^{2}H}{\delta}\right)}+\frac{4}{H\sqrt{d^5}}\log\left(\frac{4K^{2}H}{\delta}\right).
			\end{align*}
			
			We further proceed our proof under the event that  $\widehat{\Psi}_{h+1}\cap\widecheck{\Psi}_{h+1}$ holds, which implies $\mathds{1}\{\widehat{\Psi}_{i,h+1}\cap\widecheck{\Psi}_{i,h+1}\}=1$ for any $i\in[k]$.
			Denote $\mathcal{E}_h^{(2)}$ as the event that $\boldsymbol{\mu}_h\in\bigcap_{k\in[K]}\widebar{\mathcal{C}}_{k,h}\cap\widecheck{\mathcal{C}}_{k,h}$ and the above inequality holds, which happens with probability at least $1-3\delta / H$ by taking a union bound.
			In addition, we claim that under $\widehat{\Psi}_{h+1}\cap\widecheck{\Psi}_{h+1}$, with probability at least $1-3\delta / H$, for all $k\in[K]$ and fixed $h\in[H]$,
			\begin{align*}
			    &\Big\|\sum_{i=1}^{k-1}\widehat{\sigma}_{i,h}^{-2}\boldsymbol{\phi}(s_{h}^{i},a_{h}^{i}){\boldsymbol{\epsilon}_h^i}^\top\left(\boldsymbol{V}-\boldsymbol{V}_{h+1}^*\right)\\
			    &\cdot\mathds{1}\left\{\boldsymbol{\mu}_h\in\widebar{\mathcal{C}}_{i,h}\cap\widecheck{\mathcal{C}}_{i,h}\right\}\cdot\mathds{1}\left\{\boldsymbol{V}_{h+1}^*-\zeta\le\boldsymbol{V}\le\widehat{\boldsymbol{V}}_{i,h+1}+\zeta\right\}\cdot\mathds{1}\left\{\widehat{\Psi}_{i,h+1}\cap\widecheck{\Psi}_{i,h+1}\right\}\Big\|_{\widehat{\mathbf{\Lambda}}_{k,h}^{-1}}\\
				=&\left\|\sum_{i=1}^{k-1}\widehat{\sigma}_{i,h}^{-2}\boldsymbol{\phi}(s_{h}^{i},a_{h}^{i}){\boldsymbol{\epsilon}_h^i}^\top\left(\boldsymbol{V}-\boldsymbol{V}_{h+1}^*\right)\cdot\mathds{1}\left\{\boldsymbol{V}_{h+1}^*-\zeta\le\boldsymbol{V}\le\widehat{\boldsymbol{V}}_{i,h+1}+\zeta\right\}\right\|_{\widehat{\mathbf{\Lambda}}_{k,h}^{-1}}\\
				\leq& 8 \sqrt{\frac{2}{Hd^2}\log\left(1+\frac{K}{Hd\lambda}\right) \log \left(\frac{4K^{2}H}{\delta}\right)}+\frac{4}{H\sqrt{d^5}}\log\left(\frac{4K^{2}H}{\delta}\right),
			\end{align*}
			where the equality holds since under event $\mathcal{E}_h^{(2)}\cap\widehat{\Psi}_{h+1}\cap\widecheck{\Psi}_{h+1}$, for any $i\in[K]$, $\mathds{1}\{\boldsymbol{\mu}_h\in\widebar{\mathcal{C}}_{i,h}\cap\widecheck{\mathcal{C}}_{i,h}\}\cdot\mathds{1}\{\widehat{\Psi}_{i,h+1}\cap\widecheck{\Psi}_{i,h+1}\}=1$.
			
			Denote the $\varepsilon$-cover of function class $\widehat{\mathcal{V}}$ as $\widehat{\mathcal{N}}_{\varepsilon}$.
			Since $\widehat{V}_{k,h+1}(\cdot)\in\widehat{\mathcal{V}}$, for any $\widehat{V}_{k,h+1}$, there exists a $V' \in \widehat{\mathcal{N}}_{\varepsilon}$, such that $\|\widehat{\boldsymbol{V}}_{k,h+1}-\boldsymbol{V}'\|_{\infty} \leq \varepsilon$.
			This implies $\boldsymbol{V}_{h+1}^*-\varepsilon\le\widehat{\boldsymbol{V}}_{k,h+1}-\varepsilon\le\boldsymbol{V}'\le\widehat{\boldsymbol{V}}_{k,h+1}+\varepsilon\le\widehat{\boldsymbol{V}}_{i,h+1}+\varepsilon$ for any $i\in[k]$,
			where the first inequality holds by Lemma~\ref{lm:optimism} under $\widehat{\Psi}_{h+1}$, and the last inequality holds by definition of optimistic value function in Algorithm~\ref{alg:plus}.
			
			In addition, setting $\varepsilon=\zeta=H\sqrt{\lambda}/K$ makes $\mathds{1}\{\boldsymbol{V}_{h+1}^*-\zeta\le\boldsymbol{V}'\le\widehat{\boldsymbol{V}}_{i,h+1}+\zeta\}=1$ for any $i\in[k]$.
			Moreover, since $\|{\boldsymbol{\epsilon}_h^i}^\top(\widehat{\boldsymbol{V}}_{k,h+1}-\boldsymbol{V}')\|_2\le\|\boldsymbol{\epsilon}_h^i\|_1\|\widehat{\boldsymbol{V}}_{k,h+1}-\boldsymbol{V}'\|_\infty=2\varepsilon$
			and $\|\sum_{i=1}^{k-1}\widehat{\sigma}_{i, h}^{-2}\boldsymbol{\phi}(s_{h}^{i}, a_{h}^{i})\|_{\widehat{\mathbf{\Lambda}}_{k,h}^{-1}}\le K/(H\sqrt{\lambda})$, we have
			\begin{align}\label{eq:c14tmp1}
			    \left\|\sum_{i=1}^{k-1}\widehat{\sigma}_{i, h}^{-2}\boldsymbol{\phi}(s_{h}^{i},a_{h}^{i}){\boldsymbol{\epsilon}_h^i}^\top\left(\widehat{\boldsymbol{V}}_{k,h+1}-\boldsymbol{V}'\right)\cdot\mathds{1}\left\{\boldsymbol{V}_{h+1}^*-\zeta\le\boldsymbol{V}'\le\widehat{\boldsymbol{V}}_{i,h+1}+\zeta\right\}\right\|_{\widehat{\mathbf{\Lambda}}_{k,h}^{-1}} \leq\frac{2\varepsilon K}{H\sqrt{\lambda}}.
			\end{align}
			This further implies that the following inequality holds with probability at least $1-3\delta/H$:
			\begin{align*}
				    &\left\|\sum_{i=1}^{k-1}\widehat{\sigma}_{i, h}^{-2}\boldsymbol{\phi}(s_{h}^{i},a_{h}^{i}){\boldsymbol{\epsilon}_h^i}^\top\left(\widehat{\boldsymbol{V}}_{k,h+1}-\boldsymbol{V}_{h+1}^*\right)\right\|_{\widehat{\mathbf{\Lambda}}_{k,h}^{-1}}\\
					=&\left\|\sum_{i=1}^{k-1}\widehat{\sigma}_{i, h}^{-2}\boldsymbol{\phi}(s_{h}^{i},a_{h}^{i}){\boldsymbol{\epsilon}_h^i}^\top\left(\widehat{\boldsymbol{V}}_{k,h+1}-\boldsymbol{V}_{h+1}^*\right)\cdot\mathds{1}\left\{\boldsymbol{V}_{h+1}^*-\zeta\le\boldsymbol{V}'\le\widehat{\boldsymbol{V}}_{i,h+1}+\zeta\right\}\right\|_{\widehat{\mathbf{\Lambda}}_{k,h}^{-1}}\\
					\le& \left\|\sum_{i=1}^{k-1}\widehat{\sigma}_{i, h}^{-2}\boldsymbol{\phi}(s_{h}^{i},a_{h}^{i}){\boldsymbol{\epsilon}_h^i}^\top\left(\boldsymbol{V}'-\boldsymbol{V}_{h+1}^*\right)\cdot\mathds{1}\left\{\boldsymbol{V}_{h+1}^*-\zeta\le\boldsymbol{V}'\le\widehat{\boldsymbol{V}}_{i,h+1}+\zeta\right\}\right\|_{\widehat{\mathbf{\Lambda}}_{k,h}^{-1}}\\
					&+\left\|\sum_{i=1}^{k-1}\widehat{\sigma}_{i, h}^{-2}\boldsymbol{\phi}(s_{h}^{i},a_{h}^{i}){\boldsymbol{\epsilon}_h^i}^\top\left(\widehat{\boldsymbol{V}}_{k,h+1}-\boldsymbol{V}'\right)\cdot\mathds{1}\left\{\boldsymbol{V}_{h+1}^*-\zeta\le\boldsymbol{V}'\le\widehat{\boldsymbol{V}}_{i,h+1}+\zeta\right\}\right\|_{\widehat{\mathbf{\Lambda}}_{k,h}^{-1}}\\
					\leq & \left\|\sum_{i=1}^{k-1}\widehat{\sigma}_{i, h}^{-2}\boldsymbol{\phi}(s_{h}^{i},a_{h}^{i}){\boldsymbol{\epsilon}_h^i}^\top\left(\boldsymbol{V}'-\boldsymbol{V}_{h+1}^*\right)\cdot\mathds{1}\left\{\boldsymbol{V}_{h+1}^*-\zeta\le\boldsymbol{V}'\le\widehat{\boldsymbol{V}}_{i,h+1}+\zeta\right\}\right\|_{\widehat{\mathbf{\Lambda}}_{k,h}^{-1}}+\frac{2\varepsilon K}{H\sqrt{\lambda}}\\
					\le&8\sqrt{\frac{2}{Hd^2}\log\left(1+\frac{K}{Hd\lambda}\right) \left[\log\left(\frac{4K^{2}H}{\delta}\right)+\log\left|\widehat{\mathcal{N}}_{\varepsilon}\right|\right]}+\frac{4}{H\sqrt{d^5}}\left[\log\left(\frac{4K^{2}H}{\delta}\right)+\log\left|\widehat{\mathcal{N}}_{\varepsilon}\right|\right]+\frac{2\varepsilon K}{H\sqrt{\lambda}}
			\end{align*}
			where the first inequality is due to triangle inequality, the second inequality holds by Eq.~(\ref{eq:c14tmp1}), and the third inequality holds by a union bound over all functions in $\widehat{\mathcal{N}}_{\varepsilon}$ with
			\begin{align*}
			    \log\left|\widehat{\mathcal{N}}_{\varepsilon}\right| \leq dJ\log\left(1+\frac{4 L}{\varepsilon}\right)+d^{2}J \log\left(1+\frac{8\widehat{B}^{2}\sqrt{d}}{\lambda \varepsilon^{2}}\right)
			\end{align*}
			according to Lemma~\ref{lm:coverhatv} with $J=dH\log(1 + K)$ by Lemma~\ref{lm:numupdate}.
			
			Similar to Eq.~(\ref{eq:betatmp}), for any $k\in[K]$ and fixed $h\in[H]$, under $\widehat{\Psi}_{h+1}\cap\widecheck{\Psi}_{h+1}$, we have that, with probability at least $1-3\delta/H$:
			\begin{equation}
				\begin{aligned}
					&\left\|\left(\widehat{\boldsymbol{\mu}}_{k, h}-\boldsymbol{\mu}_h\right)\left(\widehat{\boldsymbol{V}}_{k,h+1}-\boldsymbol{V}_{h+1}^*\right)\right\|_{\widehat{\boldsymbol{\Lambda}}_{k,h}}\\
					\le&\frac{1}{\sqrt{\lambda}}\cdot \lambda H\sqrt{d}+\left\|\sum_{i=1}^{k-1}\widehat{\sigma}_{i, h}^{-2}\boldsymbol{\phi}(s_{h}^{i},a_{h}^{i}){\boldsymbol{\epsilon}_h^i}^\top\left(\widehat{\boldsymbol{V}}_{k,h+1}-\boldsymbol{V}_{h+1}^*\right)\right\|_{\widehat{\mathbf{\Lambda}}_{k,h}^{-1}}\\
					\le&8\sqrt{\frac{2}{Hd^2}\log\left(1+\frac{K}{Hd\lambda}\right) \left[\log\left(\frac{4K^{2}H}{\delta}\right)+dJ \log\left(1+\frac{4KL}{H\sqrt{\lambda}}\right)+d^{2}J \log\left(1+\frac{8K^2\widehat{B}^{2}\sqrt{d}}{H^2\lambda^2}\right)\right]}\\
					&+\frac{4}{H\sqrt{d^5}}\left[\log\left(\frac{4K^{2}H}{\delta}\right)+dJ\log\left(1+\frac{4KL}{H\sqrt{\lambda}}\right)+d^{2}J\log\left(1+\frac{8K^2\widehat{B}^{2}\sqrt{d}}{H^2\lambda^2}\right)\right]+H\sqrt{\lambda d}+2\\
					=&\widehat{\beta}^{(2)},
				\end{aligned}
			\end{equation}
			where the last inequality holds by the above proved self-normalized bound and $\varepsilon = \zeta=H\sqrt{\lambda}/K$.
			Thus, we conclude that for any $k\in[K]$ and fixed $h\in[H]$,  under $\widehat{\Psi}_{h+1}\cap\widecheck{\Psi}_{h+1}$, with probability at least $1-3\delta/H$:
			\begin{align*}
			    \boldsymbol{\mu}_h\in\widehat{\mathcal{C}}^{(2)}_{k,h}\cap\widebar{\mathcal{C}}_{k,h}\cap\widecheck{\mathcal{C}}_{k,h}.
			\end{align*}
		\end{proof}
	\end{lemma}
	
	\subsection{Proof of Lemma~\ref{lm:cs}}\label{sec:appcs}
	Now we are ready to prove Lemma~\ref{lm:cs}, i.e., building the sharp confidence set $\widehat{\mathcal{C}}_{k,h}$, in the main paper, based on above building blocks including confidence sets $\widebar{\mathcal{C}}_{k,h},\widetilde{\mathcal{C}}_{k,h},\widecheck{\mathcal{C}}_{k,h},\widehat{\mathcal{C}}^{(1)}_{k,h},\widehat{\mathcal{C}}^{(2)}_{k,h}$, and Lemma~\ref{lm:var} and \ref{lm:var2} for upper bounding variances of value functions.
	In the following, we present Lemma~\ref{lm:csf}, which is the full version of Lemma~\ref{lm:cs} in the main paper.
	\begin{lemma}\label{lm:csf}
		Set $\widehat{\beta}=\widehat{\beta}^{(1)}+\widehat{\beta}^{(2)}$ with $\widehat{\beta}^{(1)}$ and $\widehat{\beta}^{(2)}$ given in Lemma~\ref{lm:rightarrowbeta} and \ref{lm:leftarrowbeta}, respectively.
		Then for any $\delta\in(0,1)$, with probability at least $1-7\delta$, we have that simultaneously for any $k \in[K]$ and any $h \in[H]$,
		\begin{align*}
		    \boldsymbol{\mu}_{h}\in\widehat{\mathcal{C}}_{k,h}\cap\widehat{\mathcal{C}}^{(1)}_{k,h}\cap\widehat{\mathcal{C}}^{(2)}_{k,h}\cap\widebar{\mathcal{C}}_{k,h}\cap\widetilde{\mathcal{C}}_{k,h}\cap\widecheck{\mathcal{C}}_{k,h}
		\end{align*}
		and
		\begin{align*}
		    \left|[\widehat{\mathbb{V}}_{k, h} \widehat{V}_{k, h+1}](s_{h}^{k}, a_{h}^{k})-[\mathbb{V}_{h} V_{h+1}^*](s_{h}^{k}, a_{h}^{k})\right|\leq &U_{k, h}\\
	        \left[\mathbb{V}_{h}(V-V_{h+1}^*)\right](s_{h}^{k}, a_{h}^{k})\le& E_{k,h}
	    \end{align*}
		where
		\begin{align}
		    U_{k,h}=&\min\Big\{\widetilde{\beta}\left\|\boldsymbol{\phi}(s_{h}^{k}, a_{h}^{k})\right\|_{\widehat{\mathbf{\Lambda}}_{k,h}^{-1}}+4H\Big[\left|\langle\widehat{\boldsymbol{\mu}}_{k,h}(\widehat{\boldsymbol{V}}_{k,h+1}-\widecheck{\boldsymbol{V}}_{k,h+1}),\boldsymbol{\phi}(s_h^k,a_h^k)\rangle\right|+\widebar{\beta}\left\|\boldsymbol{\phi}(s_{h}^{k},a_{h}^{k})\right\|_{\widehat{\mathbf{\Lambda}}_{k,h}^{-1}}\notag\\
		    &+\widecheck{\beta}\left\|\boldsymbol{\phi}(s_{h}^{k},a_{h}^{k})\right\|_{\widehat{\mathbf{\Lambda}}_{k,h}^{-1}}\Big],2H^2\Big\}\label{eq:ukh}\\
		    E_{k,h}=&\min\Big\{H\Big[\langle\widehat{\boldsymbol{\mu}}_{k,h}\widehat{\boldsymbol{V}}_{k,h+1},\boldsymbol{\phi}(s_h^k,a_h^k)\rangle-\langle\widehat{\boldsymbol{\mu}}_{k,h}\widecheck{\boldsymbol{V}}_{k,h+1},\boldsymbol{\phi}(s_h^k,a_h^k)\rangle+\widebar{\beta}\left\|\boldsymbol{\phi}(s_{h}^{k},a_{h}^{k})\right\|_{\widehat{\mathbf{\Lambda}}_{k,h}^{-1}}+\widecheck{\beta}\left\|\boldsymbol{\phi}(s_{h}^{k},a_{h}^{k})\right\|_{\widehat{\mathbf{\Lambda}}_{k,h}^{-1}}\notag\\
		    &+H\sqrt{\lambda}/K\Big],H^2\Big\}\label{eq:ekh}.
		\end{align}
		Here $\widebar{\beta},\widetilde{\beta},\widecheck{\beta}$ are specified in Lemma~\ref{lm:barbeta}, \ref{lm:tildebeta}, \ref{lm:checkbeta}, respectively, and $\zeta=H\sqrt{\lambda}/K$.
	\end{lemma}
	
	\begin{proof}
		We first prove the following claim:
		
		\emph{For any $\delta\in(0,1)$, any $k\in[K]$ and fixed $h\in[H]$, with probability at least $1-7(H-h)\delta/H$, for any $h'$ such that $h\le h'\le H$:
		\begin{align*}
		    \boldsymbol{\mu}_{h'}\in\widehat{\mathcal{C}}_{k,h'}\cap\widehat{\mathcal{C}}^{(1)}_{k,h'}\cap\widehat{\mathcal{C}}^{(2)}_{k,h'}\cap\widebar{\mathcal{C}}_{k,h'}\cap\widetilde{\mathcal{C}}_{k,h'}\cap\widecheck{\mathcal{C}}_{k,h'},
		\end{align*}
		and
		\begin{align*}
		    \left|[\widehat{\mathbb{V}}_{k, h} \widehat{V}_{k, h+1}](s_{h}^{k}, a_{h}^{k})-[\mathbb{V}_{h} V_{h+1}^*](s_{h}^{k}, a_{h}^{k})\right|\leq &U_{k, h}\\
	        \left[\mathbb{V}_{h}(V-V_{h+1}^*)\right](s_{h}^{k}, a_{h}^{k})\le&E_{k,h}
	    \end{align*}
		hold simultaneously.}
		
		We prove this claim by introduction.
		\begin{itemize}
			\item We first prove the claim for $h=H$.
			Since $\widecheck{V}_{k,H+1}(\cdot)=V_{H+1}^*(\cdot)=\widehat{V}_{k,H+1}(\cdot)=0$ in Algorithm~\ref{alg:plus} for any $k\in[K]$, the conclusion holds for sure.
			
			\item Assume the claim holds for $h+1\leq H$. Then, for any $k\in[K]$ and fixed $h+1\in[H]$, with probability at least $1-7(H-h-1)\delta/H$, $\boldsymbol{\mu}_{h'}\in\widehat{\mathcal{C}}_{k,h'}\cap\widecheck{\mathcal{C}}_{k,h'}$ for any $h+1\le h'\le H$, which implies $\widecheck{\Psi}_{h+1}\cap\widehat{\Psi}_{h+1}$ holds.
			
			Combined with conclusions from Lemma~\ref{lm:rightarrowbeta} and Lemma~\ref{lm:leftarrowbeta}, for any $k\in[K]$, the following events holds with probability at least $1-7\delta/H$:			
			\begin{align*}
		        \boldsymbol{\mu}_{h}\in\widehat{\mathcal{C}}^{(1)}_{k,h}\cap\widehat{\mathcal{C}}^{(2)}_{k,h}\cap\widebar{\mathcal{C}}_{k,h}\cap\widetilde{\mathcal{C}}_{k,h}\cap\widecheck{\mathcal{C}}_{k,h}.
		    \end{align*}
		    Moreover, we have
		    \begin{align*}
    		    \left|[\widehat{\mathbb{V}}_{k, h} \widehat{V}_{k, h+1}](s_{h}^{k}, a_{h}^{k})-[\mathbb{V}_{h} V_{h+1}^*](s_{h}^{k}, a_{h}^{k})\right|\leq &U_{k, h}\\
    	        \left[\mathbb{V}_{h}(V-V_{h+1}^*)\right](s_{h}^{k}, a_{h}^{k})\le& E_{k,h}
    	    \end{align*}
			under the event $\widehat{\Psi}_{h+1}\cap\widecheck{\Psi}_{h+1}$ and $\widebar{\mathcal{C}}_{k,h}\cap\widetilde{\mathcal{C}}_{k,h}\cap\widecheck{\mathcal{C}}_{k,h}$ by Lemma~\ref{lm:var} and \ref{lm:var2} with $\zeta=H\sqrt{\lambda}/K$.
			
			Considering $\boldsymbol{\mu}_h\in\widehat{\mathcal{C}}^{(1)}_{k,h}\cap\widehat{\mathcal{C}}^{(2)}_{k, h}$, we have
			\begin{align*}
				\left\|\left(\boldsymbol{\mu}_h-\widehat{\boldsymbol{\mu}}_{k, h}\right)\widehat{\boldsymbol{V}}_{k,h+1}\right\|_{\widehat{\boldsymbol{\Lambda}}_{k, h}}\le&\left\|\left(\boldsymbol{\mu}_h-\widehat{\boldsymbol{\mu}}_{k, h}\right)\boldsymbol{V}_{h+1}^*\right\|_{\widehat{\boldsymbol{\Lambda}}_{k, h}}+\left\|\left(\boldsymbol{\mu}_h-\widehat{\boldsymbol{\mu}}_{k, h}\right)\left(\widehat{\boldsymbol{V}}_{k,h+1}-{\boldsymbol{V}}_{h+1}^*\right)\right\|_{\widehat{\boldsymbol{\Lambda}}_{k, h}}\\
				\leq&\widehat{\beta}^{(1)}+\widehat{\beta}^{(2)}=\widehat{\beta},
			\end{align*}
			which implies $\boldsymbol{\mu}_{h}\in\widehat{\mathcal{C}}_{k,h}$ since $\widehat{\mathcal{C}}_{k,h}=\{\boldsymbol{\mu}:\|(\boldsymbol{\mu}-\widehat{\boldsymbol{\mu}}_{k,h})\widehat{\boldsymbol{V}}_{k,h+1}\|_{\widehat{\boldsymbol{\Lambda}}_{k,h}}\leq\widehat{\beta}\}$.
			In other words, $\boldsymbol{\mu}_{h}\in\widehat{\mathcal{C}}_{k,h}\cap\widehat{\mathcal{C}}^{(1)}_{k,h}\cap\widehat{\mathcal{C}}^{(2)}_{k,h}\cap\widebar{\mathcal{C}}_{k,h}\cap\widetilde{\mathcal{C}}_{k,h}\cap\widecheck{\mathcal{C}}_{k,h}$.
			
			Thus, by taking a union bound over these two events, we claim that with probability at least $1-7(H-h)\delta/H$, the claim holds for $h$.
		\end{itemize}
		Therefore, the claim is proved by induction and setting $h=1$ gives the desired results in Lemma~\ref{lm:csf}.
	\end{proof}
	
	\section{Regret Upper Bound}\label{sec:appregret}
	In this section, we  upper bound the final regret, where we  show that the total regret is roughly bounded by the summation of the exploration bonus, i.e., 
	
	\begin{align*}
	    \operatorname{Regret}(K)\le&\sum_{k=1}^{K}\left[\widehat{V}_{k,1}(s_1^k)-V_1^{\pi^k}(s_1^k)\right]\lesssim\sum_{k=1}^{K}\sum_{h=1}^{H}\widehat{\beta}\|\boldsymbol{\phi}(s_h^k,a_h^k)\|_{\widehat{\mathbf{\Lambda}}_{k,h}^{-1}}\\
		\lesssim&\widehat{\beta}\underbrace{\sqrt{\sum_{k=1}^{K}\sum_{h=1}^{H}\widehat{\sigma}_{k,h}^{2}}}_{\widetilde{O}(\sqrt{HT})}\underbrace{\sqrt{\sum_{k=1}^{K}\sum_{h=1}^{H}\|\widehat{\sigma}_{k,h}^{-1}\boldsymbol{\phi}(s_h^k,a_h^k)\|_{\widehat{\mathbf{\Lambda}}_{k,h}^{-1}}^2}}_{\widetilde{O}(\sqrt{Hd})}\le\widetilde{O}(Hd\sqrt{T}).
	\end{align*}

	\begin{itemize}
	    \item The first inequality holds by the optimism of constructed value functions, which is built in Lemma~\ref{lm:optimism}.
	
	    \item The third inequality is proved in Lemma~\ref{lm:sumregret}, which bounds the cumulative difference between the optimistic value function $\widehat{V}_{k,h}$ and the value function associated with policy $\pi^k$ value function $V_{h}^{\pi^k}$ in Lemma~\ref{lm:sumregret}.
	    We further bound the cumulative difference between the optimistic value function $\widehat{V}_{k,h}$ and the pessimistic value function $\widecheck{V}_{k,h}$ in Lemma~\ref{lm:gapop}. 
	
	    \item The fourth inequality holds by Cauchy-Schwarz inequality, where first summation of estimated variance is bounded in Lemma~\ref{lm:sigma} in Appendix~\ref{sec:appsigma}.
	    The summation of estimated variance $\sum_{k=1}^K\sum_{h=1}^H\widehat{\sigma}_{k,h}^2=\widetilde{O}(HT)$, utilizing the the Law of Total Variance in \cite{lattimore2012pac}, detailed in Lemma~\ref{lm:tvl}.
	    The second summation can be bounded by classical Elliptical Potential Lemma, presented in Lemma~\ref{lm:ablog} in Appendix~\ref{sec:appaux}.
	\end{itemize}
	
	Putting these building blocks together, we are finally ready to upper bound the regret in Appendix~\ref{sec:pfthregkr} as $\widetilde{O}(Hd\sqrt{T}+H^4d^4+H^3d^6)$.
	Before the formal proof begins, denote the event when the conclusion of Lemma~\ref{lm:cs} holds as $\Upsilon$.
	Also denote the event that the conclusion of Lemma~\ref{lm:hoef1}, \ref{lm:hoef2} and \ref{lm:tvl} holds as $\Xi_1$, $\Xi_2$ and $\Xi_3$, respectively.
	The final regret bound is a high probability bound builds under event $\Upsilon\cap\Xi_1\cap\Xi_2\cap\Xi_3$.
	
	\subsection{Monotonicity}\label{sec:appopt}
	In this subsection, we build the optimism of the constructed optimistic value function $\widehat{V}_{k,h}$, and the pessimism of the constructed pessimistic value function $\widecheck{V}_{k,h}$ over the optimal value function $V_h^*$ in Lemma~\ref{lm:optimism}.
	These are the preliminaries for our later proofs.
	
	\begin{lemma}[Optimism and Pessimism]\label{lm:optimism}
		In Algorithm~\ref{alg:plus}, if $\widecheck{\Psi}_{k,h}\cap\widehat{\Psi}_{k,h}$ holds, then for any $k\in[K]$ and any $h\in[H]$, we have
		\begin{equation}
			\widecheck{V}_{k, h}(s)\stackrel{(a)}{\le}V_{h}^{*}(s)\stackrel{(b)}{\le}\widehat{V}_{k, h}(s),\quad\forall s\in\mathcal{S}.
		\end{equation}
		
		\begin{proof}
			We prove two inequalities by induction on respective hypotheses.
			
			\textbf{(a) Pessimism:}
			For any $k\in[K]$, the statement holds for $h=H+1$ since $\widecheck{V}_{k, H+1}(\cdot)=V_{H+1}^{*}(\cdot)=0$. 
			
			Assume the statement holds for $h+1$, which means $\widecheck{V}_{k, h+1}(\cdot)\le V_{h+1}^{*}(\cdot)$ under $\widecheck{\Psi}_{k,h+1}$.
			Since $\widecheck{Q}_{k,h}(\cdot,\cdot)=r_h(\cdot,\cdot)+\langle\widehat{\boldsymbol{\mu}}_{k,h}\widecheck{\boldsymbol{V}}_{k,h+1},\boldsymbol{\phi}(\cdot,\cdot)\rangle-\widecheck{\beta}\|\boldsymbol{\phi}(\cdot,\cdot)\|_{ \widehat{\mathbf{\Lambda}}_{k,h}^{-1}}$, for $\forall (s,a)\in\mathcal{S}\times\mathcal{A}$, we have:
			
			\begin{align*}
				&Q_{h}^{*}(s,a)-\widecheck{Q}_{k,h}(s,a)\\
				=&r_h(s,a)+\mathbb{P}_hV_{h+1}^*(s,a)-\left[r_h(s,a)+\langle\widehat{\boldsymbol{\mu}}_{k,h}\widecheck{\boldsymbol{V}}_{k,h+1},\boldsymbol{\phi}(s,a)\rangle-\widecheck{\beta}\|\boldsymbol{\phi}(s,a)\|_{\widehat{\mathbf{\Lambda}}_{k,h}^{-1}}\right]\\
				=&\langle\boldsymbol{\mu}_{h}\widecheck{\boldsymbol{V}}_{k,h+1},\boldsymbol{\phi}(s,a)\rangle-\langle\widehat{\boldsymbol{\mu}}_{k,h}\widecheck{\boldsymbol{V}}_{k,h+1},\boldsymbol{\phi}(s,a)\rangle+\widecheck{\beta}\|\boldsymbol{\phi}(s,a)\|_{\widehat{\mathbf{\Lambda}}_{k,h}^{-1}}+\mathbb{P}_hV_{h+1}^*(s,a)-\mathbb{P}_h\widecheck{V}_{k,h+1}(s,a)\\
				\ge&-\|(\widehat{\boldsymbol{\mu}}_{k, h}-\boldsymbol{\mu}_h)\widecheck{\boldsymbol{V}}_{k,h+1}\|_{\widehat{\boldsymbol{\Lambda}}_{k, h}}\|\boldsymbol{\phi}(s,a)\|_{\widehat{\mathbf{\Lambda}}_{k,h}^{-1}}+\widecheck{\beta}\|\boldsymbol{\phi}(s,a)\|_{\widehat{\mathbf{\Lambda}}_{k,h}^{-1}}+\mathbb{P}_hV_{h+1}^*(s,a)-\mathbb{P}_h\widecheck{V}_{k,h+1}(s,a)\\
				\ge&\mathbb{P}_hV_{h+1}^*(s,a)-\mathbb{P}_h\widecheck{V}_{k,h+1}(s,a)\\
				\ge&0,
			\end{align*}
			where the first inequality holds due to Cauchy-Schwarz inequality, the second inequality holds by the assumption that $\boldsymbol{\mu}_h\in\widecheck{\mathcal{C}}_{k,h}$ under $\widecheck{\Psi}_{k,h}$, the third inequality holds since the induction assumption $\widecheck{V}_{k,h+1}(\cdot)\le V_{h+1}^*(\cdot)$ under $\widecheck{\Psi}_{k,h+1}$ and $\mathbb{P}_h$ is a valid distribution.
			Therefore, we have $\widecheck{Q}_{k,h}(s,a)\le Q_{h}^{*}(s,a)$, for all $(s,a)\in\mathcal{S}\times\mathcal{A}$.
			Since $\widecheck{V}_{k,h}(\cdot)=\max\big\{\max_{a\in\mathcal{A}}\widecheck{Q}_{k,h}(\cdot,a),0\big\}$, for any $s\in\mathcal{S}$, we have the following two cases:
			\begin{itemize}
				\item If $\max_{a\in\mathcal{A}}\widecheck{Q}_{k,h}(s,a)\le0$, we have $0=\widecheck{V}_{k,h}(s)\le V_h^*(s)$.
				\item Otherwise, $\widecheck{V}_{k,h}(s)=\max_{a\in\mathcal{A}}\widecheck{Q}_{k,h}(s,a)\le\max_{a\in\mathcal{A}}Q_{h}^*(s,a)=V_h^*(s)$.
			\end{itemize}
			Therefore, we have $\widecheck{V}_{k, h}(\cdot)\le V_{h}^{*}(\cdot)$ under $\widecheck{\Psi}_{k,h}$ for any $(k,h)\in[K]\times[H]$.

			\textbf{(b): Optimism:}
			We first prove the optimism for some fixed episode $k\in[K]$ by induction.
			
			For any $k\in[K]$, the statement holds for $h=H+1$ since $\widehat{V}_{k, H+1}(\cdot)=V_{H+1}^{*}(\cdot)=0$.
			Assume the statement holds for $h+1$, which means $\widehat{V}_{k, h+1}(\cdot) \geq V_{h+1}^{*}(\cdot)$ under $\widehat{\Psi}_{k,h+1}$ for any $k\in[K]$.
   For any $k\in[K]$ and any $(s,a)\in\mathcal{S}\times\mathcal{A}$ and, we have
			
			\begin{align*}
				&r_h(\cdot,\cdot)+\langle\widehat{\boldsymbol{\mu}}_{k,h}\widehat{\boldsymbol{V}}_{k,h+1},\boldsymbol{\phi}(\cdot,\cdot)\rangle+\widehat{\beta}\|\boldsymbol{\phi}(\cdot,\cdot)\|_{ \widehat{\mathbf{\Lambda}}_{k,h}^{-1}}-Q_{h}^{*}(s,a)\\
				=&r_h(s,a)+\langle\widehat{\boldsymbol{\mu}}_{k,h}\widehat{\boldsymbol{V}}_{k,h+1},\boldsymbol{\phi}(s,a)\rangle+\widehat{\beta}\|\boldsymbol{\phi}(s,a)\|_{\widehat{\mathbf{\Lambda}}_{k,h}^{-1}}-\left[r_h(s,a)+\mathbb{P}_hV_{h+1}^*(s,a)\right]\\
				=&\langle\widehat{\boldsymbol{\mu}}_{k,h}\widehat{\boldsymbol{V}}_{k,h+1},\boldsymbol{\phi}(s,a)\rangle-\langle\boldsymbol{\mu}_{h}\widehat{\boldsymbol{V}}_{k,h+1},\boldsymbol{\phi}(s,a)\rangle+\widehat{\beta}\|\boldsymbol{\phi}(s,a)\|_{\widehat{\mathbf{\Lambda}}_{k,h}^{-1}}+\mathbb{P}_h\widehat{V}_{k,h+1}(s,a)-\mathbb{P}_hV_{h+1}^*(s,a)\\
				\ge&-\|(\widehat{\boldsymbol{\mu}}_{k, h}-\boldsymbol{\mu}_h)\widehat{\boldsymbol{V}}_{k,h+1}\|_{\widehat{\boldsymbol{\Lambda}}_{k, h}}\|\boldsymbol{\phi}(s,a)\|_{\widehat{\mathbf{\Lambda}}_{k,h}^{-1}}+\widehat{\beta}\|\boldsymbol{\phi}(s,a)\|_{\widehat{\mathbf{\Lambda}}_{k,h}^{-1}}+\mathbb{P}_h\widehat{V}_{k,h+1}(s,a)-\mathbb{P}_hV_{h+1}^*(s,a)\\
				\ge&\mathbb{P}_h\widehat{V}_{k,h+1}(s,a)-\mathbb{P}_hV_{h+1}^*(s,a)\\
				\ge&0,
			\end{align*}
			where the first inequality holds due to Cauchy-Schwarz inequality, the second inequality holds by the assumption that $\boldsymbol{\mu}_h\in\widehat{\mathcal{C}}_{k,h}$ under $\widehat{\Psi}_{k,h}$, the third inequality holds by the induction assumption $\widehat{V}_{k,h+1}(\cdot)\ge V_{h+1}^*(\cdot)$ under $\widehat{\Psi}_{k,h+1}$ and $\mathbb{P}_h$ is a valid distribution.
			Then we have
                \begin{align*}
                    Q_{h}^{*}(s,a)\le\min_{1\le i\le k}\min\left\{r_h(s,a)+\langle\widehat{\boldsymbol{\mu}}_{i,h}\widehat{\boldsymbol{V}}_{i,h+1},\boldsymbol{\phi}(s,a)\rangle+\widehat{\beta}\|\boldsymbol{\phi}(s,a)\|_{ \widehat{\mathbf{\Lambda}}_{i,h}^{-1}},H\right\}\le \widehat{Q}_{k,h}(s,a)
                \end{align*}
                for all $(s,a)\in\mathcal{S}\times\mathcal{A}$, which further implies $\widehat{V}_{k,h}(s)=\max_{a\in\mathcal{A}}\widehat{Q}_{k,h}\ge \max_{a\in\mathcal{A}}Q_h^*(s,a)$ for any $s\in\mathcal{S}$.

			Therefore, we have $\widehat{V}_{k, h}(\cdot)\ge V_{h}^{*}(\cdot)$ under $\widehat{\Psi}_{k,h}$ for any $(k,h)\in[K]\times[H]$.
		\end{proof}        
	\end{lemma}

	\subsection{Suboptimality Gap}\label{sec:appsub}
	
	In this subsection, we establish
	Lemma~\ref{lm:sumregret} and Lemma~\ref{lm:gapop} that bound the distance of the optimistic value function $\widehat{V}_{k,h}$ to the value function $V_{h}^{\pi^k}$ associated with policy $\pi_k$ and the pessimistic value function $\widecheck{V}_{k,h}$, respectively.
	Before that, we present two high probability events $\Xi_1$ and $\Xi_2$ in Lemma~\ref{lm:hoef1} and \ref{lm:hoef2}, respectively.
	
	\begin{lemma}\label{lm:hoef1}
	    In Algorithm~\ref{alg:plus}, for any $\delta\in(0,1)$, with probability at least $1-\delta$, simultaneously for all $h'\in[H]$, we have
	    \begin{align*}
	        \sum_{k=1}^{K} \sum_{h=h^{\prime}}^{H}\left[[\mathbb{P}_{h}(\widehat{V}_{k, h+1}-V_{h+1}^{\pi^{k}})](s_{h}^{k}, a_{h}^{k})-[\widehat{V}_{k, h+1}-V_{h+1}^{\pi^{k}}](s_{h+1}^{k})\right]\leq 2 H \sqrt{2 T \log\left(\frac{H}{\delta}\right)}
	    \end{align*}
	    \begin{proof}
	        Denote $\Delta_{k,h} = [\mathbb{P}_{h}(\widehat{V}_{k, h+1}-V_{h+1}^{\pi^{k}})](s_{h}^{k}, a_{h}^{k})-[\widehat{V}_{k, h+1}-V_{h+1}^{\pi^{k}}](s_{h+1}^{k})$.
			Since $s_{h+1}^k$ is $\mathcal{F}_{k,h+1}$-measurable, $\Delta_{k,h}$ is $\mathcal{F}_{k,{h+1}}$-measurable and $\mathbb{E}[ \Delta_{k,h} \mid \mathcal{F}_{k,h}] = 0$.
			Thus, for some $h'\in[H]$, $\{\Delta_{k,h'},\Delta_{k,h'+1},...,\Delta_{k,H}\}_{k\in[K]}$ is a martingale difference sequence.
			Since $|\Delta_{k,h}| \leq 2H$ by $-H\le\widehat{V}_{k, h+1}(\cdot)-V_{h+1}^{\pi^{k}}(\cdot)\le H$, we can apply Azuma-Hoeffding inequality (Lemma~\ref{lm:hoeffding}) to this martingale difference sequence and obtain
			\begin{equation}\label{eq:sumDelta}
				\sum_{k=1}^K\sum_{h=h'}^H\Delta_{k,h}\le2H\sqrt{2K(H-h'+1))\log(1/\delta)}\le 2H\sqrt{2T\log(1/\delta)},
			\end{equation}
			for some $h'\in[H]$ with probability at least $1 - \delta$.
			Taking a union bound over all $h'\in[H]$ gives the final conclusion.
	    \end{proof}
	\end{lemma}
	
	\begin{lemma}\label{lm:hoef2}
	    In Algorithm~\ref{alg:plus}, for any $\delta\in(0,1)$, with probability at least $1-\delta$, simultaneously for all $h'\in[H]$, we have
	    \begin{align*}
	        \sum_{k=1}^{K} \sum_{h=h^{\prime}}^{H}\left[[\mathbb{P}_{h}(\widehat{V}_{k, h+1}-\widecheck{V}_{k,h+1})](s_{h}^{k}, a_{h}^{k})-[\widehat{V}_{k, h+1}-\widecheck{V}_{k,h+1}](s_{h+1}^{k})\right]\leq 2 H \sqrt{2 T \log\left(\frac{H}{\delta}\right)}
	    \end{align*}
	    \begin{proof}
	        The proof is almost the same as that of Lemma~\ref{lm:hoef1}, except for replacing $V_{k,h+1}^{\pi^k}(\cdot)$ by  $\widecheck{V}_{k,h+1}(\cdot)$.
	    \end{proof}
	\end{lemma}
	
	\begin{lemma}\label{lm:sumregret}
		In Algorithm~\ref{alg:plus}, under $\widehat{\Psi}_{1}\cap \Xi_1$, we have
        \begin{align}
            \sum_{k=1}^K\widehat{V}_{k,1}(s_1^k)-V_{1}^{\pi^k}(s_1^k)\le&4\widehat{\beta}\sqrt{\sum_{k=1}^K\sum_{h=1}^H\widehat{\sigma}_{k,h}^2}\sqrt{H\cdot2d\log\left(1+\frac{K}{Hd\lambda}\right)}\notag\\
            &+\frac{3H^2d}{\log (2)} \log \left(1+\frac{1}{\lambda H\log (2)}\right)+2H\sqrt{2T\log\left(\frac{H}{\delta}\right)}\label{eq:sumregret1}\\
            \sum_{k=1}^K\sum_{h=1}^H[\mathbb{P}_{h}(\widehat{V}_{k,h+1}-V_{h+1}^{\pi^k})](s_h^k,a_h^k)\le&4\widehat{\beta}H\sqrt{\sum_{k=1}^K\sum_{h=1}^H\widehat{\sigma}_{k,h}^2}\sqrt{H\cdot2d\log\left(1+\frac{K}{Hd\lambda}\right)}\notag\\
            &+\frac{3H^3d}{\log(2)}\log\left(1+\frac{1}{\lambda\log(2)}\right)+4H^2\sqrt{2T\log\left(\frac{H}{\delta}\right)}\label{eq:sumregret2}
        \end{align}
		\begin{proof}
			By Algorithm~\ref{alg:plus}, for any $k\in[K],h\in[H]$, we have
			\begin{equation}\label{eq:exp}
				\begin{aligned}
					\widehat{Q}_{k,1}(s,a)\le&r_1(s,a)+\langle\widehat{\boldsymbol{\mu}}_{k_0,1}\widehat{\boldsymbol{V}}_{k_0,2},\boldsymbol{\phi}(s,a)\rangle+\widehat{\beta}\|\boldsymbol{\phi}(s,a)\|_{ \widehat{\mathbf{\Lambda}}_{k_0,1}^{-1}}\\
					=&r_1(s,a)+\widehat{\mathbb{P}}_{k_0,1}\widehat{V}_{k_0,2}(s,a)+\widehat{\beta}\|\boldsymbol{\phi}(s,a)\|_{ \widehat{\mathbf{\Lambda}}_{k_0,1}^{-1}},\\
					Q_{1}^{\pi^k}(s,a)=&r_1(s,a)+\mathbb{P}_1V_{2}^{\pi^k}(s,a).
				\end{aligned}
			\end{equation}
			Then,
			\begin{align}\label{eq:regretsumtmp1}
				\begin{split}
				&\widehat{V}_{k,1}(s_1^k)-V_{h}^{\pi^k}(s_1^k)=\widehat{Q}_{k,1}(s_1^k,a_1^k)-Q_1^{\pi^k}(s_1^k,a_1^k)\\
					\le&r_1(s_1^k,a_1^k)+\widehat{\mathbb{P}}_{k_0,1}\widehat{V}_{k_0,2}(s_1^k,a_1^k)+\widehat{\beta}\|\boldsymbol{\phi}(s_1^k,a_1^k)\|_{\widehat{\mathbf{\Lambda}}_{k_0,1}^{-1}}-r_1(s_1^k,a_1^k)-\mathbb{P}_{1}V_{2}^{\pi^k}(s_1^k,a_1^k)\\
					=&\widehat{\beta}\|\boldsymbol{\phi}(s_1^k,a_1^k)\|_{\widehat{\mathbf{\Lambda}}_{k_0,1}^{-1}}+[\widehat{\mathbb{P}}_{k_0,1}\widehat{V}_{k_0,2}(s_1^k,a_1^k)-\mathbb{P}_{1}\widehat{V}_{k_0,2}(s_1^k,a_1^k)]+[\mathbb{P}_{1}(\widehat{V}_{k,2}-V_{2}^{\pi^k})](s_1^k,a_1^k)\\
					=&\widehat{\beta}\|\boldsymbol{\phi}(s_1^k,a_1^k)\|_{\widehat{\mathbf{\Lambda}}_{k_0,1}^{-1}}+\langle(\widehat{\boldsymbol{\mu}}_{k_0,1}-\boldsymbol{\mu}_{1})\widehat{\boldsymbol{V}}_{k_0,2},\boldsymbol{\phi}(s_1^k,a_1^k)\rangle+[\mathbb{P}_{1}(\widehat{V}_{k,2}-V_{2}^{\pi^k})](s_1^k,a_1^k)\\
					\le&\widehat{\beta}\|\boldsymbol{\phi}(s_1^k,a_1^k)\|_{\widehat{\mathbf{\Lambda}}_{k_0,1}^{-1}}+\|(\widehat{\boldsymbol{\mu}}_{k_0,1}-\boldsymbol{\mu}_h)\widehat{\boldsymbol{V}}_{k_0,2}\|_{\widehat{\boldsymbol{\Lambda}}_{k_0,1}}\|\boldsymbol{\phi}(s_{1}^{k},a_{1}^{k})\|_{\widehat{\mathbf{\Lambda}}_{k_0,1}^{-1}}+[\mathbb{P}_{1}(\widehat{V}_{k,2}-V_{2}^{\pi^k})](s_1^k,a_1^k)\\
					\le&2\widehat{\beta}\|\boldsymbol{\phi}(s_1^k,a_1^k)\|_{\widehat{\mathbf{\Lambda}}_{k_0,1}^{-1}}+[\mathbb{P}_{1}(\widehat{V}_{k,2}-V_{2}^{\pi^k})](s_1^k,a_1^k)\\
					\le&4\widehat{\beta}\|\boldsymbol{\phi}(s_1^k,a_1^k)\|_{\widehat{\mathbf{\Lambda}}_{k,1}^{-1}}+[\mathbb{P}_{1}(\widehat{V}_{k,2}-V_{2}^{\pi^k})](s_1^k,a_1^k)\\					=&4\widehat{\beta}\|\boldsymbol{\phi}(s_1^k,a_1^k)\|_{\widehat{\mathbf{\Lambda}}_{k,1}^{-1}}+[\widehat{V}_{k,2}-V_{2}^{\pi^k}](s_{2}^k)+[\mathbb{P}_{1}(\widehat{V}_{k,2}-V_{2}^{\pi^k})](s_1^k,a_1^k)-[\widehat{V}_{k,2}-V_{2}^{\pi^k}](s_{2}^k)\\
					&\cdots\\
					\le&\sum_{h=1}^H4\widehat{\beta}\|\boldsymbol{\phi}(s_h^k,a_h^k)\|_{\widehat{\mathbf{\Lambda}}_{k,h}^{-1}}+\sum_{h=1}^H\left[[\mathbb{P}_{h}(\widehat{V}_{k,h+1}-V_{h+1}^{\pi^k})](s_h^k,a_h^k)-[\widehat{V}_{k,h+1}-V_{h+1}^{\pi^k}](s_{h+1}^k)\right]
				\end{split}
			\end{align}
			where the first inequality is due to Eq.~(\ref{eq:exp}),
            the first equality holds since $\widehat{V}_{k,2}=\widehat{V}_{k_0,2}$ in Algorithm~\ref{alg:plus}, the second inequality holds by the Cauchy-Schwarz inequality, third inequality holds since $\boldsymbol{\mu}_1\in\widehat{\mathcal{C}}_{k_0,1}$ under $\widehat{\Psi}_{1}$,
			the fourth inequality holds due to Lemma~\ref{lm:rsn} with the updating
			rule in Line 8 of Algorithm~\ref{alg:plus},
			and the last inequality holds since we can expand $\widehat{V}_{k,h+1}(s_h^k)-\widecheck{V}_{k,h+1}(s_h^k)$ in a recursive way until stage $H$.
			
			Summing up Eq.(\ref{eq:regretsumtmp1}) for $k\in[K]$ gives
			\begin{align*}
		            &\sum_{k=1}^K\widehat{V}_{k,1}(s_1^k)-V_{1}^{\pi^k}(s_1^k)\\
		            \le&\sum_{k=1}^K\sum_{h=1}^H4\widehat{\beta}\|\boldsymbol{\phi}(s_h^k,a_h^k)\|_{\widehat{\mathbf{\Lambda}}_{k,1}^{-1}}+\sum_{k=1}^K\sum_{h=1}^H\left[[\mathbb{P}_{h}(\widehat{V}_{k,h+1}-V_{h+1}^{\pi^k})](s_h^k,a_h^k)-[\widehat{V}_{k,h+1}-V_{h+1}^{\pi^k}](s_{h+1}^k)\right]\\
			        \le&\underbrace{4\widehat{\beta}\sum_{k=1}^K\sum_{h=1}^H\widehat{\sigma}_{k,h}\|\widehat{\sigma}_{k,h}^{-1}\boldsymbol{\phi}(s_h^k,a_h^k)\|_{\widehat{\mathbf{\Lambda}}_{k,h}^{-1}}}_{I}+2H\sqrt{2T\log\left(\frac{H}{\delta}\right)},
		    \end{align*}
		    where the second inequality holds under $\Xi_1$ by Lemma~\ref{lm:hoef1}.
    		Subsequently, we try to bound $I$.
    		For fixed $h\in[H]$, denote $\boldsymbol{x}_k=\widehat{\sigma}_{k,h}^{-1}\boldsymbol{\phi}(s_h^k, a_h^k)$ in Lemma~\ref{lm:od}. Then, there are at most $3d\log[1+d/(\lambda H\log (2))]/\log (2)$ episodes that $\|\widehat{\sigma}_{k,h}^{-1}\boldsymbol{\phi}(s_h^k, a_h^k)\|_{\widehat{\mathbf{\Lambda}}_{k,h}^{-1}}\ge1$, which further implies that there are at most $H\cdot 3d\log[1+d/(\lambda H\log (2))]/\log (2)$ episodes that there exists $h'\in[H]$ such that $\|\widehat{\sigma}_{k,h'}^{-1}\boldsymbol{\phi}(s_{h'}^k, a_{
    		h'}^k)\|_{\widehat{\mathbf{\Lambda}}_{k,h'}^{-1}}\ge1$.
    		Moreover, we can bound $\widehat{V}_{k,1}(s_1^k)-V_{1}^{\pi^k}(s_1^k)$ in these episodes by $H$ since $\widehat{V}_{k,1}(\cdot)-V_{1}^{\pi^k}(\cdot)\le H$ for any $k\in[K]$.
    		Thus, we have
    		
    		\begin{align}\label{eq:sum2}
    			\begin{split}
    			    &\sum_{k=1}^K\widehat{V}_{k,1}(s_1^k)-V_{1}^{\pi^k}(s_1^k)\\
    			    \le&4\widehat{\beta}\sum_{k=1}^K\sum_{h=1}^H\widehat{\sigma}_{k,h}\min\left\{\|\widehat{\sigma}_{k,h}^{-1}\boldsymbol{\phi}(s_h^k,a_h^k)\|_{\widehat{\mathbf{\Lambda}}_{k,h}^{-1}},1\right\}+H\cdot H\cdot \frac{3 d}{\log (2)} \log \left(1+\frac{1}{\lambda H\log (2)}\right)+2H\sqrt{2T\log\left(\frac{H}{\delta}\right)}\\
    			    \le&4\widehat{\beta}\sqrt{\sum_{k=1}^K\sum_{h=1}^H\widehat{\sigma}_{k,h}^2}\sqrt{\sum_{k=1}^K\sum_{h=1}^H\min\left\{\|\widehat{\sigma}_{k,h}^{-1}\boldsymbol{\phi}(s_h^k,a_h^k)\|_{\widehat{\mathbf{\Lambda}}_{k,h}^{-1}}^2,1\right\}}+\frac{3H^2d}{\log (2)} \log \left(1+\frac{1}{\lambda H\log (2)}\right)+2H\sqrt{2T\log\left(\frac{H}{\delta}\right)}\\
    			    \le&4\widehat{\beta}\sqrt{\sum_{k=1}^K\sum_{h=1}^H\widehat{\sigma}_{k,h}^2}\sqrt{H\cdot2d\log\left(1+\frac{K}{Hd\lambda}\right)}+\frac{3H^2d}{\log (2)} \log \left(1+\frac{L^{2}}{\lambda H\log (2)}\right)+2H\sqrt{2T\log\left(\frac{H}{\delta}\right)},
    			\end{split}
    		\end{align}
    		where the second inequality holds by Cauchy-Schwarz inequality, and the last inequality hols by Lemma~\ref{lm:ablog} with the fact that $\|\widehat{\sigma}_{k,h}^{-1}\boldsymbol{\phi}(s_h^k,a_h^k)\|_2\le1/\sqrt{H}$.
    		Thus, Eq.~(\ref{eq:sumregret1}) is obtained.
    		Besides, by similar argument in Eq.~(\ref{eq:regretsumtmp1}), we obtain
    		\begin{align*}
    		    &\widehat{V}_{k,h'}(s_{h'}^k)-V_{h'}^{\pi^k}(s_{h'}^k)\le&\sum_{h=h'}^H4\widehat{\beta}\|\boldsymbol{\phi}(s_h^k,a_h^k)\|_{\widehat{\mathbf{\Lambda}}_{k,h}^{-1}}+\sum_{h=h'}^H\left[[\mathbb{P}_{h}(\widehat{V}_{k,h+1}-V_{h+1}^{\pi^k})](s_h^k,a_h^k)-[\widehat{V}_{k,h+1}-V_{h+1}^{\pi^k}](s_{h+1}^k)\right]
    		\end{align*}
    		which further gives
    		\begin{align}\label{eq:sum3}
    		    \begin{split}
        		    &\sum_{k=1}^K\widehat{V}_{k,h'}(s_{h'}^k)-V_{h'}^{\pi^k}(s_{h'}^k)\\
        		    \le&\sum_{k=1}^K\sum_{h=h'}^H4\widehat{\beta}\|\boldsymbol{\phi}(s_h^k,a_h^k)\|_{\widehat{\mathbf{\Lambda}}_{k,h}^{-1}}+\sum_{k=1}^K\sum_{h=h'}^H\left[[\mathbb{P}_{h}(\widehat{V}_{k,h+1}-V_{h+1}^{\pi^k})](s_h^k,a_h^k)-[\widehat{V}_{k,h+1}-V_{h+1}^{\pi^k}](s_{h+1}^k)\right]\\
        		    \le&4\widehat{\beta}\sum_{k=1}^K\sum_{h=1}^H\widehat{\sigma}_{k,h}\|\widehat{\sigma}_{k,h}^{-1}\boldsymbol{\phi}(s_h^k,a_h^k)\|_{\widehat{\mathbf{\Lambda}}_{k,h}^{-1}}+\sum_{k=1}^K\sum_{h=h'}^H\left[[\mathbb{P}_{h}(\widehat{V}_{k,h+1}-V_{h+1}^{\pi^k})](s_h^k,a_h^k)-[\widehat{V}_{k,h+1}-V_{h+1}^{\pi^k}](s_{h+1}^k)\right]\\
        		    \le&4\widehat{\beta}\sqrt{\sum_{k=1}^K\sum_{h=1}^H\widehat{\sigma}_{k,h}^2}\sqrt{H\cdot2d\log\left(1+\frac{K}{Hd\lambda}\right)}+\frac{3H^2d}{\log(2)}\log\left(1+\frac{1}{\lambda\log(2)}\right)+2H\sqrt{2T\log\left(\frac{H}{\delta}\right)},
        	    \end{split}
    		\end{align}
    		where the last inequality holds by similar argument of bounding $I$ in Eq.~(\ref{eq:sum2}) and Lemma~\ref{lm:hoef1} under $\Xi_1$.
    		
    		Also note that
    		\begin{align*}
    		    &\sum_{k=1}^K\sum_{h=1}^H[\mathbb{P}_{h}(\widehat{V}_{k,h+1}-V_{h+1}^{\pi^k})](s_h^k,a_h^k)\\
    		    =&\sum_{k=1}^K\sum_{h=2}^H[\widehat{V}_{k,h+1}-V_{h+1}^{\pi^k}](s_{h+1}^k)+\sum_{k=1}^K\sum_{h=1}^H\left[[\mathbb{P}_{h}(\widehat{V}_{k,h+1}-V_{h+1}^{\pi^k})](s_h^k,a_h^k)-[\widehat{V}_{k,h+1}-V_{h+1}^{\pi^k}](s_{h+1}^k)\right]\\
    		    \le&4\widehat{\beta}H\sqrt{\sum_{k=1}^K\sum_{h=1}^H\widehat{\sigma}_{k,h}^2}\sqrt{H\cdot2d\log\left(1+\frac{K}{Hd\lambda}\right)}+\frac{3H^3d}{\log(2)}\log\left(1+\frac{1}{\lambda\log(2)}\right)\\
    		    &+2H^2\sqrt{2T\log\left(\frac{H}{\delta}\right)}+2H\sqrt{2T\log\left(\frac{H}{\delta}\right)}\\
    		    \le&4\widehat{\beta}H\sqrt{\sum_{k=1}^K\sum_{h=1}^H\widehat{\sigma}_{k,h}^2}\sqrt{H\cdot2d\log\left(1+\frac{K}{Hd\lambda}\right)}+\frac{3H^3d}{\log(2)}\log\left(1+\frac{1}{\lambda\log(2)}\right)+4H^2\sqrt{2T\log\left(\frac{H}{\delta}\right)},
    		\end{align*}
    		where the first inequality holds by sum up Eq.~(\ref{eq:sum3}) for $h'=2,...,H$ and Lemma~\ref{lm:hoef1} under $\Xi_1$.
    		Thus, Eq.~(\ref{eq:sumregret2}) is also obtained.
		\end{proof}
	\end{lemma}
	
	\begin{lemma}[Gap between Optimism and Pessimism]\label{lm:gapop}
		In Algorithm~\ref{alg:plus}, under $\widehat{\Psi}_{1}\cap\widecheck{\Psi}_{1}\cap\Xi_2$, we have
        \begin{align*}
            \sum_{k=1}^K\sum_{h=1}^H[\mathbb{P}_{h}(\widehat{V}_{k,h+1}-\widecheck{V}_{k,h+1})](s_h^k,a_h^k)\le&2(2\widehat{\beta}+\widecheck{\beta})H\sqrt{\sum_{k=1}^K\sum_{h=1}^H\widehat{\sigma}_{k,h}^2}\sqrt{H\cdot2d\log\left(1+\frac{K}{Hd\lambda}\right)}+4H^2\sqrt{2T\log\left(\frac{H}{\delta}\right)}
        \end{align*}
		\begin{proof}
			By Algorithm~\ref{alg:plus}, for any $k\in[K],h\in[H]$, we have
			\begin{equation}\label{eq:gapop}
				\begin{aligned}
					\widehat{Q}_{k,h}(s,a)\le&r_h(s,a)+\langle\widehat{\boldsymbol{\mu}}_{k_0,h}\widehat{\boldsymbol{V}}_{k_0,h+1},\boldsymbol{\phi}(s,a)\rangle+\widehat{\beta}\|\boldsymbol{\phi}(s,a)\|_{ \widehat{\mathbf{\Lambda}}_{k_0,h}^{-1}}\\
					=&r_h(s,a)+\widehat{\mathbb{P}}_{k_0,h}\widehat{V}_{k_0,h+1}(s,a)+\widehat{\beta}\|\boldsymbol{\phi}(s,a)\|_{ \widehat{\mathbf{\Lambda}}_{k_0,h}^{-1}},\\
					\widecheck{Q}_{k,h}(s,a)\ge&r_h(s,a)+\langle\widehat{\boldsymbol{\mu}}_{k,h}\widecheck{\boldsymbol{V}}_{k,h+1},\boldsymbol{\phi}(s,a)\rangle-\widecheck{\beta}\|\boldsymbol{\phi}(s,a)\|_{ \widehat{\mathbf{\Lambda}}_{k,h}^{-1}}\\
					=&r_h(s,a)+\widehat{\mathbb{P}}_{k,h}\widecheck{V}_{k,h+1}(s,a)-\widecheck{\beta}\|\boldsymbol{\phi}(s,a)\|_{ \widehat{\mathbf{\Lambda}}_{k,h}^{-1}},\\
				\end{aligned}
			\end{equation}
			
			Then,
			\begin{align}
				\begin{split}
					&\widehat{V}_{k,h}(s_h^k)-\widecheck{V}_{k,h}(s_h^k)\le\widehat{Q}_{k,h}(s_h^k,a_h^k)-\widecheck{Q}_{k,h}(s_h^k,a_h^k)\\
					\le&r_h(s_h^k,a_h^k)+\widehat{\mathbb{P}}_{k_0,h}\widehat{V}_{k_0,h+1}(s_h^k,a_h^k)+\widehat{\beta}\|\boldsymbol{\phi}(s_h^k,a_h^k)\|_{\widehat{\mathbf{\Lambda}}_{k_0,h}^{-1}}-[r_h(s_h^k,a_h^k)+\widehat{\mathbb{P}}_{k,h}\widecheck{V}_{k,h+1}(s_h^k,a_h^k)-\widecheck{\beta}\|\boldsymbol{\phi}(s_h^k,a_h^k)\|_{ \widehat{\mathbf{\Lambda}}_{k,h}^{-1}}]\\
					=&\widehat{\beta}\|\boldsymbol{\phi}(s_h^k,a_h^k)\|_{\widehat{\mathbf{\Lambda}}_{k_0,h}^{-1}}+\widecheck{\beta}\|\boldsymbol{\phi}(s_h^k,a_h^k)\|_{\widehat{\mathbf{\Lambda}}_{k,h}^{-1}}+[\widehat{\mathbb{P}}_{k_0,h}\widehat{V}_{k_0,h+1}(s_h^k,a_h^k)-\mathbb{P}_h\widehat{V}_{k_0,h+1}(s_h^k,a_h^k)]\\
					&+[\mathbb{P}_h\widecheck{V}_{k,h+1}(s_h^k,a_h^k)-\widehat{\mathbb{P}}_{k,h}\widecheck{V}_{k,h+1}(s_h^k,a_h^k)]+[\mathbb{P}_{h}(\widehat{V}_{k,h+1}-\widecheck{V}_{k,h+1})](s_h^k,a_h^k)\\
					=&\widehat{\beta}\|\boldsymbol{\phi}(s_h^k,a_h^k)\|_{\widehat{\mathbf{\Lambda}}_{k_0,h}^{-1}}+\widecheck{\beta}\|\boldsymbol{\phi}(s_h^k,a_h^k)\|_{\widehat{\mathbf{\Lambda}}_{k,h}^{-1}}+\langle(\widehat{\boldsymbol{\mu}}_{k_0,h+1}-\boldsymbol{\mu}_{h})\widehat{\boldsymbol{V}}_{k_0,h+1},\boldsymbol{\phi}(s_h^k,a_h^k)\rangle\\
					&+\langle(\widehat{\boldsymbol{\mu}}_{k,h+1}-\boldsymbol{\mu}_{h})\widecheck{\boldsymbol{V}}_{k,h+1},\boldsymbol{\phi}(s_h^k,a_h^k)\rangle+[\mathbb{P}_{h}(\widehat{V}_{k,h+1}-\widecheck{V}_{k,h+1})](s_h^k,a_h^k)\\
					\le&\widehat{\beta}\|\boldsymbol{\phi}(s_h^k,a_h^k)\|_{\widehat{\mathbf{\Lambda}}_{k_0,h}^{-1}}+\widecheck{\beta}\|\boldsymbol{\phi}(s_h^k,a_h^k)\|_{\widehat{\mathbf{\Lambda}}_{k,h}^{-1}}+\|(\widehat{\boldsymbol{\mu}}_{k_0,h}-\boldsymbol{\mu}_h)\widehat{\boldsymbol{V}}_{k_0,h+1}\|_{\widehat{\boldsymbol{\Lambda}}_{k_0,h}}\|\boldsymbol{\phi}(s_{h}^{k},a_{h}^{k})\|_{\widehat{\mathbf{\Lambda}}_{k_0,h}^{-1}}\\
					&+\|(\widehat{\boldsymbol{\mu}}_{k,h}-\boldsymbol{\mu}_h)\widecheck{\boldsymbol{V}}_{k,h+1}\|_{\widehat{\boldsymbol{\Lambda}}_{k,h}}\|\boldsymbol{\phi}(s_{h}^{k},a_{h}^{k})\|_{\widehat{\mathbf{\Lambda}}_{k,1}^{-1}}+[\mathbb{P}_{h}(\widehat{V}_{k,h+1}-\widecheck{V}_{k,h+1})](s_h^k,a_h^k)\\
					\le&2\widehat{\beta}\|\boldsymbol{\phi}(s_h^k,a_h^k)\|_{\widehat{\mathbf{\Lambda}}_{k_0,h}^{-1}}+2\widecheck{\beta}\|\boldsymbol{\phi}(s_h^k,a_h^k)\|_{\widehat{\mathbf{\Lambda}}_{k,h}^{-1}}+[\mathbb{P}_{h}(\widehat{V}_{k,h+1}-\widecheck{V}_{k,h+1})](s_h^k,a_h^k)\\
					\le&4\widehat{\beta}\|\boldsymbol{\phi}(s_h^k,a_h^k)\|_{\widehat{\mathbf{\Lambda}}_{k,h}^{-1}}+2\widecheck{\beta}\|\boldsymbol{\phi}(s_h^k,a_h^k)\|_{\widehat{\mathbf{\Lambda}}_{k,h}^{-1}}+[\mathbb{P}_{h}(\widehat{V}_{k,h+1}-\widecheck{V}_{k,h+1})](s_h^k,a_h^k)\\
					=&(4\widehat{\beta}+2\widecheck{\beta})\|\boldsymbol{\phi}(s_h^k,a_h^k)\|_{\widehat{\mathbf{\Lambda}}_{k,h}^{-1}}+[\widehat{V}_{k, h+1}-\widecheck{V}_{k,h+1}](s_{h+1}^{k})+[\mathbb{P}_{h}(\widehat{V}_{k,h+1}-\widecheck{V}_{k,h+1})](s_{h}^{k},a_{h}^{k})-[\widehat{V}_{k, h+1}-\widecheck{V}_{k,h+1}](s_{h+1}^{k})
				\end{split}
			\end{align}
			where the first inequality holds since $\widecheck{V}_{k,h+1}(s_h^k)\ge\widecheck{Q}_{k,h}(s_h^k,a_h^k)$, the second inequality holds due to Eq.~(\ref{eq:gapop}),
			the first equality holds since $\widehat{V}_{k,h}=\widehat{V}_{k_0,h}$,
			the third inequality holds by the Cauchy-Schwarz inequality, the fourth inequality holds since $\boldsymbol{\mu}_{h}\in\widehat{\mathcal{C}}_{k_0,h}\cap\widecheck{\mathcal{C}}_{k,h}$ under $\widehat{\Psi}_{1}\cap\widecheck{\Psi}_{1}$,
			and the last inequality holds due to Lemma~\ref{lm:rsn} with the updating
			rule in Line 8 of Algorithm~\ref{alg:plus}.
			
			Since $\widehat{V}_{k,h'}(s_{h'}^k)-\widecheck{V}_{k,h'}(s_{h'}^k)\le H$, we further obtains
			\begin{align}\label{eq:gapopsum1}
		        \begin{split}
			        &\widehat{V}_{k,h}(s_h^k)-\widecheck{V}_{k,h}(s_h^k)\le\min\left\{\widehat{Q}_{k,h}(s_h^k,a_h^k)-\widecheck{Q}_{k,h}(s_h^k,a_h^k),H\right\}\\
			        \le&\min\left\{(4\widehat{\beta}+2\widecheck{\beta})\|\boldsymbol{\phi}(s_h^k,a_h^k)\|_{\widehat{\mathbf{\Lambda}}_{k,h}^{-1}},H\right\}+[\mathbb{P}_{h}(\widehat{V}_{k,h+1}-\widecheck{V}_{k,h+1})](s_h^k,a_h^k)\\
			        \le&(4\widehat{\beta}+2\widecheck{\beta})\widehat{\sigma}_{k,h}\min\left\{\|\widehat{\sigma}_{k,h}^{-1}\boldsymbol{\phi}(s_h^k,a_h^k)\|_{\widehat{\mathbf{\Lambda}}_{k,h}^{-1}},1\right\}+[\mathbb{P}_{h}(\widehat{V}_{k,h+1}-\widecheck{V}_{k,h+1})](s_h^k,a_h^k)\\
					=&(4\widehat{\beta}+2\widecheck{\beta})\widehat{\sigma}_{k,h}\min\left\{\|\widehat{\sigma}_{k,h}^{-1}\boldsymbol{\phi}(s_h^k,a_h^k)\|_{\widehat{\mathbf{\Lambda}}_{k,h}^{-1}},1\right\}+[\widehat{V}_{k, h+1}-\widecheck{V}_{k,h+1}](s_{h+1}^{k})\\
					&+[\mathbb{P}_{h}(\widehat{V}_{k,h+1}-\widecheck{V}_{k,h+1})](s_{h}^{k},a_{h}^{k})-[\widehat{V}_{k, h+1}-\widecheck{V}_{k,h+1}](s_{h+1}^{k}),
				\end{split}
			\end{align}
			where the second inequality holds since $\widecheck{\beta}\widehat{\sigma}_{k,h}\ge\sqrt{Hd^2}\sqrt{H}\ge H$.
			Summing up Eq.~(\ref{eq:gapopsum1}) for $k\in[K]$ and $h=h',...,H$ gives
			\begin{align}\label{eq:gapopsum2}
			    \begin{split}
    			    &\sum_{k=1}^K\widehat{V}_{k,h'}(s_{h'}^k)-\widecheck{V}_{k,h'}(s_{h'}^k)\\
    				\le&\sum_{k=1}^K\sum_{h=h'}^H(4\widehat{\beta}+2\widecheck{\beta})\widehat{\sigma}_{k,h}\min\left\{\|\widehat{\sigma}_{k,h}^{-1}\boldsymbol{\phi}(s_h^k,a_h^k)\|_{\widehat{\mathbf{\Lambda}}_{k,h}^{-1}},1\right\}\\
    				&+\sum_{k=1}^K\sum_{h=h'}^H[\mathbb{P}_{h}(\widehat{V}_{k,h+1}-\widecheck{V}_{k,h+1})](s_{h}^{k},a_{h}^{k})-[\widehat{V}_{k, h+1}-\widecheck{V}_{k,h+1}](s_{h+1}^{k})\\
    				\le&\sum_{k=1}^K\sum_{h=1}^H(4\widehat{\beta}+2\widecheck{\beta})\widehat{\sigma}_{k,h}\min\left\{\|\widehat{\sigma}_{k,h}^{-1}\boldsymbol{\phi}(s_h^k,a_h^k)\|_{\widehat{\mathbf{\Lambda}}_{k,h}^{-1}},1\right\}+2H\sqrt{2T\log\left(\frac{H}{\delta}\right)}\\
    				\le&(4\widehat{\beta}+2\widecheck{\beta})\sqrt{\sum_{k=1}^K\sum_{h=1}^H\widehat{\sigma}_{k,h}^2}\sqrt{\sum_{k=1}^K\sum_{h=1}^H\min\left\{\|\widehat{\sigma}_{k,h}^{-1}\boldsymbol{\phi}(s_h^k,a_h^k)\|_{\widehat{\mathbf{\Lambda}}_{k,h}^{-1}}^2,1\right\}}+2H\sqrt{2T\log\left(\frac{H}{\delta}\right)}\\
    				\le&(4\widehat{\beta}+2\widecheck{\beta})\sqrt{\sum_{k=1}^K\sum_{h=1}^H\widehat{\sigma}_{k,h}^2}\sqrt{H\cdot2d\log\left(1+\frac{K}{Hd\lambda}\right)}+2H\sqrt{2T\log\left(\frac{H}{\delta}\right)}
    			\end{split}
			\end{align}
			where the second inequality holds by Lemma~\ref{lm:hoef2} under $\Xi_2$, the third inequality holds due to Cauchy-Schwarz inequality, and the last inequality hols due to Lemma~\ref{lm:ablog} with the fact that $\|\widehat{\sigma}_{k,h}^{-1}\boldsymbol{\phi}(s_h^k,a_h^k)\|_2\le1/\sqrt{H}$.

			Also note that
    		\begin{align*}
    		    &\sum_{k=1}^K\sum_{h=1}^H[\mathbb{P}_{h}(\widehat{V}_{k,h+1}-\widecheck{V}_{k,h+1})](s_h^k,a_h^k)\\
    		    =&\sum_{k=1}^K\sum_{h=2}^H[\widehat{V}_{k,h+1}-\widecheck{V}_{k,h+1}](s_{h+1}^k)+\sum_{k=1}^K\sum_{h=1}^H\left[[\mathbb{P}_{h}(\widehat{V}_{k,h+1}-\widecheck{V}_{k,h+1})](s_h^k,a_h^k)-[\widehat{V}_{k,h+1}-\widecheck{V}_{k,h+1}](s_{h+1}^k)\right]\\
    		    \le&(4\widehat{\beta}+2\widecheck{\beta})H\sqrt{\sum_{k=1}^K\sum_{h=1}^H\widehat{\sigma}_{k,h}^2}\sqrt{H\cdot2d\log\left(1+\frac{K}{Hd\lambda}\right)}+2H^2\sqrt{2T\log\left(\frac{H}{\delta}\right)}+2H\sqrt{2T\log\left(\frac{H}{\delta}\right)}\\
    		    \le&2(2\widehat{\beta}+\widecheck{\beta})H\sqrt{\sum_{k=1}^K\sum_{h=1}^H\widehat{\sigma}_{k,h}^2}\sqrt{H\cdot2d\log\left(1+\frac{K}{Hd\lambda}\right)}+4H^2\sqrt{2T\log\left(\frac{H}{\delta}\right)},
    		\end{align*}
    		where the first inequality holds by sum up Eq.~(\ref{eq:gapopsum2}) for $h'=2,...,H$ and Lemma~\ref{lm:hoef2} under $\Xi_2$.
		\end{proof}
	\end{lemma}
	
	\subsection{Summation of Estimated Variances}\label{sec:appsigma}
	
	In this subsection, we try to bound the summation of estimated variance in Lemma~\ref{lm:sigma}.
	As shown in Lemma~\ref{lm:sigma}, the summation of estimated variance $\sum_{k=1}^K\sum_{h=1}^H\widehat{\sigma}_{k,h}^2=\widetilde{O}(HT)$, which utilizes the Law of Total Variance in \cite{lattimore2012pac}, detailed in Lemma~\ref{lm:tvl}.
	
	\begin{lemma}[Total variance lemma, Lemma C.5 in \cite{jin2018q}]\label{lm:tvl}
		With probability at least $1-\delta$, we have
		\begin{align*}
		    \sum_{k=1}^{K} \sum_{h=1}^{H}\left[\mathbb{V}_{h} V_{h+1}^{\pi^{k}}\right](s_{h}^{k},a_{h}^{k}) \leq 3\left(H T+H^{3} \log (1 / \delta)\right).
		\end{align*}
	\end{lemma}
	
	
	By the definition of $\widehat{\sigma}_{k,h}$, the summation of $\varsigma_{k,h}^2$ will influence the summation of $\widehat{\sigma}_{k,h}^2$.
	We need to make $\sum_{k=1}^{K}\sum_{h=1}^{H}\varsigma_{k,h}^2$ small such that it will not become dominant term in the upper bound.
	However, enlarging $\varsigma_{k,h}$ is required in some stages of some episodes, as stated in Remark~\ref{rm:small} in the main paper.
	To address this dilemma, we build the following critical lemma which characterizes the conservatism of the elliptical potential, i.e., $\left\|\boldsymbol{x}_{t}\right\|_{\mathbf{Z}_{t-1}^{-1}}$ is small in most episodes, as detailed in Lemma~\ref{lm:od}.
	
	\begin{lemma}[Elliptical Potentials: You cannot have many big intervals]\label{lm:od}
		Given $\lambda>0$ and sequence $\left\{\boldsymbol{x}_{t}\right\}_{t=1}^{T} \subset$ $\mathbb{R}^{d}$ with $\left\|\boldsymbol{x}_{t}\right\|_{2} \leq L$ for all $t\in[T]$, define $\mathbf{Z}_{t}=\lambda \mathbf{I}+\sum_{i=1}^{t} \boldsymbol{x}_{i} \boldsymbol{x}_{i}^{\top}$ for $t\ge1$ and $\mathbf{Z}_{0}=\lambda \mathbf{I}$. During $[T]$, the number of times $\left\|\boldsymbol{x}_{t}\right\|_{\mathbf{Z}_{t-1}^{-1}}\geq c$ is at most
		\begin{align*}
		    \frac{3 d}{\log (1+c^2)} \log \left(1+\frac{L^{2}}{\lambda \log (1+c^2)}\right),
		\end{align*}
		where $c>0$ is a constant.
		\begin{proof}
			The proof of this lemma is firstly proposed at Exercise 19.3 in \cite{lattimore2020bandit} for the case of $C=1$, i.e., Lemma~\ref{lm:ood}, we generalize it to the case with any positive constant $C$.
			
			Let $\mathcal{T}$ be the set of rounds $t$ when $\left\|\boldsymbol{x}_{t}\right\|_{\mathbf{Z}_{t-1}^{-1}}^{2}\geq C$ for $t\in[T]$ and $\mathbf{Y}_{t}=\mathbf{Z}_{0}+\sum_{i=1}^{t}\mathds{1}\left\{i\in\mathcal{T}\right\}\boldsymbol{x}_{i} \boldsymbol{x}_{i}^{\top}$. Then
			\begin{align*}
				\left(\frac{d \lambda+|\mathcal{T}| L^{2}}{d}\right)^{d} & \geq\left(\frac{\operatorname{trace}\left(\mathbf{Y}_{T}\right)}{d}\right)^{d} \\
				& \geq \operatorname{det}\left(\mathbf{Y}_{T}\right) \\
				&=\operatorname{det}\left(\mathbf{Z}_{0}\right) \prod_{t \in \mathcal{T}}\left(1+\left\|\boldsymbol{x}_{t}\right\|_{\mathbf{Y}_{t-1}^{-1}}^{2}\right)\\
				& \geq \operatorname{det}\left(\mathbf{Z}_{0}\right) \prod_{t \in \mathcal{T}}\left(1+\left\|\boldsymbol{x}_{t}\right\|_{\mathbf{Z}_{t-1}^{-1}}^{2}\right)\\
				& \geq \lambda^{d}(1+c^2)^{|\mathcal{T}|}
			\end{align*}
			Rearranging and taking the logarithm show that
			\begin{align*}
			    |\mathcal{T}| \leq \frac{d}{\log (1+c^2)} \log \left(1+\frac{|\mathcal{T}| L^{2}}{d \lambda}\right).
			\end{align*}
			
			Abbreviate $x=d / \log (1+c^2)$ and $y=L^{2} / d \lambda$, which are both positive. Then
			\begin{align*}
			    x \log (1+y(3 x \log (1+x y))) \leq x \log \left(1+3 x^{2} y^{2}\right) \leq x \log (1+x y)^{3}=3 x \log (1+x y) .
			\end{align*}
			Define $f(z)=z-x \log (1+y z)$ for $z\ge0$, we have $f'(z)=[1+y(z-x)]/(1+yz)$, which implies $f(z)$ is increasing if $1-xy\ge0$, or $f(z)$ is first decreasing then increasing, otherwise. Since $f(0)=0$ and $f(3x\log(1+xy))\ge0$, if $f(z)\le0$, we must have $z\le3x\log(1+xy)$. In other words, $f(z)$ is increasing for $z \geq 3 x \log (1+x y)$. It then follows that
			\begin{align*}
			    |\mathcal{T}| \leq 3 x \log (1+x y)=\frac{3 d}{\log (1+c^2)} \log \left(1+\frac{L^{2}}{\lambda \log (1+c^2)}\right).
			\end{align*}
		\end{proof}
	\end{lemma}
	
	The following Lemma is required to upper bounds the summation of offset term $U_{k,h}$.
	
	\begin{lemma}\label{lm:delta}
		In Algorithm~\ref{alg:plus}, under $\Upsilon$, for any $k\in[K]$ and any $h\in[H]$, we have
		\begin{align*}
			&\left|\langle\widehat{\boldsymbol{\mu}}_{k,h}(\widehat{\boldsymbol{V}}_{k,h+1}-\widecheck{\boldsymbol{V}}_{k,h+1}),\boldsymbol{\phi}(s_h^k,a_h^k)\rangle\right|
			\le&\widebar{\beta}\left\|\boldsymbol{\phi}(s_{h}^{k},a_{h}^{k})\right\|_{\widehat{\mathbf{\Lambda}}_{k,h}^{-1}}+\mathbb{P}_{h}(\widehat{V}_{k,h+1}-\widecheck{V}_{k,h+1})(s_h^k,a_h^k)+\widecheck{\beta}\left\|\boldsymbol{\phi}(s_{h}^{k},a_{h}^{k})\right\|_{\widehat{\mathbf{\Lambda}}_{k,h}^{-1}}
		\end{align*}
		\begin{proof}
			\begin{align*}
				&\left|\langle\widehat{\boldsymbol{\mu}}_{k,h}(\widehat{\boldsymbol{V}}_{k,h+1}-\widecheck{\boldsymbol{V}}_{k,h+1}),\boldsymbol{\phi}(s_h^k,a_h^k)\rangle\right|=\left|\widehat{\mathbb{P}}_{k,h}\widehat{V}_{k,h+1}(s_h^k,a_h^k)-\widehat{\mathbb{P}}_{k,h}\widecheck{V}_{k,h+1}(s_h^k,a_h^k)\right|\\
				\le&\left|\widehat{\mathbb{P}}_{k,h}\widehat{V}_{k,h+1}(s_h^k,a_h^k)-\mathbb{P}_{h}\widehat{V}_{k,h+1}(s_h^k,a_h^k)\right|+\left|\mathbb{P}_{h}\widehat{V}_{k,h+1}(s_h^k,a_h^k)-\mathbb{P}_{h}\widecheck{V}_{k,h+1}(s_h^k,a_h^k)\right|\\
				&+\left|\mathbb{P}_{h}\widecheck{V}_{k,h+1}(s_h^k,a_h^k)-\widehat{\mathbb{P}}_{k,h}\widecheck{V}_{k,h+1}(s_h^k,a_h^k)\right|\\
				=&\left|\langle(\boldsymbol{\mu}_{h}-\widehat{\boldsymbol{\mu}}_{k,h})\widehat{\boldsymbol{V}}_{k,h+1},\boldsymbol{\phi}(s_h^k, a_h^k)\rangle\right|+\left|\mathbb{P}_{h}(\widehat{V}_{k,h+1}-\widecheck{V}_{k,h+1})(s_h^k,a_h^k)\right|+\left|\langle(\boldsymbol{\mu}_{h}-\widehat{\boldsymbol{\mu}}_{k,h})\widecheck{\boldsymbol{V}}_{k,h+1},\boldsymbol{\phi}(s_h^k, a_h^k)\rangle\right|\\
				=&\left|\langle(\boldsymbol{\mu}_{h}-\widehat{\boldsymbol{\mu}}_{k,h})\widehat{\boldsymbol{V}}_{k,h+1},\boldsymbol{\phi}(s_h^k, a_h^k)\rangle\right|+\mathbb{P}_{h}(\widehat{V}_{k,h+1}-\widecheck{V}_{k,h+1})(s_h^k,a_h^k)+\left|\langle(\boldsymbol{\mu}_{h}-\widehat{\boldsymbol{\mu}}_{k,h})\widecheck{\boldsymbol{V}}_{k,h+1},\boldsymbol{\phi}(s_h^k, a_h^k)\rangle\right|\\
				\le&\left\|(\boldsymbol{\mu}_{h}-\widehat{\boldsymbol{\mu}}_{k,h})\widehat{\boldsymbol{V}}_{k,h+1}\right\|_{\widehat{\boldsymbol{\Lambda}}_{k,h}}\left\|\boldsymbol{\phi}(s_{h}^{k},a_{h}^{k})\right\|_{\widehat{\mathbf{\Lambda}}_{k,h}^{-1}}+\mathbb{P}_{h}(\widehat{V}_{k,h+1}-\widecheck{V}_{k,h+1})(s_h^k,a_h^k)\\
				&+\left\|(\boldsymbol{\mu}_{h}-\widehat{\boldsymbol{\mu}}_{k,h})\widecheck{\boldsymbol{V}}_{k,h+1}\right\|_{\widehat{\boldsymbol{\Lambda}}_{k,h}}\left\|\boldsymbol{\phi}(s_{h}^{k},a_{h}^{k})\right\|_{\widehat{\mathbf{\Lambda}}_{k,h}^{-1}}\\
				\le&\widebar{\beta}\left\|\boldsymbol{\phi}(s_{h}^{k},a_{h}^{k})\right\|_{\widehat{\mathbf{\Lambda}}_{k,h}^{-1}}+\mathbb{P}_{h}(\widehat{V}_{k,h+1}-\widecheck{V}_{k,h+1})(s_h^k,a_h^k)+\widecheck{\beta}\left\|\boldsymbol{\phi}(s_{h}^{k},a_{h}^{k})\right\|_{\widehat{\mathbf{\Lambda}}_{k,h}^{-1}},
			\end{align*}
			where the first inequality holds due to triangle inequality,
			the second equality holds since $\widehat{V}_{k,h+1}(\cdot)\ge\widecheck{V}_{k,h+1}(\cdot)$ under $\Upsilon$ by Lemma~\ref{lm:optimism},
			the second inequality holds due to Cauchy-Schwarz inequality,
			and the last inequality holds since under $\Upsilon$ we have $\boldsymbol{\mu}_{h}\in\widebar{\mathcal{C}}_{k, h}\cap\widecheck{\mathcal{C}}_{k, h}$.
		\end{proof}
	\end{lemma}
	
	Now we are ready to upper bounds $\sum_{k=1}^{K} \sum_{h=1}^{H} \widehat{\sigma}_{k, h}^{2}$ in Lemma~\ref{lm:sigma} under the high probability event $\Upsilon\cap\Xi_2\cap\Xi_3$.
	
	\begin{lemma}\label{lm:sigma}
		In Algorithm~\ref{alg:plus}, under event $\Upsilon\cap\Xi_2\cap\Xi_3$, we have
		\begin{align*}
				\sum_{k=1}^{K}\sum_{h=1}^{H}\widehat{\sigma}_{k, h}^{2}\le&8HT+\frac{6H^5d^6}{\log(1+c^2)}\log\left(1+\frac{d}{\lambda H\log (1+c^2)}\right)+6H^{3} \log\left(\frac{1}{\delta}\right)\\
				&+8H^3(6+Hd^3)\sqrt{2T\log\left(\frac{H}{\delta}\right)}+\frac{12H^4d}{\log(2)}\log\left(1+\frac{1}{\lambda\log(2)}\right)+2H^4d^3\sqrt{\lambda}+\\
				&8Hd\left[H\left(2H(6+Hd^3)\widehat{\beta}+(4+Hd^3)\widebar{\beta}+(H+1)(4+Hd^3)\widecheck{\beta}\right)+\widetilde{\beta}\right]^2\log\left(1+\frac{K}{Hd\lambda}\right)
		\end{align*}
		where $c=1/(H^3d^5)$.
		\begin{proof}
			Initially, by definition of $\widehat{\sigma}_{k, h}$ in Algorithm~\ref{alg:plus}, we have
			\begin{align}\label{eq:sigmat}
				\begin{split}
					\sum_{k=1}^{K}\sum_{h=1}^{H}\widehat{\sigma}_{k, h}^{2}
					\le&\underbrace{\sum_{k=1}^{K}\sum_{h=1}^{H}\varsigma_{k,h}^2}_{I_1}+\underbrace{\sum_{k=1}^{K}\sum_{h=1}^{H}Hd^3E_{k,h}}_{I_2}+\underbrace{\sum_{k=1}^{K}\sum_{h=1}^{H}\left[[\widehat{\mathbb{V}}_{k,h}\widehat{V}_{k,h+1}](s_{h}^{k},a_{h}^{k})+U_{k, h}\right]}_{I_3},
				\end{split}
			\end{align}
            \paragraph{Bounding $I_1$:}
			Denote $c=1/(H^3d^5)$.
			For fixed $h\in[H]$, set $\boldsymbol{x}_k$ as $\widetilde{\sigma}_{k,h}^{-1}\boldsymbol{\phi}(s_h^k, a_h^k)$ in Lemma~\ref{lm:od}.
			Then, there are at most $3d\log[1+d/(\lambda H\log (1+c^2))]/\log (1+c^2)$ episodes that $\|\widetilde{\sigma}_{k,h}^{-1}\boldsymbol{\phi}(s_h^k, a_h^k)\|_{\widetilde{\mathbf{\Lambda}}_{k,h}^{-1}}\ge c$ such that $\varsigma_{k}=H^2\sqrt{d^5}$.
			Thus, we obtain
			\begin{align}\label{eq:sumsigmai1}
			    \begin{split}
			        I_1\le\sum_{k=1}^{K}\sum_{h=1}^{H}H+H\cdot H^4d^5\frac{3d}{\log(1+c^2)}\log\left(1+\frac{d}{\lambda H\log (1+c^2)}\right).
			    \end{split}
			\end{align}
			
			\paragraph{Bounding $I_2$:}
			
			Since $\boldsymbol{\mu}_h\in\widebar{\mathcal{C}}_{k,h}\cap\widecheck{\mathcal{C}}_{k,h}$ under $\Upsilon$, we have
			\begin{align}\label{eq:sumsigmarelation1}
	            \begin{split}
	                \langle\widehat{\boldsymbol{\mu}}_{k,h}\widehat{\boldsymbol{V}}_{k,h+1},\boldsymbol{\phi}(s_h^k,a_h^k)\rangle=\widehat{\mathbb{P}}_{k,h}\widehat{V}_{k,h+1}(s_h^k,a_h^k)\le&\mathbb{P}_{h}\widehat{V}_{k,h+1}(s_h^k,a_h^k)+\widebar{\beta}\left\|\boldsymbol{\phi}(s_{h}^{k},a_{h}^{k})\right\|_{\widehat{\mathbf{\Lambda}}_{k,h}^{-1}}\\
	                \langle\widehat{\boldsymbol{\mu}}_{k,h}\widecheck{\boldsymbol{V}}_{k,h+1},\boldsymbol{\phi}(s_h^k,a_h^k)\rangle=\widehat{\mathbb{P}}_{k,h}\widecheck{V}_{k,h+1}(s_h^k,a_h^k)\ge&\mathbb{P}_{h}\widecheck{V}_{k,h+1}(s_h^k,a_h^k)-\widecheck{\beta}\left\|\boldsymbol{\phi}(s_{h}^{k},a_{h}^{k})\right\|_{\widehat{\mathbf{\Lambda}}_{k,h}^{-1}}
	            \end{split}
	        \end{align}
	        Combing Eq.~(\ref{eq:sumsigmarelation1}) and the definition of $E_{k,h}$ in Eq.~(\ref{eq:ekh}) gives
	        \begin{align}\label{eq:sumsigmai2}
			    \begin{split}
			        I_2\le&H^2d^3\sum_{k=1}^{K}\sum_{h=1}^{H}[\mathbb{P}_{h}(\widehat{V}_{k,h+1}-\widecheck{V}_{k,h+1})](s_h^k,a_h^k)+H^2d^3\sum_{k=1}^{K}\sum_{h=1}^{H}\min\left\{2(\widebar{\beta}+\widecheck{\beta})\left\|\boldsymbol{\phi}(s_{h}^{k},a_{h}^{k})\right\|_{\widehat{\mathbf{\Lambda}}_{k,h}^{-1}},H\right\}+H^4d^3\sqrt{\lambda}\\
			        \le&H^2d^3\sum_{k=1}^{K}\sum_{h=1}^{H}[\mathbb{P}_{h}(\widehat{V}_{k,h+1}-\widecheck{V}_{k,h+1})](s_h^k,a_h^k)+2(\widebar{\beta}+\widecheck{\beta})H^2d^3\sum_{k=1}^{K}\sum_{h=1}^{H}\widehat{\sigma}_{k,h}\min\left\{\left\|\widehat{\sigma}_{k,h}^{-1}\boldsymbol{\phi}(s_{h}^{k},a_{h}^{k})\right\|_{\widehat{\mathbf{\Lambda}}_{k,h}^{-1}},1\right\}\\
			        &+H^4d^3\sqrt{\lambda}
			    \end{split}
			\end{align}
			where the second inequality holds since $\widecheck{\beta}\widehat{\sigma}_{k,h}\ge\sqrt{Hd^2}\sqrt{H}\ge H$.
			
			\paragraph{Bounding $I_3$:}
			$I_3$ can be bounded by
			\begin{align*}	I_3\le&\underbrace{\sum_{k=1}^{K} \sum_{h=1}^{H}\left[\mathbb{V}_{h} V_{h+1}^*\right](s_{h}^{k},a_{h}^{k})-\left[\mathbb{V}_{h} V_{h+1}^{\pi^{k}}\right](s_{h}^{k},a_{h}^{k})}_{J_1}+2 \underbrace{\sum_{k=1}^{K} \sum_{h=1}^{H} U_{k, h}}_{J_2}+\underbrace{\sum_{k=1}^{K} \sum_{h=1}^{H}\left[\mathbb{V}_{h} V_{h+1}^{\pi^{k}}\right](s_{h}^{k},a_{h}^{k}) }_{J_3}\\
					&+\underbrace{\sum_{k=1}^{K} \sum_{k=1}^{H}\left[\left[\widehat{\mathbb{V}}_{k, h} \widehat{V}_{k, h+1}\right](s_{h}^{k},a_{h}^{k})-\left[\mathbb{V}_{h} V_{h+1}^*\right](s_{h}^{k},a_{h}^{k})-U_{k, h}\right] }_{J_4}
			\end{align*}
			
			\paragraph{Bounding $J_1$}
			\begin{align}\label{eq:sumsigmaj1}
				\begin{split}
					J_{1}=&\sum_{k=1}^{K} \sum_{h=1}^{H}\left[\mathbb{P}_{h} {V_{h+1}^*}^{2}(s_{h}^{k},a_{h}^{k})-\left[\mathbb{P}_{h} V_{h+1}^*(s_{h}^{k}, a_{h}^{k})\right]^2\right]-\left[\mathbb{P}_{h}{V_{h+1}^{\pi^{k}}}^{2}(s_{h}^{k}, a_{h}^{k})-\left[\mathbb{P}_{h}V_{h+1}^{\pi^{k}}(s_{h}^{k}, a_{h}^{k})\right]^2\right]\\
					\le&\sum_{k=1}^{K} \sum_{h=1}^{H}\mathbb{P}_{h} {V_{h+1}^*}^{2}(s_{h}^{k}, a_{h}^{k})-\mathbb{P}_{h}{V_{h+1}^{\pi^{k}}}^{2}(s_{h}^{k}, a_{h}^{k}) \\
					\le&2 H \sum_{k=1}^{K} \sum_{h=1}^{H}[\mathbb{P}_{h}(V_{h+1}^*-V_{h+1}^{\pi^k})](s_{h}^{k}, a_{h}^{k})\\
					\le&2 H \sum_{k=1}^{K} \sum_{h=1}^{H}[\mathbb{P}_{h}(\widehat{V}_{k,h+1}-V_{h+1}^{\pi^k})](s_{h}^{k}, a_{h}^{k}),
				\end{split}
			\end{align}
			where the first inequality holds since $V_{h+1}^*(\cdot)\ge V_{h+1}^{\pi^k}$, the second inequality holds since $V_{h+1}^{\pi^k}\le V_{h+1}^*(\cdot)\le H$,
			and the last inequality holds since $\widehat{V}_{k,h+1}(\cdot)\ge V_{h+1}^*(\cdot)$ under $\Upsilon$ by Lemma~\ref{lm:optimism}.
			
			\paragraph{Bounding $J_2$}
			By the definition of $U_{k,h}$ in Eq.~(\ref{eq:ukh}) and Lemma~\ref{lm:delta}, we have
			
			\begin{align}\label{eq:sumsigmaj2}
			    \begin{split}
			        J_2\le&
			        \sum_{k=1}^K\sum_{h=1}^H\left\{2\widetilde{\beta}\left\|\boldsymbol{\phi}(s_{h}^{k},a_{h}^{k})\right\|_{\widehat{\mathbf{\Lambda}}_{k,h}^{-1}},2H^2\right\}+4H\sum_{k=1}^K\sum_{h=1}^H[\mathbb{P}_{h}(\widehat{V}_{k,h+1}-\widecheck{V}_{k,h+1})](s_h^k,a_h^k)\\
			        &+\sum_{k=1}^K\sum_{h=1}^H\min\left\{8H(\widebar{\beta}+\widecheck{\beta})\left\|\boldsymbol{\phi}(s_{h}^{k},a_{h}^{k})\right\|_{\widehat{\mathbf{\Lambda}}_{k,h}^{-1}},2H^2\right\}\\
			        \le&[2\widetilde{\beta}+8H(\widebar{\beta}+\widecheck{\beta})]\sum_{k=1}^K\sum_{h=1}^H\widehat{\sigma}_{k,h}\min\left\{\left\|\widehat{\sigma}_{k,h}^{-1}\boldsymbol{\phi}(s_{h}^{k},a_{h}^{k})\right\|_{\widehat{\mathbf{\Lambda}}_{k,h}^{-1}},1\right\}+4H\sum_{k=1}^K\sum_{h=1}^H[\mathbb{P}_{h}(\widehat{V}_{k,h+1}-\widecheck{V}_{k,h+1})](s_h^k,a_h^k),
			    \end{split}
			\end{align}
			where the second inequality holds since $\widetilde{\beta}\widehat{\sigma}_{k,h}\ge\sqrt{H^3}\cdot\sqrt{H}\ge H^2$, $8H(\widebar{\beta}+\widecheck{\beta})\widehat{\sigma}_{k,h}\ge8H\sqrt{H}\cdot \sqrt{H}\ge2H^2$.

			\paragraph{Bounding $J_3$}
			Since $\Xi_3$ holds, we have
			\begin{align}\label{eq:sumsigmaj3}
			    J_3\leq 3\left[H T+H^{3} \log\left(\frac{1}{\delta}\right)\right]
			\end{align}
			\paragraph{Bounding $J_4$}
			Due to Lemma~\ref{lm:cs}, we have
			\begin{align}\label{eq:sumsigmaj4}
			    J_{4} \leq 0
			\end{align}
			
			\paragraph{Putting Together}
			Initially, we have
			\begin{align}\label{eq:ptsigma}
			    \begin{split}
			        &\sum_{k=1}^K\sum_{h=h'}^H\widehat{\sigma}_{k,h}\min\left\{\|\widehat{\sigma}_{k,h}^{-1}\boldsymbol{\phi}(s_h^k,a_h^k)\|_{\widehat{\mathbf{\Lambda}}_{k,h}^{-1}},1\right\}\\
    			   	\le&\sqrt{\sum_{k=1}^K\sum_{h=1}^H\widehat{\sigma}_{k,h}^2}\sqrt{\sum_{k=1}^K\sum_{h=1}^H\min\left\{\|\widehat{\sigma}_{k,h}^{-1}\boldsymbol{\phi}(s_h^k,a_h^k)\|_{\widehat{\mathbf{\Lambda}}_{k,h}^{-1}}^2,1\right\}}\\
    			   	\le&\sqrt{\sum_{k=1}^K\sum_{h=1}^H\widehat{\sigma}_{k,h}^2}\sqrt{H\cdot2d\log\left(1+\frac{K}{Hd\lambda}\right)},
			    \end{split}
			\end{align}
			where the first inequality holds due to Cauchy-Schwarz inequality,
			and the second inequality holds due to Lemma~\ref{lm:ablog} with the fact that $\|\widehat{\sigma}_{k,h}^{-1}\boldsymbol{\phi}(s_h^k,a_h^k)\|_2\le1/\sqrt{H}$.
			
			Subsequently, combining Eq.~(\ref{eq:sumsigmai2}), (\ref{eq:sumsigmaj1}), (\ref{eq:sumsigmaj2}) gives
			\begin{align*}
			    &I_2+J_1+J_2\\
			    \le&2H\sum_{k=1}^{K}\sum_{h=1}^{H}[\mathbb{P}_{h}(\widehat{V}_{k,h+1}-V_{h+1}^{\pi^k})](s_{h}^{k},a_{h}^{k})+H(4+Hd^3)\sum_{k=1}^{K}\sum_{h=1}^{H}[\mathbb{P}_{h}(\widehat{V}_{k,h+1}-\widecheck{V}_{k,h+1})](s_h^k,a_h^k)+H^4d^3\sqrt{\lambda}\\
			    &+2\left[H\left((4+Hd^3)\widebar{\beta}+(4+Hd^3)\widecheck{\beta}\right)+\widetilde{\beta}\right]\sum_{k=1}^K\sum_{h=1}^H\widehat{\sigma}_{k,h}\min\left\{\left\|\widehat{\sigma}_{k,h}^{-1}\boldsymbol{\phi}(s_{h}^{k},a_{h}^{k})\right\|_{\widehat{\mathbf{\Lambda}}_{k,h}^{-1}},1\right\}\\
			    \le&2H\sum_{k=1}^{K}\sum_{h=1}^{H}[\mathbb{P}_{h}(\widehat{V}_{k,h+1}-V_{h+1}^{\pi^k})](s_{h}^{k},a_{h}^{k})+H(4+Hd^3)\sum_{k=1}^{K}\sum_{h=1}^{H}[\mathbb{P}_{h}(\widehat{V}_{k,h+1}-\widecheck{V}_{k,h+1})](s_h^k,a_h^k)+H^4d^3\sqrt{\lambda}\\
			    &+2\left[H\left((4+Hd^3)\widebar{\beta}+(4+Hd^3)\widecheck{\beta}\right)+\widetilde{\beta}\right]\sqrt{\sum_{k=1}^K\sum_{h=1}^H\widehat{\sigma}_{k,h}^2}\sqrt{H\cdot2d\log\left(1+\frac{K}{Hd\lambda}\right)}\\
			    \le&4H^3(6+Hd^3)\sqrt{2T\log\left(\frac{H}{\delta}\right)}+\frac{6H^4d}{\log(2)}\log\left(1+\frac{1}{\lambda\log(2)}\right)+H^4d^3\sqrt{\lambda}+\\
			    &2\left[H\left(2H(6+Hd^3)\widehat{\beta}+(4+Hd^3)\widebar{\beta}+(H+1)(4+Hd^3)\widecheck{\beta}\right)+\widetilde{\beta}\right]\sqrt{\sum_{k=1}^K\sum_{h=1}^H\widehat{\sigma}_{k,h}^2}\sqrt{H\cdot2d\log\left(1+\frac{K}{Hd\lambda}\right)}
			\end{align*}
			where the second inequality holds due to Eq.~\ref{eq:ptsigma},
			and the third inequality holds by Lemma~\ref{lm:sumregret} under $\Upsilon\cap\Xi_1$ and Lemma~\ref{lm:gapop} under $\Upsilon\cap\Xi_2$.
			Further considering Eq.~(\ref{eq:sumsigmai1}), (\ref{eq:sumsigmaj3}), (\ref{eq:sumsigmaj4}) gives
			\begin{align}\label{eq:sumsigmaf}
			    \begin{split}
			        &\sum_{k=1}^{K}\sum_{h=1}^{H}\widehat{\sigma}_{k, h}^{2}\\
			        \le&4HT+\frac{3H^5d^6}{\log(1+c^2)}\log\left(1+\frac{d}{\lambda H\log (1+c^2)}\right)+3H^{3} \log\left(\frac{1}{\delta}\right)\\
			        &+4H^3(6+Hd^3)\sqrt{2T\log\left(\frac{H}{\delta}\right)}+\frac{6H^4d}{\log(2)}\log\left(1+\frac{1}{\lambda\log(2)}\right)+H^4d^3\sqrt{\lambda}+\\
			        &2\left[H\left(2H(6+Hd^3)\widehat{\beta}+(4+Hd^3)\widebar{\beta}+(H+1)(4+Hd^3)\widecheck{\beta}\right)+\widetilde{\beta}\right]\sum_{k=1}^K\sum_{h=1}^H\sqrt{\sum_{k=1}^K\sum_{h=1}^H\widehat{\sigma}_{k,h}^2}\sqrt{H\cdot2d\log\left(1+\frac{K}{Hd\lambda}\right)}
			    \end{split}
			\end{align}
			Besides, for any $x,a,b\ge0$, if $x\le a\sqrt{x}+b$, then we have $\sqrt{x}\le\sqrt{b+a^2/4}+\sqrt{a^2/4}\le\sqrt{2(b+a^2/4+a^2/4)}$, i.e., $x\le2b+a^2$.
			Thus, Eq.~(\ref{eq:sumsigmaf}) implies the final conclusion.
		\end{proof}
	\end{lemma}
	
	\subsection{Proof of Theorem~\ref{th:regkr}}\label{sec:pfthregkr}
	Putting these building blocks together, we are finally ready to give high probability upper bound on the regret in this subsection, which is based on high probability event $\Upsilon\cap\Xi_1\cap\Xi_2\cap\Xi_3$.
	
	\begin{proof}[Proof of Theorem~\ref{th:regkr}]
		
		By construction, taking a union bound, we have that with probability $1-10\delta$, $\Upsilon\cap\Xi_1\cap\Xi_2\cap\Xi_3$ holds.
		In the remainder of the proof, assume that we are conditioning on this event.
		Initially, we have
		
		\begin{align}\label{eq:regretff}
		    \begin{split}
		        \operatorname{Regret}(K)=&\sum_{k=1}^{K}\left[V_{1}^{\star}(s_{1}^{k})-V_{1}^{\pi_{k}}(s_{1}^{k})\right]\le\sum_{k=1}^{K}\left[\widehat{V}_{k,1}(s_1^k)-V_1^{\pi^k}(s_1^k)\right]\\
		        \le&4\widehat{\beta}\sqrt{\sum_{k=1}^K\sum_{h=1}^H\widehat{\sigma}_{k,h}^2}\sqrt{H\cdot2d\log\left(1+\frac{K}{Hd\lambda}\right)}+\frac{3H^2d}{\log (2)} \log \left(1+\frac{1}{\lambda H\log (2)}\right)+2H\sqrt{2T\log\left(\frac{H}{\delta}\right)},
		    \end{split}
		\end{align}
		where the first inequality holds since $\widehat{V}_{k,1}(\cdot)\ge V_{1}^*(\cdot)$ under $\Upsilon$ by Lemma~\ref{lm:optimism},
		and the second inequality holds due to Lemma~\ref{lm:sumregret}.

		Then, by Lemma~\ref{lm:sigma}, we have
		\begin{align*}
			\sum_{k=1}^{K}\sum_{h=1}^{H}\widehat{\sigma}_{k, h}^{2}\le&8HT+\frac{6H^5d^6}{\log(1+c^2)}\log\left(1+\frac{d}{\lambda H\log (1+c^2)}\right)+6H^{3} \log\left(\frac{1}{\delta}\right)\\
			&+8H^3(6+Hd^3)\sqrt{2T\log\left(\frac{H}{\delta}\right)}+\frac{12H^4d}{\log(2)}\log\left(1+\frac{1}{\lambda\log(2)}\right)+2H^4d^3\sqrt{\lambda}+\\
			&8Hd\left[H\left(2H(6+Hd^3)\widehat{\beta}+(4+Hd^3)\widebar{\beta}+(H+1)(4+Hd^3)\widecheck{\beta}\right)+\widetilde{\beta}\right]^2\log\left(1+\frac{K}{Hd\lambda}\right)
		\end{align*}
		under $\Upsilon\cap\Xi_1\cap\Xi_2$, where $c=1/(H^3d^5)$.
		
		On the one hand, since $1/\log(1+1/x)\le2x$ for any $x\ge1$, we have
		\begin{align*}
		    \frac{1}{\log(1+c^2)}\le\frac{1}{\log\left(1+\frac{1}{H^6d^{10}}\right)}\le2H^6d^{10}
		\end{align*}
		for $H^6d^10\ge1$.
		One the other hand, setting $\lambda=1/(H^2\sqrt{d})$ gives $\widehat{\beta}=O(\sqrt{d}\log(T))$,
		$\widebar{\beta}=O(H\sqrt{d^3}\log(T))$, $\widetilde{\beta}=O(H^2\sqrt{d^3}\log(T))$ and $\widecheck{\beta}=O(H\sqrt{d^3}\log(T))$, where $T=KH$.
		
		Thus, by some basic manipulations, we obtain
		\begin{align}\label{eq:regret1}
		    \begin{split}
		        \sum_{k=1}^{K}\sum_{h=1}^{H}\widehat{\sigma}_{k,h}^2\le&O\left(HT+H^{11}d^{16}+H^{4}d^3\sqrt{T\log(T)}+H^9d^{10}\log^3(T)\right)\\
		        \le&O\left(HT+H^{11}d^{16}+H^7d^4\log(T)+H^9d^{10}\log^3(T)\right)\\
		        \le&O\left(HT+H^{11}d^{16}+H^9d^{10}\log^3(T)\right)
		    \end{split}
		\end{align}
		where the second inequality holds since $H^4d^3\sqrt{T\log(T)}\le [HT+H^{7}d^{6}\log(T)]/2$, and the third inequality holds since $H^{7}d^{4}\log(T)\le H^{9}d^{10}\log^3(T)$.
		Substituting Eq.~(\ref{eq:regret1}) in (\ref{eq:regretff}) gives
		\begin{align*}
			&\operatorname{Regret}(K)\\
			\le&O\left(\sqrt{d}\log(T)\sqrt{HT+H^{11}d^{16}+H^{9}d^{10}\log^3(T)}\sqrt{Hd\log(T)}+\frac{3H^2d}{\log (2)}\log\left(1+\frac{H\sqrt{d}}{\log(2)}\right)+2H\sqrt{2T\log\left(\frac{H}{\delta}\right)}\right)\\
			\le&O\left(Hd\sqrt{T\log^3(T)}+H^6d^9\sqrt{\log^3(T)}+H^5d^6\log^2(T)\right)\\
			=&\widetilde{O}\left(Hd\sqrt{T}+H^6d^9\right),
		\end{align*}
		where the second inequality holds by dropping lower order terms and the fact that $\sqrt{a+b+c}\le\sqrt{a}+\sqrt{b}+\sqrt{c}$ for any $a,b,c>0$.
	\end{proof}
	
	\section{Lower Bound}\label{sec:applower}
	Remark 23 in Appendix of \cite{zhou2021nearly} constructs a hard-to-learn linear MDP instance, which shares the same order of regret lower bound as a hard-to-learn linear mixture MDP instance. According to Theorem 8 in \cite{zhou2021nearly}, the linear mixture MDP has a regret lower bound of $\Omega(Hd\sqrt{T})$, which means the linear MDP with known reward has the same regret lower bound as well.
	We present the construction of this hard-to-learn linear MDP from \cite{zhou2021nearly} in this section for completeness.
	This hard-to-learn linear MDP can be regarded as an extension of hard instances in linear bandits literature \cite{dani2008stochastic,lattimore2020bandit}.
	We first illustrate the structure of this MDP and then present the specific linear parametrization.
	
	\paragraph{Hard MDP Instance}
	This MDP instance is denoted as $\mathcal{M}=\{\mathcal{S}, \mathcal{A}, H, \{\mathbb{P}_1\}_h, \{r_h\}_h\}$.
	The state space $\mathcal{S}$ consists of states $s_{1}, \ldots s_{H+2}$ such that $|\mathcal{S}|=H+2$.
	There are $2^{d-1}$ actions and $\mathcal{A}=\{-1,1\}^{d-1}$ such that each action $\boldsymbol{a}\in\mathcal{A}$ is denoted in vector from.
	\begin{itemize}
		\item Reward: For any stage $h \in[H+2]$, only transitions originating at $s_{H+2}$ incurs a reward.
		\item Transition: $s_{H+1}$ and $s_{H+2}$ are absorbing regardless of what action is taken.
		For state $s_{i}$ with $i \leq H$, the transition probability is given as
		$$
		\begin{aligned}
			\mathbb{P}_h(s'|s_i,\boldsymbol{a})=
			\begin{cases}
				\iota+\langle\boldsymbol{\mu}_{h}, \boldsymbol{a}\rangle,&s'=s_{H+2}\\
				1-\left(\iota+\langle\boldsymbol{\mu}_{h}, \boldsymbol{a}\rangle\right),&s'=s_{i+1}\\
				0,&\text{Otherwise}
			\end{cases}
		\end{aligned},
		$$
		where $\iota=1 / H$ and $\boldsymbol{\mu}_{h} \in\{-\Delta, \Delta\}^{d-1}$ with $\Delta=\sqrt{\iota / K} /(4 \sqrt{2})$  to make the probabilities are well-defined.
		The transition of this MDP is detailed in Figure~\ref{fig:hard}.
	\end{itemize}
	
	\begin{figure}[H]
		\begin{center}
			\centerline{\includegraphics[width=\columnwidth]{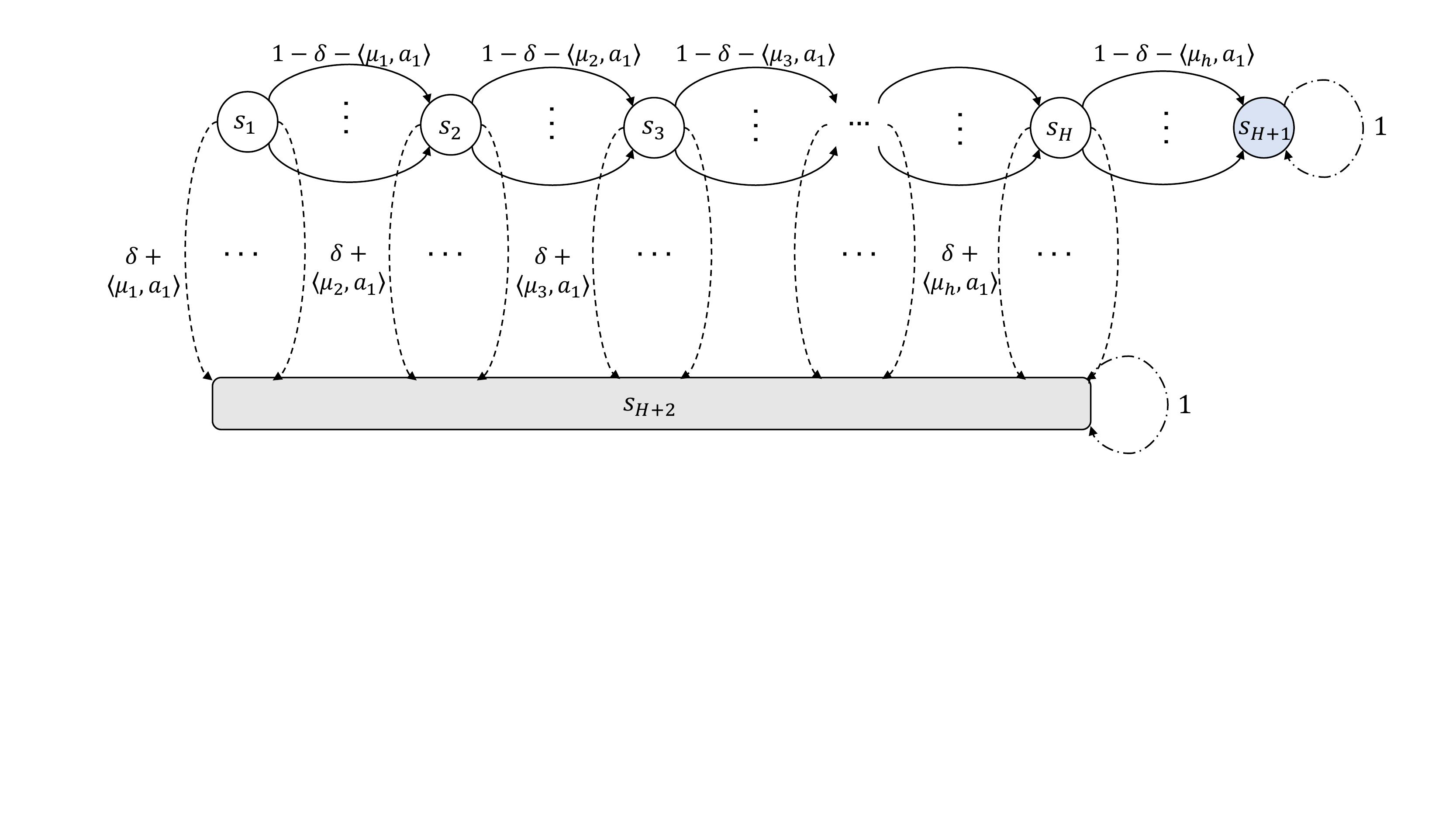}\vspace{-4cm}}
			\caption{The transition matrix $\mathbb{P}_{h}$ of the hard-to-learn MDP.}
			\label{fig:hard}
		\end{center}
	\end{figure}
	
	\paragraph{Linear Parametrization}
	Then, we specify the linear parametrization of this MDP.
	For each $h \in[H]$, the transition probability matrix ${\mathbb{P}}_{h}$ and the reward function ${r}_{h}$ are defined as
	${\mathbb{P}}_{h}\left(s^{\prime} \mid s, \boldsymbol{a}\right)=\langle\boldsymbol{\phi}(s, \boldsymbol{a}),\boldsymbol{\mu}_{h}\left(s^{\prime}\right)\rangle$
	and
	${r}_{h}(s, \boldsymbol{a})=\langle\boldsymbol{\phi}(s, \boldsymbol{a}),\boldsymbol{\theta}_{h}\rangle$, where $\boldsymbol{\phi}(s, a)\in$ $\mathbb{R}^{d+1}$ is the known feature mapping, $\boldsymbol{\mu}_h=(\mu_h(s_1),...,\mu_h(s_{H+2}))\in\mathbb{R}^{(d+1)\times(H+2)}$ and $\boldsymbol{\theta}_{h} \in \mathbb{R}^{d+1}$ are unknown parameters in linear MDPs.
	Here, $\phi(s, \boldsymbol{a}), \boldsymbol{\mu}_{h}, \boldsymbol{\theta}_{h}$ are specified as:
	$$
	\begin{aligned}
		\boldsymbol{\phi}(s, \boldsymbol{a})= &\begin{cases}\left(\alpha, \beta \boldsymbol{a}^{\top}, 0\right)^{\top}, & s=s_{h}, h \in[H+1]\\
			\left(0, \mathbf{0}^{\top}, 1\right)^{\top}, & s=s_{H+2} \end{cases} \\
		\boldsymbol{\mu}_{h}\left(s^{\prime}\right)=& \begin{cases}\left((1-\iota) / \alpha,-\boldsymbol{\mu}_{h}^{\top} / \beta, 0\right)^{\top}, & s^{\prime}=s_{H+1}\\
			\left(\iota / \alpha, \boldsymbol{\mu}_{h}^{\top} / \beta, 1\right)^{\top}, & s^{\prime}=s_{H+2}\\
			\mathbf{0}, & \text { otherwise }\end{cases}\\
		\boldsymbol{\theta}_{h}=&\left(\mathbf{0}^{\top}, 1\right)^{\top}
	\end{aligned},
	$$
	where $\alpha=\sqrt{1/(1+\Delta(d-1))}$, $\beta=\sqrt{\Delta/(1+\Delta(d-1))}$.

	\paragraph{Norm Assumption}
	We check the norm assumption of linear MDPs in the following:
	\begin{enumerate}
		\item For $s=s_h$ where $h\in[H+1]$, $\|\boldsymbol{\phi}(s,\boldsymbol{a})\|_2=\sqrt{\alpha^2+(d-1)\beta^2}=1$ and $\|\boldsymbol{\phi}(s_{H+2},\boldsymbol{a})\|_2=1$.
		Thus, $\|\boldsymbol{\phi}(s, \mathbf{a})\|_{2} \leq 1$ for any $(s, \boldsymbol{a}) \in \mathcal{S} \times \mathcal{A}$.
		
		\item For any $\boldsymbol{v}=(v_1,\ldots,v_{H+2})\in\mathbb{R}^{H+2}$ such that $\|\boldsymbol{v}\|_\infty\le1$, we have
		$$
		\|\boldsymbol{\mu}_h\boldsymbol{v}\|_2^2=\left[\frac{v_1(1-\iota)}{\alpha}+\frac{v_1\iota}{\alpha}\right]^2+v_{H+2}^2\le\frac{1}{\alpha^2}+1=\left[1+\Delta(d-1)\right]^2+1=\left[1+\frac{\sqrt{\delta/K}}{4\sqrt{2}}(d-1)\right]^2+1\le d+1,
		$$
		where the last inequality holds by assuming episode number $K\ge(d-1)/(32H(\sqrt{d}-1))$.
		Thus, $\|\boldsymbol{\mu}_h\boldsymbol{v}\|_2\le\sqrt{d+1}$ for any $h\in[H]$.
		\item In addition, $\|\boldsymbol{\theta}_h\|_2\le1\le\sqrt{d+1}$ for any $h\in[H]$.
	\end{enumerate}
	
	\paragraph{Lower Bound}
	The constructed linear MDP above has the same state space $\mathcal{S}$, action space $\mathcal{A}$, episode length $H$, reward function $\{r_h\}_{h\in[H]}$ and transition probability $\{\mathbb{P}_h\}_{h\in[H]}$ as the constructed hard-to-learn linear mixture MDP in Appendix E. of \cite{zhou2021nearly}, which shares the same regret lower bound $\Omega(Hd\sqrt{T})$ as shown in Theorem 8 in \cite{zhou2021nearly} and is formalized in the following theorem.
	
	\begin{lemma}[Lower bound of linear MDPs]\label{th:lower}
		Let $d>1$ and suppose $K \geq \max \left\{(d-1)^{2} H / 2,(d-1) /(32 H(d-1))\right\}$, $d \geq 4$, $H \geq 3$. Then for any algorithm there exists an episodic linear MDP parameterized by $\{\boldsymbol{\mu}_h\}_{h\in[H]},\{\boldsymbol{\theta}_h\}_{h\in[H]}$ and satisfy the norm assumption given in Definition~\ref{def:linear}, such that the expected regret is lower bounded as follows:
		$$
		\mathbb{E}[\operatorname{Regret}(K)]\geq \Omega(Hd\sqrt{T}),
		$$
		where $T=K H$ and the expectation is taken over the probability distribution generated by the interconnection of the algorithm and the MDP.
		\begin{proof}
			The proof is the same as that of Theorem~8 in \cite{zhou2021nearly}, except for changing $B$ to $\sqrt{d}$ to satisfy the norm assumption of linear MDPs.
		\end{proof}
	\end{lemma}
	
	\section{Auxiliary Lemmas}\label{sec:appaux}
	In this section, we give some auxiliary lemmas which serve as the preliminary for the proof above.
	We also include some other lemmas that are unnecessary for our theoretical analysis but can help readers be more familiar with related works.
	In general, these lemmas are categorized into four subsections:
	
	\begin{itemize}
		\item Appendix~\ref{sec:appci} for some concentration inequalities;
		\item Appendix~\ref{sec:appep} for properties related to elliptical potentials;
		\item Appendix~\ref{sec:applmp} presents some useful properties for linear MDPs;
		\item Appendix~\ref{sec:appcover} builds the covering number for covering net over some function classes of our interests.
	\end{itemize}

	\subsection{Concentration Inequality}\label{sec:appci}
	In this subsection, Lemma~\ref{lm:hoeffding} presents the Azuma-Hoeffding inequality, Lemma~\ref{lm:freedman} presents the Freedman's inequality in \cite{freedman1975tail}, Lemma~\ref{lm:vectorhoeffding} a Hoeffding-type self-normalized bound, and Lemma~\ref{lm:selffull} presents the full version of Theorem~\ref{th:self} in main paper.
	
	\begin{lemma}[Azuma-Hoeffding Inequality]\label{lm:hoeffding}
		Let $\{x_i\}_{i=1}^n$ be a martingale difference sequence with respect to a filtration $\{\mathcal{G}_i\}_{i=1}^{n+1}$ such that $|x_i| \leq M$ almost surely. That is, $x_i$ is $\mathcal{G}_{i+1}$-measurable and $\mathbb{E}\left[ x_i \mid \mathcal{G}_i\right]=0$ a.s.  Then for any $0 < \delta < 1$, with probability at least $1 - \delta$, 
		$$ \sum_{i=1}^n x_i \leq M\sqrt{2n\log(1/\delta)}$$ 
	\end{lemma}
	
	\begin{lemma}[Freedman's Inequality, \cite{freedman1975tail}]\label{lm:freedman}
		Let $\left\{x_{i}, \mathcal{F}_{i}\right\}$ be a martingale difference sequence with $\forall i\ge1$, $\mathbb{E}\left(x_{i}\mid \mathcal{F}_{i-1}\right)=0$, $\mathbb{E}\left(x_{i}^{2} \mid \mathcal{F}_{i-1}\right)=\sigma_{i}^{2}$, $V_{i}^{2}=\sum_{j=1}^{i} \sigma_{j}^{2} .$ Furthermore, assume that $\mathbb{P}\left(\left|x_{i}\right| \leq c \mid \mathcal{F}_{i-1}\right)=1$ for any $0<c<\infty$.
		
		Then, for fixed $t\ge1$ and any $\delta>0$, with probability at least $1-\delta$, we have:
		$$
		\sum_{i=1}^{t} d_{i} \leq \sqrt{2V_t^2\log (1 / \delta)}+\frac{2}{3}c \log (1 / \delta).
		$$
	\end{lemma}
	
	\begin{lemma}[Hoeffding inequality for vector-valued martingales, Theorem 1 in \cite{abbasi2011improved}]\label{lm:vectorhoeffding}
		Let $\left\{\mathcal{G}_{t}\right\}_{t=1}^{\infty}$ be a filtration, $\left\{\boldsymbol{x}_{t}, \eta_{t}\right\}_{t \geq 1}$ be a stochastic process so that $\boldsymbol{x}_{t} \in \mathbb{R}^{d}$ is $\mathcal{G}_{t}$-measurable and $\eta_{t} \in \mathbb{R}$ is $\mathcal{G}_{t+1}$-measurable.
		
		Denote $\mathbf{Z}_{t}=\lambda \mathbf{I}+\sum_{i=1}^{t} \boldsymbol{x}_{i} \boldsymbol{x}_{i}^{\top}$ for $t\ge1$ and $\mathbf{Z}_{0}=\lambda \mathbf{I}$.
		If $\|\boldsymbol{x}_t\|_2\le L$, and $\eta_{t}$ satisfies
		$$\mathbb{E}\left[\eta_{t} \mid \mathcal{G}_{t}\right]=0,\quad\left|\eta_t\right|\le R$$
		for all $t\ge1$. Then, for any $0<\delta<1$, with probability at least $1-\delta$ we have:
		$$
		\forall t>0,\left\|\sum_{i=1}^{t} \boldsymbol{x}_{i} \eta_{i}\right\|_{\mathbf{Z}_{t}^{-1}} \leq R\sqrt{d\log\left(1+tL^2/d\lambda\right)+\log(1/\delta)}.
		$$
	\end{lemma}
	
	\begin{lemma}[Bernstein inequality for vector-valued martingales, full version of Theorem~\ref{th:self}]\label{lm:selffull}
		Let $\left\{\mathcal{G}_{t}\right\}_{t=1}^{\infty}$ be a filtration, $\left\{\boldsymbol{x}_{t}, \eta_{t}\right\}_{t \geq 1}$ be a stochastic process so that $\boldsymbol{x}_{t} \in \mathbb{R}^{d}$ is $\mathcal{G}_{t}$-measurable and $\eta_{t} \in \mathbb{R}$ is $\mathcal{G}_{t+1}$-measurable.
		
		If $\|\boldsymbol{x}_t\|_2\le L$, and $\eta_{t}$ satisfies
		$$\mathbb{E}\left[\eta_{t} \mid \mathcal{G}_{t}\right]=0,\quad\mathbb{E}\left[\eta_{t}^{2} \mid \mathcal{G}_{t}\right] \leq \sigma^{2},\quad\left|\eta_t\cdot\min\left\{1,\|\boldsymbol{x}_t\|_{\mathbf{Z}_{t-1}^{-1}}\right\}\right|\le R$$
		for all $t\ge1$. Then, for any $0<\delta<1$, with probability at least $1-\delta$ we have:
		$$
		\forall t>0,\left\|\sum_{i=1}^{t} \boldsymbol{x}_{i} \eta_{i}\right\|_{\mathbf{Z}_{t}^{-1}} \leq8 \sigma \sqrt{d \log \left(1+t L^{2} /(d \lambda)\right) \log \left(4 t^{2} / \delta\right)}+4 R \log \left(4 t^{2} / \delta\right)
		$$
		where $\mathbf{Z}_{t}=\lambda \mathbf{I}+\sum_{i=1}^{t} \boldsymbol{x}_{i} \boldsymbol{x}_{i}^{\top}$ for $t\ge1$ and $\mathbf{Z}_{0}=\lambda \mathbf{I}$.
	\end{lemma}
	
	\subsection{Elliptical Potentials}\label{sec:appep}
	In this subsection, we present Lemma~\ref{lm:ablog} in \cite{abbasi2011improved}, which is an important lemma for building $\sqrt{O}(\sqrt{T})$ worst-case regret for many algorithms for linear bandits or RL with linear function approximation.
    Then, we present Lemma~\ref{lm:ood}, about elliptical potentials (Exercise 19.3 in \cite{lattimore2020bandit}), which states that one cannot have more than $O(d)$ big intervals.
	Lemma~\ref{lm:ood} is further generalized in Lemma~\ref{lm:od} in Appendix~\ref{sec:appsigma}.
    In addition, we also present Lemma~\ref{lm:rsn} (Lemma 12 in \cite{abbasi2011improved}) and Lemma~\ref{lm:numupdate} (Lemma E.1. in \cite{he2022nearly}) about the ``rare-switching” update strategy of the constructed value function.

	\begin{lemma}[Lemma 11, \cite{abbasi2011improved}]\label{lm:ablog}
		Given $\lambda>0$ and sequence $\left\{\boldsymbol{x}_{t}\right\}_{t=1}^{T} \subset$ $\mathbb{R}^{d}$ with $\left\|\boldsymbol{x}_{t}\right\|_{2} \leq L$ for all $t\in[T]$, define $\mathbf{Z}_{t}=\lambda \mathbf{I}+\sum_{i=1}^{t} \boldsymbol{x}_{i} \boldsymbol{x}_{i}^{\top}$ for $t\ge1$ and $\mathbf{Z}_{0}=\lambda \mathbf{I}$. We have
		$$
		\sum_{t=1}^{T} \min \left\{1,\left\|\boldsymbol{x}_{t}\right\|_{\mathbf{Z}_{t-1}^{-1}}^{2}\right\} \leq 2 d \log\left(1+\frac{T L^{2}}{d \lambda}\right).
		$$
	\end{lemma}
	
	\begin{lemma}[Elliptical Potentials: You cannot have more than $O(d)$ big intervals. Exercise 19.3 in \cite{lattimore2020bandit}]\label{lm:ood}
		Given $\lambda>0$ and sequence $\left\{\boldsymbol{x}_{t}\right\}_{t=1}^{T} \subset$ $\mathbb{R}^{d}$ with $\left\|\boldsymbol{x}_{t}\right\|_{2} \leq L$ for all $t\in[T]$, define $\mathbf{Z}_{t}=\lambda \mathbf{I}+\sum_{i=1}^{t} \boldsymbol{x}_{i} \boldsymbol{x}_{i}^{\top}$ for $t\ge1$ and $\mathbf{Z}_{0}=\lambda \mathbf{I}$. The number of times $\left\|\boldsymbol{x}_{t}\right\|_{\mathbf{Z}_{t-1}^{-1}}^{2}\geq 1$ is at most
		$$
		\frac{3 d}{\log (2)} \log \left(1+\frac{L^{2}}{\lambda \log (2)}\right).
		$$
	\end{lemma}

\begin{lemma}[Lemma 12 in \cite{abbasi2011improved}]\label{lm:rsn}
    Suppose $\mathbf{A},\mathbf{B}\in\mathbb{R}^{d\times d}$ are two positive definite matrices satisfying that $\mathbf{A}\succeq\mathbf{B}$, then for any $\boldsymbol{x}\in\mathbb{R}^{d}$, $\|\boldsymbol{x}\|_\mathbf{A}\le\|\boldsymbol{x}\|_\mathbf{B}\cdot\sqrt{\det{\mathbf{A}}/\det{\mathbf{A}}}$.
\end{lemma}
 \begin{lemma}[Lemma E.1. in \cite{he2022nearly}]\label{lm:numupdate}
		The number of episodes where the algorithm updates the value function in Algorithm~\ref{alg:plus} is upper bounded by $dH\log(1 + K)$.
	\end{lemma}
 
	\subsection{Linear MDP Property}\label{sec:applmp}
	This subsection gives some indirect results about the estimated parameter $\widehat{\boldsymbol{\mu}}_{k,h}$ in linear MDPs.
	
	\begin{lemma}\label{lm:mud}
		In Algorithm~\ref{alg:plus}, for any $k\in[K]$ and any $h\in[H]$, we have:
		
		\begin{equation}\label{eq:mud2}
			\widehat{\boldsymbol{\mu}}_{k,h}-\boldsymbol{\mu}_h=\widehat{\mathbf{\Lambda}}_{k,h}^{-1}\left[-\lambda\boldsymbol{\mu}_h+\sum_{i=1}^{k-1}\widehat{\sigma}_{i,h}^{-2}\boldsymbol{\phi}(s_{h}^{i},a_{h}^{i}){\boldsymbol{\epsilon}_h^i}^\top\right]
		\end{equation}
		
		\begin{proof}
			We start from the closed-form solution of $\widehat{\boldsymbol{\mu}}_{k,h}$ :
			$$
			\begin{aligned}
				\widehat{\boldsymbol{\mu}}_{k,h}=&\widehat{\mathbf{\Lambda}}_{k,h}^{-1}\sum_{i=1}^{k-1}\widehat{\sigma}_{i,h}^{-2}\boldsymbol{\phi}(s_{h}^{i},a_{h}^{i})\boldsymbol{\delta}\left(s_{h+1}^{i}\right)^{\top}=\widehat{\mathbf{\Lambda}}_{k,h}^{-1}\sum_{i=1}^{k-1}\widehat{\sigma}_{i,h}^{-2}\boldsymbol{\phi}(s_{h}^{i},a_{h}^{i})\left(\mathbb{P}_h(\cdot\mid s_h^i,a_h^i)^\top+{\boldsymbol{\epsilon}_h^i}^\top\right)\\
				=&\widehat{\mathbf{\Lambda}}_{k,h}^{-1}\sum_{i=1}^{k-1}\widehat{\sigma}_{i,h}^{-2}\boldsymbol{\phi}(s_{h}^{i},a_{h}^{i})\left(\boldsymbol{\phi}(s_h^i,a_h^i)^\top\boldsymbol{\mu}_h+{\boldsymbol{\epsilon}_h^i}^\top\right)\\
				=&\widehat{\mathbf{\Lambda}}_{k,h}^{-1}\sum_{i=1}^{k-1}\widehat{\sigma}_{i,h}^{-2}\boldsymbol{\phi}(s_{h}^{i},a_{h}^{i})\boldsymbol{\phi}(s_h^i,a_h^i)^\top\boldsymbol{\mu}_h+\widehat{\mathbf{\Lambda}}_{k,h}^{-1}\sum_{i=1}^{k-1}\widehat{\sigma}_{i,h}^{-2}\boldsymbol{\phi}(s_{h}^{i},a_{h}^{i}){\boldsymbol{\epsilon}_h^i}^\top\\
				=&\widehat{\mathbf{\Lambda}}_{k,h}^{-1}\left(\widehat{\mathbf{\Lambda}}_{k,h}-\lambda\mathbf{I}\right)\boldsymbol{\mu}_h+\widehat{\mathbf{\Lambda}}_{k,h}^{-1}\sum_{i=1}^{k-1}\widehat{\sigma}_{i,h}^{-2}\boldsymbol{\phi}(s_{h}^{i},a_{h}^{i}){\boldsymbol{\epsilon}_h^i}^\top\\
				=&\boldsymbol{\mu}_h-\lambda\widehat{\mathbf{\Lambda}}_{k,h}^{-1}\boldsymbol{\mu}_h+\widehat{\mathbf{\Lambda}}_{k,h}^{-1}\sum_{i=1}^{k-1}\widehat{\sigma}_{i,h}^{-2}\boldsymbol{\phi}(s_{h}^{i},a_{h}^{i}){\boldsymbol{\epsilon}_h^i}^\top.
			\end{aligned}
			$$
			Rearranging terms gives Eq.~(\ref{eq:mud2}).
		\end{proof}
	\end{lemma}
	
	\subsection{Covering Net}\label{sec:appcover}
	This subsection presents lemmas required for bounding the covering number of considered function classes, including $\widehat{\mathcal{V}},\widehat{\mathcal{V}}^2,\widecheck{\mathcal{V}}$.
	
	\begin{lemma}[Covering Number of Euclidean Ball, Lemma D.5. in \cite{jin2020provably}]\label{lm:ball}
		For any $\varepsilon>0$, the $\varepsilon$-covering number of the Euclidean ball in $\mathbb{R}^{d}$ with radius $R>0$ is upper bounded by $(1+2 R / \varepsilon)^{d}$.
	\end{lemma}

	\begin{lemma}[Lemma E.6. in \cite{he2022nearly}]\label{lm:coverhatv}
		Let $\widehat{\mathcal{N}}_{\varepsilon}$ be the $\varepsilon$-covering of $\widehat{\mathcal{V}}$ with respect to the distance $\operatorname{dist}\left(V, V^{\prime}\right)=\sup _{x}\left|V(x)-V^{\prime}(x)\right|$, where $\widehat{\mathcal{V}}$ is defined in Definition~\ref{def:hatv}. Then
		$$
		\log|\widehat{\mathcal{N}}_{\varepsilon}|\leq dJ\log (1+4 L / \varepsilon)+d^{2}J\log \left[1+8 d^{1 / 2} B^{2} /\left(\lambda \varepsilon^{2}\right)\right] .
		$$
	\end{lemma}
	
	\begin{lemma}[Lemma E.8. in \cite{he2022nearly}]\label{lm:coverhatv2}

 Let $\widehat{\mathcal{N}}^2_{\varepsilon}$ be the $\varepsilon$-covering of $\widehat{\mathcal{V}}^2$ with respect to the distance $\operatorname{dist}\left(V, V^{\prime}\right)=\sup _{x}\left|V(x)-V^{\prime}(x)\right|$, where $\widehat{\mathcal{V}}^2$ is defined in Definition~\ref{def:hatv2}. Then
		$$
		\log|\widehat{\mathcal{N}}^2_{\varepsilon}|\leq d J\log (1+8LH/ \varepsilon)+d^{2} J\log \left[1+32d^{1 / 2} B^{2} H^2/\left(\lambda \varepsilon^{2}\right)\right] .
		$$
	\end{lemma}
	
	\begin{lemma}\label{lm:covercheckv}
		Let $\widecheck{\mathcal{N}}_{\varepsilon}$ be the $\varepsilon$-covering of $\widecheck{\mathcal{V}}$ with respect to the distance $\operatorname{dist}\left(V, V^{\prime}\right)=\sup _{x}\left|V(x)-V^{\prime}(x)\right|$, where $\widecheck{\mathcal{V}}$ is defined in Definition~\ref{def:checkv}. Then
		$$
		\log|\widecheck{\mathcal{N}}_{\varepsilon}|\leq d \log (1+4 L / \varepsilon)+d^{2} \log \left[1+8 d^{1 / 2} B^{2} /\left(\lambda \varepsilon^{2}\right)\right] .
		$$
		\begin{proof}
		    Denote $\mathbf{A}=\beta^{2} \Lambda^{-1}$, then for any $\widecheck{V}(\cdot)\in\widecheck{\mathcal{V}}$,
			\begin{equation}\label{eq:newform1}
				\widecheck{V}(\cdot)=\max \left\{\max _{a} \boldsymbol{w}^{\top} \boldsymbol{\phi}(\cdot, a)-\sqrt{\boldsymbol{\phi}(\cdot, a)^{\top} \mathbf{A} \boldsymbol{\phi}(\cdot, a)}, 0\right\}
			\end{equation}
			for $\|\boldsymbol{w}\| \leq L$ and $\|\mathbf{A}\| \leq B^{2} \lambda^{-1}$.
			The proof is almost the same as that for Lemma~\ref{lm:coverhatv}, since for any two functions $\widecheck{V}_{1}, \widecheck{V}_{2} \in \widecheck{\mathcal{V}}$, let them take the form in Eq.~(\ref{eq:newform1}) with parameters $\left(\boldsymbol{w}_{1}, \mathbf{A}_{1}\right)$ and $\left(\boldsymbol{w}_{2}, \mathbf{A}_{2}\right)$, respectively.
			Since $|\min\{x,0\}-\min\{y,0\}|\le|x-y|$ for any $x,y\in\mathbb{R}$ and $\max _{a}$ is a contraction mapping, we have
			\begin{equation}\label{eq:dist}
				\begin{aligned}
					\operatorname{dist}(\widehat{V}_{1}, \widehat{V}_{2}) & \leq \sup _{s, a}\left|\left[\boldsymbol{w}_{1}^{\top} \boldsymbol{\phi}(s, a)-\sqrt{\boldsymbol{\phi}(s, a)^{\top} \mathbf{A}_{2} \boldsymbol{\phi}(s, a)}\right]-\left[\boldsymbol{w}_{2}^{\top} \boldsymbol{\phi}(s, a)-\sqrt{\boldsymbol{\phi}(s, a)^{\top} \mathbf{A}_{2} \boldsymbol{\phi}(s, a)}\right]\right| \\
					& \leq \sup _{\boldsymbol{\phi}:\|\boldsymbol{\phi}\|_2 \leq 1}\left|\left[\boldsymbol{w}_{1}^{\top} \boldsymbol{\phi}-\sqrt{\boldsymbol{\phi}^{\top} \mathbf{A}_{2} \boldsymbol{\phi}}\right]-\left[\boldsymbol{w}_{2}^{\top} \boldsymbol{\phi}-\sqrt{\boldsymbol{\phi}^{\top} \mathbf{A}_{2} \boldsymbol{\phi}}\right]\right| \\
					& \leq \sup _{\boldsymbol{\phi}:\|\boldsymbol{\phi}\|_2 \leq 1}\left|\left(\boldsymbol{w}_{1}-\boldsymbol{w}_{2}\right)^{\top} \boldsymbol{\phi}\right|+\sup _{\boldsymbol{\phi}:\|\boldsymbol{\phi}\|_2 \leq 1} \sqrt{\left|\boldsymbol{\phi}^{\top}\left(\mathbf{A}_{2}-\mathbf{A}_{1}\right) \boldsymbol{\phi}\right|} \\
					&=\left\|\boldsymbol{w}_{1}-\boldsymbol{w}_{2}\right\|_2+\sqrt{\left\|\mathbf{A}_{1}-\mathbf{A}_{2}\right\|_2} \leq\left\|\boldsymbol{w}_{1}-\boldsymbol{w}_{2}\right\|_2+\sqrt{\left\|\mathbf{A}_{1}-\mathbf{A}_{2}\right\|_{F}},
				\end{aligned}
			\end{equation}
                where the second last inequality follows from the fact that $|\sqrt{x}-\sqrt{y}|\le\sqrt{x-y}$ holds for any $x,y\ge0$.
                For matrices, $\|\cdot\|_2$ and $\|\cdot\|_F$ denote the matrix operator norm and Frobenius norm, respectively.

                Let $\mathcal{C}_{\boldsymbol{w}}$ be an $\varepsilon / 2$-cover of $\left\{\boldsymbol{w} \in \mathbb{R}^{d} \mid\|\boldsymbol{w}\|_2 \leq L\right\}$ with respect to the 2-norm, and $\mathcal{C}_{\mathbf{A}}$ be an $\varepsilon^{2} / 4$-cover of $\left\{\mathbf{A} \in \mathbb{R}^{d \times d} \mid\|\mathbf{A}\|_{F} \leq d^{1 / 2} B^{2} \lambda^{-1}\right\}$ with respect to the Frobenius norm.
			By Lemma~\ref{lm:ball}, we have:
			$$
			\left|\mathcal{C}_{\boldsymbol{w}}\right| \leq(1+4 L / \varepsilon)^{d}, \quad\left|\mathcal{C}_{\mathbf{A}}\right| \leq\left[1+8 d^{1 / 2} B^{2} /\left(\lambda \varepsilon^{2}\right)\right]^{d^{2}}
			$$
			By Eq.~(\ref{eq:dist}), for any $\widehat{V}_{1} \in \widehat{\mathcal{V}}$, there exists $\boldsymbol{w}_{2} \in \mathcal{C}_{\boldsymbol{w}}$ and $\mathbf{A}_{2} \in \mathcal{C}_{\mathbf{A}}$ such that $\widehat{V}_{2}$ parametrized by $\left(\boldsymbol{w}_{2}, \mathbf{A}_{2}\right)$ satisfies $\operatorname{dist}(\widehat{V}_{1}, \widehat{V}_{2}) \leq \varepsilon$.
			Hence, it holds that $|\mathcal{N}_{\varepsilon}| \leq\left|\mathcal{C}_{\boldsymbol{w}}\right| \cdot\left|\mathcal{C}_{\mathbf{A}}\right|$, which gives
			$$
			\log|\mathcal{N}_{\varepsilon}| \leq \log \left|\mathcal{C}_{\boldsymbol{w}}\right|+\log \left|\mathcal{C}_{\mathbf{A}}\right| \leq d \log (1+4 L / \varepsilon)+d^{2} \log \left[1+8 d^{1 / 2} B^{2} /\left(\lambda \varepsilon^{2}\right)\right].
			$$
			This concludes the proof.
		\end{proof}
	\end{lemma}
	

\end{document}